\documentclass[twoside,11pt]{article}

\usepackage{jmlr2e}

\usepackage{lastpage}
\usepackage{amsmath,amssymb}
\usepackage{hyperref} % create hyperlinks
\usepackage{acronym}  % make an acronym
\usepackage{psfrag}
\usepackage{upgreek}
\usepackage{booktabs}
\usepackage{algorithm}
\usepackage{relsize}
\usepackage{algorithmic}
\usepackage{setspace}
\usepackage{subfigure}
\usepackage{mathtools}
\usepackage{fancyhdr}
\usepackage{stfloats}
\usepackage{acronym} 
\usepackage{graphicx,color,float}
\usepackage{multirow}
\usepackage{mathrsfs}
\usepackage{microtype}
\usepackage{todonotes}
\usepackage{upgreek} 
\newcommand\V[1]  { \mathbf{#1} }
\newcommand\B[1]  { \boldsymbol{#1} }
\newcommand\up[1] {\mathrm{#1}}
\newcommand\set[1] {\mathcal{#1}}
\newcommand\sset[1] {\mathsmaller{\mathcal{#1}}}
\newcommand{\pun}{\hspace{0.2mm}\B{\cdot}\hspace{0.2mm}}

\acrodef{SVM}{support vector machine}
\acrodef{SRM}{structural risk minimization}
\acrodef{NN}{neural network}
\acrodef{NB}{naive Bayes}
\acrodef{LR}{logistic regression}
\acrodef{CRF}{conditional random field}
\acrodef{LP}{linear program}
\acrodef{GP}{geometric program}
\acrodef{ERM}{empirical risk minimization}
\acrodef{RKHS}{reproducing kernel Hilbert space}
\acrodef{MRC}{minimax risk classifier}
\acrodef{RRM}{robust risk minimization}
\acrodef{LLN}{law of large numbers}
\acrodef{MMD}{maximum mean discrepancy}
\acrodef{DRLR}{distributionally robust logistic regression}
\acrodef{ASM}{accelerated subgradient method}
\acrodef{SM}{subgradient method}
 \hypersetup{ hidelinks }

\jmlrheading{24}{2023}{1-\pageref{LastPage}}{4/22; Revised 2/23}{7/23}{22-0339}{Santiago Mazuelas, Mauricio Romero, and Peter Gr\"{u}nwald}

\ShortHeadings{Minimax Risk Classifiers with 0\,-1 Loss}{Mazuelas, Romero, and Gr\"{u}nwald}
\firstpageno{1}

\begin{document}

\title{Minimax Risk Classifiers with 0\,-1 Loss} 

\author{\name Santiago Mazuelas \email smazuelas@bcamath.org \\
       \addr Basque Center for Applied Mathematics (BCAM)\\
       IKERBASQUE-Basque Foundation for Science\\
       Bilbao 48009, Spain
       \AND
       \name Mauricio Romero \email msicre@ufba.br \\
       \addr Federal University of Bahia\\
       Ondina 40170, Brazil
       \AND
       \name Peter Gr\"{u}nwald \email Peter.Grunwald@cwi.nl\\
       \addr National Research Institute for Mathematics and Computer Science (CWI)\\
       Amsterdam 94079, Netherlands
}
\editor{Sivan Sabato}

\maketitle

\begin{abstract}%   <- trailing '%' for backward compatibility of .sty file
Supervised classification techniques use training samples to learn a classification rule with small expected 0\,-1 loss (error probability). Conventional methods enable tractable learning and provide out-of-sample generalization by using surrogate losses instead of the 0\,-1 loss and considering specific families of rules (hypothesis classes). This paper presents minimax risk classifiers (MRCs) that minimize the worst-case \mbox{0\,-1 loss} with respect to uncertainty sets of distributions that can include the underlying distribution, with a tunable confidence. We show that MRCs can provide tight performance guarantees at learning and are strongly universally consistent using feature mappings given by characteristic kernels. The paper also proposes efficient optimization techniques for MRC learning and shows that the methods presented can provide accurate classification together with tight performance guarantees in practice. 
\end{abstract}

% minimize the worst-case \mbox{0\,-1 loss} with respect to an uncertainty set of distributions over general classification rules and provide tight performance guarantees at learning. 


\section{Introduction}

Supervised classification techniques use training samples to learn a classification rule that assigns labels to instances with small expected 0\,-1 loss (error probability). Conventional methods enable tractable learning and provide out-of-sample generalization by using surrogate losses instead of the 0\,-1 loss and considering specific families of rules (hypothesis classes). The surrogate losses are usually taken to be convex upper bounds of the 0\,-1 loss such as hinge loss, logistic loss, and exponential loss, see e.g., \cite{BarJorMca:06}. The families of rules considered are usually given by parametric functions such as those defined by \acp{NN} and those belonging to \acp{RKHS}, see e.g., \cite{ShaBen:14}.
Such techniques can result in strongly universally consistent methods using surrogate losses that are classification calibrated and Lipschitz \citep{BarJorMca:06,TewBar:07} together with rich families of rules such as those given by functions in \acp{RKHS} corresponding with universal kernels \citep{MicXuZha:06,Ste:05}. 

Most learning methods are based on the \ac{ERM} approach that minimizes the empirical expected loss of training samples (see e.g., \cite{Vap:98,MehRos:18,ShaBen:14}). Other methods are based on the \ac{RRM} approach that minimizes the worst-case expected loss with respect to an uncertainty set of distributions (see e.g., \cite{AsiXinBeh:15,ShaKuhMoh:19,DucNam:19}).  These methods correspond with generalized maximum entropy techniques as shown in \cite{MazShePer:22} using the theoretical framework in \cite{GruDaw:04}.

\ac{RRM} methods mainly differ in the type of uncertainty set considered. These sets are determined by constraints over probability distributions given in terms of metrics such as f-divergence \citep{DucNam:19}, Wasserstein distances \citep{ShaKuhMoh:19},  moments' fits \citep{AsiXinBeh:15}, and maximum mean discrepancies \citep{StaJeg:19}. In addition, the uncertainty sets considered often only include probability distributions with instances' marginal that coincides with the empirical marginal of training samples \citep{AsiXinBeh:15,FarTse:16,FatAnq:16,CorKuz:15}. An important advantage of \ac{RRM} methods is that the function minimized at learning can be an upper bound for the expected loss if the true underlying distribution is included in the uncertainty set considered. In such cases, \ac{RRM} techniques automatically ensure out-of-sample generalization and provide performance guarantees at learning. The uncertainty sets considered by existing techniques can include the true underlying distribution in methods that address a transductive setting \citep{BalFre:15, BalFre:16} or utilize Wasserstein distances \citep{ShaKuhMoh:19,LeeRag:18,FroClaChiSol:21} and maximum mean discrepancies \citep{StaJeg:19}. However, uncertainty sets formed by distributions with instances' marginal that coincides with the empirical do not include the true underlying distribution for finite sets of training samples.

The original 0\,-1 loss is utilized by certain \ac{RRM} techniques \citep{AsiXinBeh:15,FarTse:16,FatAnq:16,BalFre:15} which provided inspiration for the present work. Most of these pioneering techniques considered uncertainty sets of distributions  with instances' marginal that coincides with the empirical of training samples, and therefore they do not provide tight performance guarantees at learning.
%However, these methods do not provide tight performance guarantees at learning since they consider uncertainty sets of distributions with instances' marginal that coincides with the empirical. % In particular, the generalization bounds in \cite{FarTse:16} do not decrease with the number of training samples $n$ if the error of expectation estimates is of the usual order $\varepsilon=O(1/\sqrt{n})$. \color{blue} I am thinking that it may be better to remove this last sentence from here and join this paragraph with the next one, what do you think?\color{black}
Such performance guarantees are provided by PAC-Bayes methods \citep{AmbParSha:07,GerLacLavMarRoy:15,MhaGruLGue:19}  and \ac{RRM} techniques in transductive settings \citep{BalFre:15} or based on Wasserstein distances \citep{ShaKuhMoh:19,LeeRag:18}. PAC-Bayes methods consider specific classification rules such as margin-based and ensemble-based classifiers. In addition, the tightness of their performance bounds relies on that of multiple inequalities including Jensen's, Markov's, and Donsker-Varadhan's change of measure \citep{BegGerLavRoy:16}. Wasserstein-based \ac{RRM} techniques utilize surrogate losses and consider specific families of rules. In addition, the tightness of their performance bounds relies heavily on the adequacy of the Wasserstein radius used \citep{ShaKuhMoh:19,FroClaChiSol:21}.

This paper presents \acp{MRC} that minimize the worst-case \mbox{0\,-1 loss} over general classification rules and provide tight performance bounds at learning. Specifically, the main results presented in the paper are as follows.

\begin{itemize}
\item We develop learning methods that obtain classification rules with the smallest worst-case expected 0\,-1 loss with respect to uncertainty sets that can include the underlying distribution, with a tunable confidence (Section~\ref{sec-2.2}). 
In addition, we detail how to determine uncertainty sets from training data by estimating the expectation of a feature mapping and obtaining confidence vectors for such estimates (Section~\ref{sec-3}).
\item We characterize the performance guarantees of \acp{MRC} in terms of tight upper and lower bounds for the error probability and also in terms of generalization bounds with respect to the smallest minimax risk (Sections \ref{sec-4.1} and \ref{sec-4.2}). In addition, we show that \acp{MRC} are strongly universally consistent using feature mappings given by characteristic kernels (Section~\ref{sec-4.3}). 
\item We present efficient optimization techniques for \ac{MRC} learning that obtain the classifiers' parameters and the tight performance bounds using reduced sets of instances and efficient accelerated subgradient methods (Section~\ref{sec-5}). In addition, we quantify \acp{MRC}' performance with respect to existing techniques and show the suitability of the performance bounds presented (Section~\ref{sec-numerical}).
\end{itemize} 

Some of the results presented in this paper have appeared before in \cite{MazZanPer:20}. The main new results presented in this paper include: extension to instances' sets that are general Borel subsets; analysis of the generalization capabilities of the approach presented even in cases where the underlying distribution is not included in the uncertainty set considered; universal consistency of the methods proposed using rich feature mappings; and efficient optimization techniques based on accelerated subgradient methods. Specifically, the proofs of the theoretical results are extended to allow for infinite sets of instances instead of finite sets by using Fenchel duality instead of Lagrange duality. The \acp{MRC}' generalization bounds are extended in new Theorem~\ref{th_6} to cases when the uncertainty set does not include the underlying distribution, and new Theorem~\ref{th_7} shows the universal consistency of \acp{MRC} that use feature mappings given by characteristic kernels.  In addition, new Theorems \ref{th_8} and \ref{thm:comp.Tao-1} show that \ac{MRC} learning can be efficiently addressed using reduced sets of instances and subgradient methods that exploit the specific structure of the optimization problems in \ac{MRC} learning.

\emph{Notation:} calligraphic upper case letters denote sets, e.g., $\set{Z}$; $|\set{Z}|$ denotes de cardinality of set $\set{Z}$; vectors and matrices are denoted by bold lower and upper case letters, respectively, e.g., $\V{v}$ and $\V{M}$; for a vector $\V{v}$, $\|\V{v}\|$, $\|\V{v}\|_1$, and $\|\V{v}\|_\infty$ denote its L2-, L1-, and L-infinity norms, respectively,  $\V{v}^{\text{T}}$ denotes its transpose,  $v^{(i)}$ denotes its $i$-th component, $|\V{v}|$ and $\V{v}_+$ denote the vector given by the component-wise absolute value and positive part of $\V{v}$, respectively, and $\text{diag}(\V{v})$ denotes the diagonal matrix with diagonal given by $\V{v}$; probability distributions and classification rules are denoted by upright fonts, e.g., $\up{p}$ and $\up{h}$; $\mathbb{E}_{z\sim \up{p}}$ or simply $\mathbb{E}_{\up{p}}$ denotes the expectation w.r.t. probability distribution $\up{p}$ of random variable $z$; for a probability distribution $\up{p}$ over $\set{X}\times\set{Y}$, we denote by $\up{p}_x,\up{p}_y$ its corresponding marginals over $\set{X}$ and $\set{Y}$, respectively, and by $\up{p}_{x|y}$ the corresponding conditional distribution given $y\in\set{Y}$; $\preceq$ and $\succeq$ denote vector (component-wise) inequalities; $\V{1}$ denotes a vector with all components equal to $1$  and $\mathbb{I}\{\pun\}$ denotes the indicator function; finally, for a function of two variables $f:\set{X}\times\set{Y}\to\set{Z}$, $f(x,\pun)$ denotes the function $f(x,\pun):\set{Y}\to\set{Z}$ obtained by fixing $x\in\set{X}$.

%$\Delta(\set{Z})$ denotes the set of probability distributions with support $\set{Z}$

\section{Minimax risk classifiers}

This section first states the problem of supervised classification and summarizes different learning approaches, then we describe \acp{MRC} with 0\,-1 loss and their relationship with existing techniques.

\subsection{Problem formulation and learning approaches}\label{sec-2.1}

Classification techniques assign instances in a set $\set{X}$ to labels in a set $\set{Y}$, where $\set{X}$ is a Borel subset of $\mathbb{R}^d$ and $\set{Y}$ is a finite set represented by $\{1,2,\ldots,|\set{Y}|\}$. We denote by $\Delta(\set{X}\times\set{Y})$ the set of Borel probability measures $\up{p}$ on $\set{X}\times\set{Y}$, also referred to as probability distributions. Both randomized and deterministic classification rules  are given by functions from instances to probability distributions on labels (Markov transitions). We denote the set of all classification rules by $\text{T}(\set{X},\set{Y})$, and for $\up{h}\in  \text{T}(\set{X},\set{Y})$ we denote by $\up{h}(y|x)$ the probability with which instance $x\in\set{X}$ is classified by label $y\in\set{Y}$ ($\up{h}(y|x)\in\{0,1\}$ if labels are deterministically assigned to instances).

%In the following, we take both $\set{X}$ and $\set{Y}$ to be finite and represent $\set{Y}$ by $\{1,2,\ldots,|\set{Y}|\}$.\footnote{The finiteness assumption for the instances' set is not essential for the techniques presented and is made to avoid technical subtleties, which could be tackled by using minimax theorems for general metric spaces as in \cite{GruDaw:04} and Fenchel duality instead of Lagrangian duality as in \cite{AltSmo:06}. Note also that such assumption does not result in a loss of generality in practice since digital computers utilize finite-precision arithmetic.} We denote by $\Delta(\set{X}\times\set{Y})$ the set of probability distributions on $\set{X}\times\set{Y}$, and for $\up{p}\in \Delta(\set{X}\times\set{Y})$ we denote by $\up{p}(x,y)$ the probability assigned by $\up{p}$ to the instance-label pair $(x,y)$. Both randomized and deterministic classification rules  are given by functions from instances to probability distributions on labels (Markov transitions). We denote the set of all classification rules by $\text{T}(\set{X},\set{Y})$, and for $\up{h}\in  \text{T}(\set{X},\set{Y})$ we denote by $\up{h}(y|x)$ the probability with which instance $x\in\set{X}$ is classified by label $y\in\set{Y}$ ($\up{h}(y|x)\in\{0,1\}$ if labels are deterministically assigned to instances).

Supervised classification techniques use $n$ training instance-label pairs \linebreak \mbox{$(x_1,y_1),(x_2,y_2),\ldots,(x_n,y_n)$} from the underlying distribution $\up{p}^*$ to find classification rules that assign labels to instances with small error probability. If $\up{p}^*\in\Delta(\set{X}\times\set{Y})$ is the true underlying distribution of instance-label pairs, the error probability of a classification rule $\up{h}\in \text{T}(\set{X},\set{Y})$ is its expected 0\,-1 loss denoted by $R(\up{h})$, that is
$$R(\up{h})=\mathbb{E}_{\up{p}^*}\{\ell(\up{h},(x,y))\}$$
with $\ell(\up{h},(x,y))=1-\up{h}(y|x)$ the  0\,-1 loss of rule $\up{h}$ at instance-label pair $(x,y)$. In the following, we overload the notation for the loss function and denote the expected loss of $\up{h}$ with respect to probability distribution $\up{p}$ as $\ell(\up{h},\up{p})\coloneqq\mathbb{E}_{\up{p}}\{\ell(\up{h},(x,y))\}$. 

The optimal classification rule is the solution of the optimization problem
\begin{align}\label{opt-bayes}\mathscr{P}^{*}:\  \inf_{\up{h}\in \text{T}(\set{X},\set{Y})} \ell(\up{h},\up{p}^*).\end{align}
This solution is known as Bayes classification rule and the minimum value above is known as Bayes risk. Optimization problem $\mathscr{P}^{*}$ cannot be addressed in practice since the underlying distribution $\up{p}^*$ is unknown and only training samples $(x_1,y_1),(x_2,y_2),\ldots,(x_n,y_n)$ from $\up{p}^*$ are available. 

Existing supervised classification techniques can be interpreted as approximations of $\mathscr{P}^{*}$ by an optimization problem of the form
\begin{align}\label{opt-approx}\mathscr{P}:\  \adjustlimits\inf_{\up{h}\in \set{F}\,\,}\sup_{\up{p}\in\set{U}}\,\, L(\up{h},\up{p})\end{align}
for $\set{F}$ a family of classification rules, $\set{U}$ an uncertainty set of distributions, and $L$ a surrogate loss function. 

The \ac{ERM} approach considers uncertainty sets of distributions that contain only the empirical distribution of training samples. In such an approach, the approximation of $\mathscr{P}^{*}$ by $\mathscr{P}$ is controlled by the choice of the family $\set{F}\subset \text{T}(\set{X},\set{Y})$ through a bias-complexity trade-off addressed by \ac{SRM} \citep{Vap:98}. Families of rules reduced enough to provide uniform convergence of empirical averages can ensure that the objective function in $\mathscr{P}$ uniformly approximates that of $\mathscr{P}^{*}$, i.e., empirically averaged losses accurately approximate expected losses for all the rules considered. If the family of rules is also general enough to contain a classification rule near the Bayes rule, then the minimum value of $\mathscr{P}$ is similar to the minimum value of $\mathscr{P}^{*}$ and the \ac{ERM} approach leads to near-optimal performance.  This trade-off for the generality of the family of rules considered is usually controlled by selecting a parameter that determines the size of the family of rules, e.g., the radius of the \ac{RKHS} ball of functions that defines the family of rules.

The \ac{RRM} approach considers uncertainty sets of distributions determined by constraints obtained from training samples. In such an approach, the approximation of $\mathscr{P}^{*}$ by $\mathscr{P}$ is controlled by the choice of the uncertainty set $\set{U}\subset\Delta(\set{X}\times\set{Y})$. Uncertainty sets general enough to contain the underlying distribution can ensure that the objective function in $\mathscr{P}$ upper bounds that of $\mathscr{P}^{*}$, i.e., the worst-case expected loss upper bounds the actual expected loss for any classification rule. If the uncertainty set is also reduced enough to provide a tight upper bound, then the minimum value of $\mathscr{P}$ is similar to the minimum value of $\mathscr{P}^{*}$ and the \ac{RRM} approach leads to near-optimal performance. This trade-off for the generality of the uncertainty set considered is usually controlled by selecting a parameter that determines the size of the uncertainty set, e.g., the radius of the Wasserstein ball that defines the uncertainty set.

An important advantage of the RRM approach with respect to ERM is that the former does not require to constrain the classification rules considered in order to provide provable out-of-sample generalization. Such property is due to the fact that the optimum of $\mathscr{P}$ for \ac{RRM} upper bounds the expected loss with respect to any distribution in the uncertainty set, independently of the family of rules considered. On the other hand, the optimum of $\mathscr{P}$ for \ac{ERM} is ensured to be near the out-of-sample risk only if the family of classification rules considered provides uniform convergence of empirical averages. 

Optimization problem $\mathscr{P}$ not only has to reliably approximate $\mathscr{P}^*$ but also has to be tractable computationally. Such tractability is commonly achieved by 1) considering families of classification rules determined by certain parameters (e.g., weights of neurons in \acp{NN} or coefficients of linear classifiers), and 2) substituting the 0\,-1 loss $\ell$ by a surrogate loss $L$ (e.g., hinge-loss or logistic-loss). The usage of non-parametric classification rules can lead to huge-scale optimization problems (a general rule $\up{h}\in\text{T}(\set{X},\set{Y})$ is determined by $|\set{X}||\set{Y}|$ values), and the minimization of 0\,-1 loss is often NP-hard (see e.g., \cite{BenEirLon:03,FelGurRagWu:12}).

%the dimensionality of $\mathscr{P}$ the minimization of 0\,-1 loss is often NP-hard  

%The computational tractability of optimization problem $\mathscr{P}$ is commonly achieved by considering families of classification rules determined by certain parameters such as the weights of neurons in \acp{NN} and the coefficients of linear combinations in \acp{RKHS}. The ensuing optimization of such parameters cannot be addressed using 0\,-1 loss (see e.g., \cite{ShaBen:14,BarJorMca:06}), so that conventional techniques substitute 0\,-1 loss by a surrogate loss such as hinge-loss or logistic-loss. \color{blue} An important advantage of the \ac{RRM} approach with respect to \ac{ERM} is that the former does not require to constrain the classification rules considered in order to ensure provable out-of-sample generalization. Such property is due to the fact that the optimum of $\mathscr{P}$ for \ac{RRM} upper bounds the expected loss with respect to any distribution in the uncertainty set independently of the family of rules considered. Therefore, \ac{RRM} techniques can consider non-parametric classification rules and optimize 0\,-1 loss. \color{black}On the other hand, the optimum of $\mathscr{P}$ for \ac{ERM} is ensured to be near the out-of-sample risk only if the family of classification rules considered provide uniform convergence of empirical averages.

In the following we show how the proposed \acp{MRC} approximate the optimization problem $\mathscr{P}^*$ in \eqref{opt-bayes} by an optimization problem of the form 
\begin{align}\label{opt-mrc}\mathscr{P}_{\text{MRC}}:\  \adjustlimits\inf_{\up{h}\in \text{T}(\set{X},\set{Y})\,\,} \sup_{\up{p}\in\set{U}}\,\,\ell(\up{h},\up{p})\end{align}
for an uncertainty set $\set{U}$ that can include the underlying distribution with a tunable confidence. \acp{MRC} do not rely on a choice of surrogate loss and family of rules. The only change in $\mathscr{P}_{\text{MRC}}$ with respect to $\mathscr{P}^{*}$ consists on using an uncertainty set $\set{U}$ instead of the underlying distribution $\up{p}^*$. 

Certain previous \acp{RRM} methods also address an optimization problem of the form \eqref{opt-mrc} \citep{AsiXinBeh:15,FatAnq:16}. However, the uncertainty sets considered by such methods only include distributions with instances' marginals that coincide with the empirical. With such an additional constraint for the uncertainty set, these approaches become equivalent to \ac{ERM} with a surrogate loss referred to as adversarial zero-one loss. Therefore, their out-of-sample generalization properties are akin to those of other \ac{ERM} methods and cannot exploit the \ac{RRM}'s benefits of approximating $\mathscr{P}^{*}$ by an upper bound.  

\subsection{\acp{MRC} with 0\,-1 loss}\label{sec-2.2}

The classification loss $\ell(\up{h},(x,y))$ of rule $\up{h}$ at instance-label pair $(x,y)$ quantifies the loss of the rule evaluated at instance $x\in\set{X}$ when the label is $y\in\set{Y}$. In this paper we consider \mbox{0\,-1 loss} that is given by  $\ell(\up{h},(x,y))=1-\up{h}(y|x)$, while \acp{MRC} for general loss functions are described in \citep{MazShePer:22}. The usage of 0\,-1 loss is specially suitable for discriminative approaches since it quantifies the classification error, while other loss functions such as logistic loss can be more suitable for conditional probability estimation since they score probability assessments. Specifically, if $\up{p}^*$ is the true underlying distribution of instance-label pairs, the expected 0\,-1 loss $\ell(\up{h},\up{p}^*)$, also referred to as the risk $R(\up{h})$, coincides with the error probability of classification rule $\up{h}$. 

The proposed \acp{MRC} minimize the worst-case expected 0\,-1 loss with respect to distributions in uncertainty sets that can contain the true underlying distribution with a tunable confidence. These uncertainty sets are given by constraints on the expectations of a vector-valued bounded and Borel measurable function $\Phi:\set{X}\times\set{Y}\to\mathbb{R}^m$ referred to as feature mapping, e.g., multiple polynomials on $x$ and $y$ or one-hot encodings of the last layers in an \ac{NN}. Such mappings are commonly used in machine learning to represent instance-label pairs as real vectors (see e.g., \cite{MehRos:18,BenCouVin:13} and next Section~\ref{sec-features}).

In the following we consider uncertainty sets
\begin{align}\label{uncertainty}\set{U}=\Big\{\up{p}\in\Delta(\set{X},\set{Y}):\ |\mathbb{E}_{\up{p}}\{\Phi\} -\B{\tau}|\preceq \B{\lambda}\Big\}\end{align}
given by a feature mapping $\Phi$ together with mean and confidence vectors $\B{\tau}\in\mathbb{R}^m$ and $\B{\lambda}\in\mathbb{R}^m$, and satisfying one of the following regularity conditions.
\begin{itemize}\itemsep0em
\item[R1] The set $\set{X}$ is finite and there exists a probability measure $\up{p}$ in $\set{U}$.
%\item[R2] There exists $\up{p}\in\set{U}$ such that $|\mathbb{E}_{\up{p}}\{\Phi\} -\B{\tau}|\prec \B{\lambda}$.
\item[R2] There exists a probability measure $\up{p}$ in $\set{U}$ such that $|\mathbb{E}_{\up{p}}\{\Phi(x,y)^{(i)}\} -\B{\tau}^{(i)}|<\lambda^{(i)}$ for any \mbox{$i\in\{1,2,\ldots,m\}$} with $\lambda^{(i)}>0$ and
\begin{itemize}\itemsep0em
\item[R2.1] $\lambda^{(i)}>0$ for all $i\in\{1,2,\ldots,m\}$ or
\item[R2.2]  the support of the r.v. $\Phi(x,y)$ for $(x,y)\sim \up{p}$ is not contained in a proper affine subspace of $\mathbb{R}^m$.
\end{itemize}
\end{itemize}
%$|\mathbb{E}_{\up{p}}\{\Phi(x,y)\} -\B{\tau}|\prec\B{\lambda}$ or
%\item[R2.2]  the support of the r.v. $\Phi(x,y)$ for $(x,y)\sim \up{p}$ is not contained in a proper affine subspace of $\mathbb{R}^m$, and for any \mbox{$i\in\{1,2,\ldots,m\}$} with $\lambda^{(i)}>0$ we have that $|\mathbb{E}_{\up{p}}\{\Phi(x,y)^{(i)}\} -\B{\tau}^{(i)}|<\lambda^{(i)}$.
%\end{itemize}
%\end{itemize}

%$\B{\lambda}\succ\V{0}$ and we have that 

These regularity conditions are utilized to ensure strong duality holds for the inner maximization in the minimax problems, and they are satisfied with wide generality. For instance, such conditions are satisfied if there exists a probability measure $\up{p}$ such that $\mathbb{E}_\up{p} \{\Phi\}=\B{\tau}$ and $\{\Phi(x_i,y_i), i=1,2,\ldots,n\}$ is not contained in a proper subspace of $\mathbb{R}^m$. In addition, if those $n$ points are contained in a proper affine subspace $\Gamma\subset\mathbb{R}^m$ with $\text{dim}\, \Gamma=m'<m$ the feature mapping could be easily modified so that R2.2 is satisfied for instance by using $$\Phi'(x,y)=\big((\Phi(x,y)-\V{z}_0)^\text{T}\V{v}_1,(\Phi(x,y)-\V{z}_0)^\text{T}\V{v}_2,\ldots,(\Phi(x,y)-\V{z}_0)^\text{T}\V{v}_{m'}\big)\in\mathbb{R}^{m'}$$ where $\Gamma=\{\V{z}_0+\alpha_1\V{v}_1+\alpha_2\V{v}_2+\ldots+\alpha_{m'}\V{v}_{m'}; \alpha_1,\alpha_2,\ldots,\alpha_{m'}\in\mathbb{R}\}$.

% R2.1 is satisfied if $\B{\lambda}\succ\V{0}$ and there exists a probability measure $\up{p}$ such that $\mathbb{E}_\up{p} \{\Phi\}=\B{\tau}$, and R2.2 is satisfied for properly chosen feature mappings,  for instance if $\{\Phi(x_i,y_i), i=1,2,\ldots,n\}$ is not contained in a proper subspace of $\mathbb{R}^m$. 

The mean vector $\B{\tau}\in\mathbb{R}^m$ in \eqref{uncertainty} is an estimate of the feature mapping expectation $\mathbb{E}_{\up{p}^*}\{\Phi\}$ with respect to the underlying distribution. In this paper, we consider expectation estimates obtained as the sample average
\begin{align}\label{tau}
\B{\tau}_n=\frac{1}{n}\sum_{i=1}^n\Phi(x_i,y_i)
\end{align}
obtained from the $n$ training samples $(x_1,y_1),(x_2,y_2),\ldots,(x_n,y_n)$ that are assumed to be independent samples from the underlying distribution $\up{p}^*$. Nevertheless, it is important to note that most of the results presented in the paper as well as the general methodology proposed can be utilized with general types of expectation estimates. Alternative estimators for the expectation can be preferred in several practical scenarios including cases with different distributions at training and test  \citep{MohMun:12,MazPer:20,AlvMazLoz:22} and cases where the distribution of features has heavy tails \citep{LugMen:19,HsuSab:16}.

%$n$ training instance-label pairs $(x_1,y_1),(x_2,y_2),\ldots,(x_n,y_n)$ that are assumed to be independent samples from  the underlying distribution $\up{p}^*$. Nevertheless, it is important to note that most of the results presented in the paper as well as the general methodology proposed can be utilized with general types of expectation estimates. Alternative estimators for the expectation can be preferred in several practical scenarios including cases with different distributions at training and test  \citep{MohMun:12,MazPer:20} and cases where the distribution of features has heavy tails \citep{LugMen:19,HsuSab:16}.

The confidence vector $\B{\lambda}\in\mathbb{R}^m$ in \eqref{uncertainty} is an estimate of the mean vector component-wise accuracy $|\mathbb{E}_{\up{p}^*}\{\Phi\}-\B{\tau}|$, and controls the size of the uncertainty set considered.  In particular, if $\B{\lambda}$ corresponds to the length of confidence intervals at level $1-\delta$ centered at $\B{\tau}$, then the true underlying distribution is included in the uncertainty set \eqref{uncertainty} with probability at least $1-\delta$. Section~\ref{sec-lambda} describes how to choose such confidence vectors for different feature mappings.

Classification rules that minimize the worst-case error probability over uncertainty sets $\set{U}$ given by \eqref{uncertainty} are referred to as \acp{MRC}.
\begin{definition}
%Let $\Phi:\set{X}\times\set{Y}\to\mathbb{R}^m$ be a feature mapping and $\set{U}$ be the uncertainty set
%\begin{align}\label{uncertainty}\set{U}=\Big\{\up{p}\in\Delta(\set{X},\set{Y}):\ |\mathbb{E}_{\up{p}}\{\Phi\} -\B{\tau}|\preceq \B{\lambda}\Big\}\end{align}
%where $\B{\tau},\B{\lambda}\in\mathbb{R}^m$. 
We say that a classification rule $\up{h}^{\sset{U}}$ is a 0\,-1 \ac{MRC} for uncertainty set $\set{U}$ if 
$$\up{h}^{\sset{U}}\in\arg\adjustlimits\inf_{\up{h}\in \text{T}(\set{X},\set{Y})\,\,}\sup_{\up{p}\in\set{U}}\,\,\ell(\up{h},\up{p})$$
and we denote by $\overline{R}(\set{U})$ the minimax risk against $\set{U}$, i.e., 
 $$\overline{R}(\set{U})=\adjustlimits\inf_{\up{h}\in \text{T}(\set{X},\set{Y})\,\,}\sup_{\up{p}\in\set{U}}\,\,\ell(\up{h},\up{p}).$$
\end{definition}

The following result shows how 0\,-1 \acp{MRC} can be determined by a linear-affine combination of the feature mapping. The coefficients of such combination can be obtained at learning by solving the convex optimization problem
\begin{align}\label{opt-prob}\begin{array}{lcl}\mathscr{P}_{\B{\tau},\B{\lambda}}:
&\underset{\B{\mu}}{\min}&1-\B{\tau}^{\text{T}}\B{\mu}+\varphi(\B{\mu})+\B{\lambda}^{\text{T}}|\B{\mu}|\end{array}\end{align}
where\footnote{Here, as in the sequel, we implicitly assume that the set $\set{C}\subseteq\set{Y}$ over which we take the maximum excludes the empty set.}
\begin{align}\label{varphi}\varphi(\B{\mu})=\sup_{x\in\set{X},\set{C}\subseteq\set{Y}}\frac{\sum_{y\in\set{C}}\Phi(x,y)^{\text{T}}\B{\mu}-1}{|\set{C}|}.
\end{align}
\begin{theorem}\label{th1}
Let $\set{U}$ be an uncertainty set as in \eqref{uncertainty} that satisfies R1 or R2 and $\B{\mu}^*$ be a solution of optimization problem $\mathscr{P}_{\B{\tau},\B{\lambda}}$.
%$\B{\tau},\B{\lambda}\in\mathbb{R}^m$ be such that the corresponding uncertainty set $\set{U}$ in \eqref{uncertainty} is not empty and $\B{\mu}^*$ be a solution of optimization problem $\mathscr{P}_{\B{\tau},\B{\lambda}}$.
If a classification rule $\up{h}^{\sset{U}}\in \text{T}(\set{X},\set{Y})$ satisfies
\begin{align}\label{opt-sol}\up{h}^{\sset{U}}(y|x)\geq\Phi(x,y)^{\text{T}}\B{\mu}^*-\varphi(\B{\mu}^*),\  \forall x\in\set{X}, y\in\set{Y}\end{align}
then $\up{h}^{\sset{U}}$ is a 0\,-1 \ac{MRC} for $\set{U}$. In addition, we have that 
\begin{align}\label{minimax}\overline{R}(\set{U})=1-\B{\tau}^{\text{T}}\B{\mu}^*+\varphi(\B{\mu}^*)+\B{\lambda}^{\text{T}}|\B{\mu}^*|.\end{align}
\end{theorem}
\begin{proof}
See Appendix~\ref{apd:proof_th_1}.
\end{proof}

A classification rule satisfying \eqref{opt-sol} always exists because for every $x\in\set{X}$ the sum over $y\in\set{Y}$ of the positive part of the right hand side of \eqref{opt-sol} is not larger than one. Specifically, for any $x\in\set{X}$, the number
\begin{align}\label{cx}c_x=\sum_{y\in\set{Y}}\big(\Phi(x,y)^{\text{T}}\B{\mu}^*-\varphi(\B{\mu}^*)\big)_+\end{align}
is zero or equal to
$$\max_{\set{C}\subseteq\set{Y}}\sum_{y\in\set{C}}\Phi(x,y)^{\text{T}}\B{\mu}^*-|\set{C}|\varphi(\B{\mu}^*)$$
that is not larger than one by definition of $\varphi(\pun)$ in \eqref{varphi}.

The characterization of \acp{MRC} in terms of the inequality in \eqref{opt-sol} may seem counter-intuitive because for some pairs $(x,y)$ such classification rules are not uniquely determined. However, such pairs do not affect the worst-case error probability. Specifically, the set of pairs $(x,y)$ where \eqref{opt-sol} is not satisfied with equality has zero probability under a worst-case distribution in $\set{U}$ because for $\up{p}^{\sset{U}}\in\set{U}$ 
\begin{align*}&\adjustlimits\inf_{\up{h}\in\text{T}(\set{X},\set{Y})\,\,}\sup_{\up{p}\in\set{U}}\,\,\ell(\up{h},\up{p})=\ell(\up{h}^{\sset{U}},\up{p}^{\sset{U}})\\&\hspace{1cm}\Rightarrow 1-\int\up{h}^{\sset{U}}(y|x)\text{d}\up{p}^{\sset{U}}(x,y)=1-\B{\tau}^{\text{T}}\B{\mu}^*+\B{\lambda}^{\text{T}}|\B{\mu}^*|+\varphi(\B{\mu}^*)\\&\hspace{1cm}
\Rightarrow \int\up{h}^{\sset{U}}(y|x)-\big(\Phi(x,y)^{\text{T}}\B{\mu}^*-\varphi(\B{\mu}^*)\big)\text{d}\up{p}^{\sset{U}}(x,y) =\big(\B{\tau}-\mathbb{E}_{\up{p}^{\sset{U}}}\Phi(x,y)\big)^{\text{T}}\B{\mu}^*-\B{\lambda}^{\text{T}}|\B{\mu}^*|\\&\hspace{1cm}
\Rightarrow \int\up{h}^{\sset{U}}(y|x)-\big(\Phi(x,y)^{\text{T}}\B{\mu}^*-\varphi(\B{\mu}^*)\big)\text{d}\up{p}^{\sset{U}}(x,y)=0.
\end{align*}
since $\up{p}^{\sset{U}}\in\set{U}\Rightarrow \big(\B{\tau}-\mathbb{E}_{\up{p}^{\sset{U}}}\Phi(x,y)\big)^{\text{T}}\B{\mu}^*-\B{\lambda}^{\text{T}}|\B{\mu}^*|\leq 0$.

In order to avoid ambiguities, we will refer to 0\,-1 \ac{MRC} for uncertainty set $\set{U}$ as the classification rule $\up{h}^{\sset{U}}$ obtained by normalizing the positive part of the right hand side of \eqref{opt-sol}. Specifically, let for $x\in\set{X}$
\begin{align}\label{MRC}
\up{h}^{\sset{U}}(y|x)=\left\{\begin{array}{cc}(\Phi(x,y)^{\text{T}}\B{\mu}^*-\varphi(\B{\mu}^*))_+/c_x&\mbox{if }c_x\neq 0\\1/|\set{Y}|&\mbox{if }c_x=0\end{array}\right.
\end{align}
with $c_x$ given by \eqref{cx}.
%with $c=\sum_{y\in\set{Y}}(\Phi(x,y)^{\text{T}}\B{\mu}^*-\varphi(\B{\mu}^*))_+$. 
Such classification rule is univocally determined by $\B{\mu}^*$ and satisfies \eqref{opt-sol} because $c_x\leq 1$ for any $x\in\set{X}$.  
In addition, we denote by $\up{h}^{\sset{U}}_{\text{d}}$ the deterministic classification rule corresponding to $\up{h}^{\sset{U}}$, that is
\begin{align}\label{det-MRC}\up{h}^{\sset{U}}_{\text{d}}(y|x)=\mathbb{I}\{y\in\arg\max\up{h}^{\sset{U}}(\pun|x)\}=\mathbb{I}\{y\in\arg\max\Phi(x,\pun)^{\text{T}}\B{\mu}^*\}\end{align}
where a tie in the $\arg\max$ above can be resolved arbitrarily. In the following, we will refer to such classification rule $\up{h}^{\sset{U}}_{\text{d}}$ as the deterministic 0\,-1 \ac{MRC} for $\set{U}$.

The above Theorem~\ref{th1} provides a representer theorem for \acp{MRC}. The classical representer theorem for \ac{ERM} over \acp{RKHS} (see e.g., \cite{MehRos:18,EvgPonPog:00}) states that the solution of \ac{ERM} over the set of classification rules given by functions in an \ac{RKHS} ball is determined by a linear combination of $n$ functions corresponding with the training samples. Such result enables to address the minimization of empirical loss over an infinite-dimensional \ac{RKHS} because it becomes equivalent to a minimization over $n$ parameters.  Analogously, Theorem~\ref{th1} states that the solution of $\mathscr{P}_{\text{MRC}}$ over the set of all classification rules is determined by a linear-affine combination of the $m$ components of the feature mapping. Even if the optimization problem addressed by \acp{MRC} does not impose constraints on the classification rules considered, the feature mapping utilized determines the parametric form of its solutions. Theorem~\ref{th1} enables to address the minimization of the worst-case error probability over general classification rules since it becomes equivalent to a minimization over the $m$ parameters $\B{\mu}$. As shown  in Section~\ref{sec-5} below, optimization \eqref{opt-prob} can be efficiently addressed in practice.  In particular, Theorem~\ref{th_8} shows that the set $\set{X}$ in \eqref{varphi} can be substituted by a reduced set of instances such as that formed by the instances at training.

Note also that \acp{MRC} are often given by sparse combinations of the components of the feature mapping since the last term in optimization problem \eqref{opt-prob} imposes an L1-type regularization for parameters. L1-norm regularization is broadly used in machine learning (see e.g., \cite{HasTibWai:19,MehRos:18,DemVitRos:09}) and the regularization parameter used to weight the L1-norm is commonly obtained by cross-validation methods. The result in Theorem~\ref{th1} can directly provide appropriate regularization parameters from the length of the expectations' interval estimates. In addition, the regularization term $\B{\lambda}^{\text{T}}|\B{\mu}|=\sum_{i=1}^m\lambda^{(i)}|\mu^{(i)}|$ in \eqref{opt-prob} allows to penalize differently each component of $\B{\mu}$. Such type of L1-norm regularization is usually referred to as adaptive or weighted, and has shown to significantly improve performance \citep{Zou:06,CanWakBoy:08}. For \acp{MRC}, the  regularization term causes that feature components with poorly estimated expectations (i.e., components $\Phi^{(i)}$ with large $\lambda^{(i)}$) have a reduced or null influence on the final classification rule.

\subsection{\acp{MRC} with 0\,-1 loss and fixed marginals}

%The following result \color{blue} shows the difference between the \acp{MRC} shown above and the existing \ac{RRM} methods that utilize 0-1 loss. Specifically, the usage of uncertainty sets 

The following result shows how existing \ac{RRM} methods that utilize 0-1 loss correspond to \acp{MRC} that use uncertainty sets of distributions with instances' marginal given by the empirical distribution. In particular, such \acp{MRC} correspond to L1-regularized \ac{ERM} with a loss referred to as adversarial zero-one in \cite{FatAnq:16} or as minimax hinge in \cite{FarTse:16}.

%given by additional constraints that fix the instances' marginal to coincide with the empirical marginal. In particular,  

%existing methods describes \acp{MRC} that consider uncertainty sets of distributions given by additional constraints that fix the instances' marginal distribution. Using the empirical marginal distribution, such \acp{MRC} correspond to L1-regularized \ac{ERM} with a loss referred to as adversarial zero-one in \cite{FatAnq:16} or as minimax hinge in \cite{FarTse:16}.

For the following result we consider uncertainty sets of the form 
\begin{align}\label{uncertainty-emp}\set{V}=\Big\{\up{p}\in\Delta(\set{X},\set{Y}):\ |\mathbb{E}_{\up{p}}\{\Phi\} -\B{\tau}|\preceq \B{\lambda}\mbox{ and }\up{p}_x=\up{p}^n_x\Big\}\end{align} 
where $\up{p}^n_x$ is the uniform distribution over instances with support $\{x_1,x_2,\ldots,x_n\}\subset\set{X}$ for $n$ specific instances.

\begin{theorem}\label{th_empirical}

Let $\B{\tau},\B{\lambda}\in\mathbb{R}^m$ be such that the uncertainty set $\set{V}$ in \eqref{uncertainty-emp} is not empty.
If $\B{\mu}^*$ is a solution of the optimization problem
\begin{align}\label{opt-emp}\min_{\B{\mu}}1-\B{\tau}^{\text{T}}\B{\mu}+\frac{1}{n}\sum_{i=1}^n\upvarphi(\B{\mu},x_i)+\B{\lambda}^{\text{T}}|\B{\mu}|\end{align}
where
$$\upvarphi(\B{\mu},x)=\max_{\set{C}\subseteq\set{Y}}\frac{\sum_{y\in\set{C}}\Phi(x,y)^{\text{T}}\B{\mu}-1}{|\set{C}|}$$
and $\up{h}^{\sset{V}}$ is the classification rule
\begin{align}\label{sol-emp}\up{h}^{\sset{V}}(y|x)= \big(\Phi(x,y)^{\text{T}}\B{\mu}^*-\upvarphi(\B{\mu}^*,x)\big)_+, \  \forall x\in\set{X}, y\in\set{Y}\end{align} then, $\up{h}^{\sset{V}}$ is a 0\,-1 \ac{MRC} for $\set{V}$, that is
$$\up{h}^{\sset{V}}\in\arg\adjustlimits\inf_{\up{h}\in \text{T}(\set{X},\set{Y})\,\,}\sup_{\up{p}\in\set{V}}\,\,\ell(\up{h},\up{p}).$$
\end{theorem} 
\begin{proof}
See Appendix~\ref{apd:proof_th_1_1}.
\end{proof}

If the instances $\{x_1,x_2,\ldots,x_n\}$ are those obtained at training, the 0\,-1 \acp{MRC} for uncertainty sets given by both expectations and marginals constrains as in \eqref{uncertainty-emp} 
correspond to existing techniques known as maximum entropy machines \citep{FarTse:16} or zero-one adversarial \citep{FatAnq:16}.  Specifically, if $\{(x_i,y_i)\}_{i=1}^n$ are training samples and $\B{\tau}$ is given by \eqref{tau}, then optimization problem \eqref{opt-emp} becomes
\begin{align*}\min_{\B{\mu}}\frac{1}{n}\sum_{i=1}^n\max_{\set{C}\subseteq\set{Y}}\frac{\sum_{y\in\set{C}}(\Phi(x_i,y)-\Phi(x_i,y_i))^{\text{T}}\B{\mu}+|\set{C}|-1}{|\set{C}|}+\B{\lambda}^{\text{T}}|\B{\mu}|\end{align*}
that corresponds to L1-regularized \ac{ERM} with a loss referred to as minimax hinge in \cite{FarTse:16} or as
adversarial zero-one in \cite{FatAnq:16}.

Uncertainty sets $\set{V}$ given by \eqref{uncertainty-emp} do not contain the true underlying distribution for a finite number of training samples since such sets only contain distributions with instances' marginal that is uniform over the training instances. Hence, the usage of such uncertainty sets does not allow to obtain the performance guarantees shown in Section~\ref{sec-guarantees} below for uncertainty sets given by \eqref{uncertainty}.

\section{Uncertainty sets of distributions}\label{sec-3}

As described above, \ac{MRC}  classification rules have the smallest worst-case error probability over distributions in an uncertainty set. This set is determined by expectations estimates of a feature mapping. In this section we first describe common feature mappings and then characterize the accuracy of expectation estimates obtained from training samples.

\subsection{Feature mappings}\label{sec-features}

Most of the results presented in the paper are valid for general feature mappings ${\Phi:\set{X}\times\set{Y}\to\mathbb{R}^m}$. In this section we describe feature mappings that are common in supervised classification and will be used later in the paper. Such feature mappings are given by real-valued functions over the set of instances $\psi:\set{X}\to\mathbb{R}$, referred to as scalar features. The most common and simple way to define feature mappings over $\set{X}$ and $\set{Y}$ is to use multiple scalar features over $\set{X}$ together with a one-hot encoding of the elements of $\set{Y}$ \citep{TsoJoaHofAlt:05,CraDekKesShaSin:06,MehRos:18} as follows 
\begin{align}\label{one-hot}\Phi(x,y)=[\mathbb{I}\{y=1\}\Psi(x)^{\text{T}},\mathbb{I}\{y=2\}\Psi(x)^{\text{T}},\ldots,\mathbb{I}\{y=|\set{Y}|\}\Psi(x)^{\text{T}}]^{\text{T}}=\V{e}_y\otimes \Psi(x)\end{align}
where $\Psi(x)=[\psi_1(x),\psi_2(x),\ldots,\psi_D(x)]^{\text{T}}$ for $D$ scalar features $\psi_1,\psi_2,\ldots,\psi_D$, $\V{e}_y$ is the $y$-th vector in the standard basis of $\mathbb{R}^{|\set{Y}|}$, and $\otimes$ denotes the Kronecker product. The feature mapping $\Phi$ represents each instance-label pair $(x,y)$ by an $m$-dimensional real vector with $m=|\set{Y}|D$, so that $\Phi(x,y)$ is composed by $|\set{Y}|$ $D$-dimensional blocks with values $\Psi(x)$ in the block corresponding to $y$ and zero otherwise.
Other feature mappings can exploit relationships among classes as described for instance in \cite{BakHofSchSmo:07,CapMicPonYin:08}.

Multiple types of scalar features are commonly used in supervised classification including those given by thresholds/decision stumps \citep{LebLaf:01}, last layers in an \ac{NN} \citep{BenCouVin:13}, and random features corresponding to a \ac{RKHS} \citep{RahRec:08}. Most of the results shown in this paper are valid for general features, while for those showing the universal consistency of \acp{MRC} we use random features that embed instances into rich feature spaces corresponding to \acp{RKHS}.

\subsection{Confidence vectors for expectation estimates}\label{sec-lambda}

In this section, we describe conditions for confidence vectors $\B{\lambda}$ that lead to uncertainty sets given by \eqref{uncertainty} that contain the true underlying distribution with high probability. Specifically, we consider uncertainty sets $\set{U}$ given by feature mappings defined by \eqref{one-hot} using scalar features in a family $\set{F}$ together with mean vectors obtained as in \eqref{tau} using the sample average of $n$ independent samples from the underlying distribution $\up{p}^*$. 

We denote by $\B{\lambda}_\delta$ any confidence vector that has coverage probability at least $1-\delta$, i.e., $\mathbb{P}\{|\mathbb{E}_{\up{p}^*}\{\Phi\}-\B{\tau}|\preceq\B{\lambda}_\delta\}\geq 1-\delta$, so that the underlying distribution $\up{p}^*$ is included with probability at least $1-\delta$ in the uncertainty set given by $\B{\lambda}_\delta$. The following result describes such confidence vectors in terms of the cardinality $|\set{F}|$, the empirical variance of the feature mapping, and the Rademacher complexity $\set{R}_n(\set{F})$. In most of the results in the paper, we utilize feature mappings given by scalar features chosen before the training samples are observed. In such cases, appropriate confidence vectors are given in terms of the cardinality $|\set{F}|$, i.e., the number of scalar features used. However, the methods proposed can also be used with data-dependent feature mappings. In such cases, appropriate confidence vectors are given in terms of the Rademacher complexity $\set{R}_n(\set{F})$ where $\set{F}$ is the family of candidate scalar features. For instance, $\set{F}$ can be family of all decision stumps and the feature mapping can be defined using the decision stumps found by one-dimensional decision trees learned using the training samples, as in \cite{MazZanPer:20}.

%In cases where the scalar features are selected using the training samples 

%The results given in terms of Rademacher complexity are of interest in cases where the cardinality of $\set{F}$ is allowed to be large or even infinite. Such families $\set{F}$ can enable to use data-based feature mappings in which the $D$ scalar features defining the feature mapping are selected from $\set{F}$ using the training samples.

\begin{theorem}\label{th2}
Let the mean vector be given by $\B{\tau}=\B{\tau}_n$ as in \eqref{tau} for $n$ training samples, and $\set{F}$ be a family of bounded scalar features, that is $|\psi(x)|\leq C$ for all $\psi\in\set{F}$ and $x\in\set{X}$.

If $|\set{F}|<\infty$,  $\lambda_\delta^{(i)}$ for $i=1,2,\ldots,m$ can be taken as
 \begin{align}\lambda_\delta^{(i)}= C\sqrt{\frac{2\log 2|\set{F}||\set{Y}|/\delta}{n}}.\end{align}
In addition, if $\upsilon_n^{(i)}$ is the sample variance of the $i$-th component of $\Phi$, $\lambda_\delta^{(i)}$ for $i=1,2,\ldots,m$ can be taken as
\begin{align}\label{bernstein}\lambda_\delta^{(i)}=2C\sqrt{\frac{2\upsilon_n^{(i)}\log 4 |\set{F}||\set{Y}|/\delta}{n}}+\frac{14C\log 4|\set{F}||\set{Y}|/\delta}{3(n-1)}.\end{align}

If the Rademacher complexity of $\set{F}$ satisfies $\set{R}_n(\set{F})\leq R/\sqrt{n}$,  $\lambda_\delta^{(i)}$ for $i=1,2,\ldots,m$ can be taken as
\begin{align}\lambda_\delta^{(i)}= 2\sqrt{\frac{n_{j(i)}}{n}}\frac{R}{\sqrt{n}}+C\left(1+2\sqrt{\frac{n_{j(i)}}{n}}\right)\sqrt{\frac{\log 4|\set{Y}|/\delta}{2n}}\end{align}
where $j(i)$ denotes the label for which the $i$-th component of $\Phi$ is non-zero, and $n_j$ denotes the number of samples with label $j\in\set{Y}$.

\end{theorem}
\begin{proof}
See Appendix~\ref{apd:proof_th_2}.
\end{proof}

The previous result shows how to obtain uncertainty sets $\set{U}$ as in \eqref{uncertainty} that contain the true underlying distribution with high probability.  Data-based tight confidence vectors can be obtained using sample standard deviations as shown in \eqref{bernstein}. Note that sample standard deviations can also be used to determine confidence vectors for infinite families of scalar features by using their growth function \citep{MauPon:09}.

Theorem~\ref{th2} shows that confidence vectors $\B{\lambda}$ can decrease at a rate $\set{O}(1/\sqrt{n})$ using general bounded features. Such rate can be obtained similarly with unbounded sub-Gaussian features \citep{Wai:19} and also with heavy-tailed features by using robust estimators for expectations \citep{LugMen:19}. In addition, the order $\set{O}(\sqrt{\log(|\set{F}|)/n})$ for $\B{\lambda}$ in the above result is consistent with that shown to provide consistent estimators using L1-regularization \citep{BuhGee:11}.

\section{Performance guarantees}\label{sec-guarantees}

This section characterizes the out-of-sample performance of 0\,-1 \acp{MRC}. We first present techniques that provide tight performance bounds at learning, then we show finite-sample generalization bounds for \acp{MRC}. In particular, we show that the proposed techniques can provide strong universal consistency using rich feature representations.

\subsection{Tight performance bounds}\label{sec-4.1}
The following result shows that the proposed approach allows to obtain upper and lower bounds for expected losses by solving convex optimization problems.
\begin{theorem}\label{th-bounds}
Let $\set{U}$ be an uncertainty set given by \eqref{uncertainty}, $\up{h}$ be any classification rule, and $\underline{R}(\set{U},\up{h})$ and $\overline{R}(\set{U},\up{h})$ be given by
\begin{align}
\underline{R}(\set{U},\up{h})=&\sup_{\B{\mu}}\,\,1-\B{\tau}^\text{T}\B{\mu}+\inf_{x\in\set{X},y\in\set{Y}}\{\Phi(x,y)^{\text{T}}\B{\mu}-\up{h}(y|x)\}-\B{\lambda}^{\text{T}}|\B{\mu}|\label{opt-lower}\\
\overline{R}(\set{U},\up{h})=&\inf_{\B{\mu}}\,\,\,1-\B{\tau}^\text{T}\B{\mu}+\sup_{x\in\set{X},y\in\set{Y}}\{\Phi(x,y)^{\text{T}}\B{\mu}-\up{h}(y|x)\}+\B{\lambda}^{\text{T}}|\B{\mu}|.\label{opt-upper}
\end{align}
Then, for any $\up{p}\in\set{U}$ we have that $\underline{R}(\set{U},\up{h})\leq\ell(\up{h},\up{p})\leq\overline{R}(\set{U},\up{h})$. In addition, if $\set{U}$ satisfies R1 or R2 the optimal values in \eqref{opt-lower} and \eqref{opt-upper} are attained.% and $\underline{R}(\set{U},\up{h})\geq 0$ and 
\end{theorem}
\begin{proof}
See Appendix~\ref{apd:proof_th_3}.
\end{proof}

The techniques proposed in \citep{DucGlyNam:16,AbaMohKuh:15,ShaKuhMoh:19}
obtain upper and lower bounds corresponding with \ac{RRM} methods that use uncertainty sets defined in terms of f-divergences and Wasserstein distances. Such methods obtain classification rules by minimizing the upper bound of a surrogate expected loss while \acp{MRC} minimize the upper bound of the 0\,-1 expected loss (error probability).
The bounds in \citep{DucGlyNam:16,AbaMohKuh:15,ShaKuhMoh:19} as well as those in Theorem~\ref{th-bounds}  become bounds for the true expected loss (risk) if the uncertainty set considered includes the true underlying distribution. Such situation can be attained with a tunable confidence using uncertainty sets defined by Wasserstein distances as in \citep{AbaMohKuh:15,ShaKuhMoh:19} or using the proposed uncertainty sets in \eqref{uncertainty} with expectation confidence intervals. 

The above theorem provides lower and upper bounds for the error probability of any classification rule. These bounds can be determined by solving the two convex optimization problems \eqref{opt-lower} and \eqref{opt-upper}. As shown below, the upper bound for \acp{MRC} is obtained as a by-product of the learning process that solves optimization \eqref{opt-prob}.
%.

\begin{corollary}\label{cor-upper}
Let $\up{h}^{\sset{U}}$ be a 0\,-1 \ac{MRC} for an uncertainty set $\set{U}$ that satisfies R1 or R2, then $\overline{R}(\set{U},\up{h}^{\sset{U}})=\overline{R}(\set{U})$ given by \eqref{minimax}.
\end{corollary}

\begin{proof}
Straightforward consequence of the above results since
$$\overline{R}(\set{U},\up{h}^{\sset{U}})=\sup_{\up{p}\in\set{U}}\ell(\up{h}^{\sset{U}},\up{p})=\adjustlimits\inf_{\up{h}\in \text{T}(\set{X},\set{Y})\,\,}\sup_{\up{p}\in\set{U}}\,\,\ell(\up{h},\up{p})=\overline{R}(\set{U})$$
\end{proof}
In order to simplify notation, we will denote $\underline{R}(\set{U})\coloneqq\underline{R}(\set{U},\up{h}^{\sset{U}})$ for an \ac{MRC} $\up{h}^{\sset{U}}$ given by~\eqref{MRC}. Tight performance bounds can also be obtained at learning for deterministic 0\,-1 \acp{MRC}. Specifically, such bounds are given by $\underline{R}(\set{U},\up{h}^{\sset{U}}_\text{d})$ and $\overline{R}(\set{U},\up{h}^{\sset{U}}_\text{d})$ that are obtained by solving the two optimization problems \eqref{opt-lower} and \eqref{opt-upper}. 

The next section provides generalization bounds for 0\,-1 \acp{MRC}. Such bounds also ensure generalization for deterministic 0\,-1 \acp{MRC} because $R(\up{h}^{\sset{U}}_\text{d})\leq 2 R(\up{h}^{\sset{U}})$ as a direct consequence of the fact that 
$$1-\up{h}^{\sset{U}}_\text{d}(y|x)\leq 2(1-\up{h}^{\sset{U}}(y|x)),\  \forall x,y\in\set{X}\times\set{Y}.$$

\subsection{Finite-sample generalization bounds}\label{sec-4.2}
This section provides \acp{MRC}' generalization bounds in terms of optimal minimax rules, i.e., \acp{MRC} with the smallest minimax risk. Such rules correspond with uncertainty sets given by the true expectation of the feature mapping $\Phi$, that is
\begin{align}\label{unc-inf}\set{U}_\infty=\{\up{p}\in\Delta(\set{X}\times\set{Y}):\  \mathbb{E}_{\up{p}}\{\Phi(x,y)\}=\B{\tau}_\infty\}\end{align}
where the mean vector $\B{\tau}_\infty=\mathbb{E}_{\up{p}^*}\{\Phi(x,y)\}$ is the exact expectation with respect to the underlying distribution. The minimum wost-case error probability (minimax risk) over distributions in $\set{U}_\infty$ is given by 
$$R_{\Phi}=\min_{\B{\mu}}1-\B{\tau}_{\infty}^{\text{T}}\B{\mu}+\varphi(\B{\mu})=1-\B{\tau}_{\infty}^{\text{T}}\B{\mu}_\infty+\varphi(\B{\mu}_\infty)$$
corresponding with the \ac{MRC} given by parameters $\B{\mu}_\infty$. Such classification rule is referred to as the optimal minimax rule for feature mapping $\Phi$ because for any uncertainty set $\set{U}$ given by \eqref{uncertainty} that contains the underlying distribution $\up{p}^*$, we have that $\set{U}_\infty\subset\set{U}$ and hence  $R_{\Phi}\leq \overline{R}(\set{U})$. The optimal minimax rule could only be obtained by an exact estimation of the expectation of the feature mapping $\Phi$ that in turn would require an infinite amount of training samples. %The following shows that the difference between the error probability of the optimal minimax rule and that of an \ac{MRC} obtained with $n$ training samples decreases with $n$ at the same rate as the confidence vector.

%The following result shows that with broad generality, tight upper and lower bounds for the error probability of an \ac{MRC} can be obtained at learning, and that difference between \ac{MRC} error probability and the smallest minimax risk decreases  with $n$ at the same rate as the error in mean vector estimates.

%The following result shows that, with wide generality, the error probability of \acp{MRC} can be estimated at learning and its sub-optimality is determined by the accuracy of mean vectors estimates.

The following result provides generalization bounds for \acp{MRC} in terms of excess error probability with respect to the minimax risk at learning and the smallest minimax risk for the feature mapping used.

\begin{theorem}\label{th_6}
Let $\set{U}$ and $\set{U}_\infty$ be uncertainty sets given by \eqref{uncertainty} and \eqref{unc-inf}, respectively, that satisfy R1 or R2.  If $\up{h}^{\sset{U}}$ is a 0\,-1 \ac{MRC} for uncertainty set $\set{U}$, we have that
%Let $\B{\tau}$ be a mean vector estimate given by training samples and $\up{h}^{\sset{U}}$ be a 0\,-1 \ac{MRC} for uncertainty set $\set{U}$ given by \eqref{uncertainty} with mean vector $\B{\tau}$ and confidence vector $\B{\lambda}\succeq\V{0}$.
%If $\B{\mu}^*$ and $\underline{\B{\mu}}$ are solutions to \eqref{opt-prob} and \eqref{opt-lower}, respectively, then, we have that
\begin{align}
R(\up{h}^{\sset{U}})&\leq \overline{R}(\set{U})+(|\B{\tau}_\infty-\B{\tau}|-\B{\lambda})^{\text{T}}|\B{\mu}^*|\label{bound3u}\\
R(\up{h}^{\sset{U}})&\geq\underline{R}(\set{U})-(|\B{\tau}_\infty-\B{\tau}|-\B{\lambda})^{\text{T}}|\underline{\B{\mu}}| \label{bound3l}\\
R(\up{h}^{\sset{U}})&\leq R_\Phi+|\B{\tau}_\infty-\B{\tau}|^{\text{T}}|\B{\mu}_\infty-\B{\mu}^*|+\B{\lambda}^{\text{T}}(|\B{\mu}_\infty|-|\B{\mu}^*|)\label{bound4}\end{align}
with $\B{\mu}^*$ and $\underline{\B{\mu}}$ solutions to \eqref{opt-prob} and \eqref{opt-lower}, respectively. In particular, if $\B{\lambda}$ is a confidence vector with coverage probability $1-\delta$, i.e., $\mathbb{P}\{|\B{\tau}_\infty-\B{\tau}|\preceq\B{\lambda}\}\geq 1-\delta$, then, we have that
\begin{align}
\underline{R}(\set{U})\leq R(\up{h}^{\sset{U}})&\leq \overline{R}(\set{U})\label{bound1}\\
R(\up{h}^{\sset{U}})&\leq R_\Phi+2\B{\lambda}^{\text{T}}|\B{\mu}_\infty|
\leq R_\Phi+2\|\B{\lambda}\|_\infty\|\B{\mu}_\infty\|_1.\label{bound2}
\end{align}
with probability at least $1-\delta$.
\end{theorem}
\begin{proof}
See Appendix~\ref{apd:proof_th_6}.
\end{proof}

Inequalities \eqref{bound3u}, \eqref{bound3l},  and \eqref{bound1} bound \acp{MRC}' probabilities of error w.r.t. the corresponding minimax error probability $\overline{R}(\set{U})$ and lower bound $\underline{R}(\set{U})$. Such quantities can be obtained at learning,  $\overline{R}(\set{U})$ is given as a byproduct of the learning process that solves optimization $\mathscr{P}_{\B{\tau},\B{\lambda}}$ in \eqref{opt-prob} while $\underline{R}(\set{U})$ requires to solve an additional convex optimization problem given by \eqref{opt-lower}. 
Inequalities \eqref{bound4} and \eqref{bound2} bound \acp{MRC}' probabilities of error w.r.t. the smallest minimax risk $R_\Phi$. As shown in the next section, this smallest minimax risk becomes the Bayes risk using rich feature mappings so that such generalization bounds enable to prove universal consistency results for \acp{MRC}. 

The generalization bounds in the above result depend on the accuracy of mean vector estimates and on the confidence vector used. Several important conclusions can be drawn from such bounds:
\begin{itemize} 
\item The excess error probability with respect to minimax risks decreases at the same rate as the error of the mean vector estimates. As shown in Theorem~\ref{th2}, with wide generality the bounds in the above theorem show differences that decrease with $n$ at a rate $\set{O}(1/\sqrt{n})$ using mean vector estimates given by sample averages as in \eqref{tau}.
Methods that utilize 0\,-1 loss but consider uncertainty sets with fixed marginals provide significantly coarser performance guarantees. In particular, the bounds in Theorem~3 of \cite{FarTse:16} are $\set{O}(1/\varepsilon\sqrt{n})$ where $\varepsilon$ describes the norm of the error in sample mean estimates, which is commonly $\varepsilon=\set{O}(1/\sqrt{n})$. In addition, \ac{RRM} methods based on Wasserstein distances achieve generalization bounds that decrease with $n$ at a rate $\set{O}(1/n^{1/d})$ that deteriorates with   the instances' dimensionality $d$ \citep{ShaKuhMoh:19,FroClaChiSol:21}. 
 
\item \acp{MRC} do not heavily rely on the choice of confidence vectors. For a given mean vector $\B{\tau}$, the upper bound in \eqref{bound3u} takes its smallest value for the \ac{MRC} corresponding with the confidence vector $\B{\lambda}$ given by the error $|\B{\tau}_\infty-\B{\tau}|$ in the mean vector estimate. Specifically, if $\set{U}_\text{e}$ is the uncertainty set given by \eqref{uncertainty} with mean vector $\B{\tau}$ and confidence vector $\B{\lambda}=|\B{\tau}_\infty-\B{\tau}|$, the upper bound in \eqref{bound3u} becomes $\overline{R}(\set{U}_\text{e})$, and, for any uncertainty set $\set{U}$ given by \eqref{uncertainty} with mean vector $\B{\tau}$, we have that 
\begin{align*}
\overline{R}(\set{U}_{\text{e}})&=\min_{\B{\mu}}1-\B{\tau}^{\text{T}}\B{\mu}+\varphi(\B{\mu})+|\B{\tau}_\infty-\B{\tau}|^{\text{T}}|\B{\mu}|\\
&\leq 1-\B{\tau}^{\text{T}}\B{\mu}^*+\varphi(\B{\mu}^*)+\B{\lambda}^{\text{T}}|\B{\mu}^*|+(|\B{\tau}_\infty-\B{\tau}|-\B{\lambda})^{\text{T}}|\B{\mu}^*|\\
&=\overline{R}(\set{U})+(|\B{\tau}_\infty-\B{\tau}|-\B{\lambda})^{\text{T}}|\B{\mu}^*|
\end{align*}
with $\B{\mu}^*$ solution of \eqref{opt-prob} for uncertainty set $\set{U}$. 

%we have that 
%\begin{align*}\overline{R}(\set{U})+(|\B{\tau}_\infty-\B{\tau}|-\B{\lambda})^{\text{T}}|\B{\mu}^*|&=1-\B{\tau}\B{\mu}^*+\varphi(\B{\mu}^*)+|\B{\tau}_\infty-\B{\tau}|^{\text{T}}|\B{\mu}^*|\\
%&\geq 1-\B{\tau}\B{\mu}^{\text{e}}+\varphi(\B{\mu}^{\text{e}})+|\B{\tau}_\infty-\B{\tau}|^{\text{T}}|\B{\mu}^{\text{e}}|\end{align*}
%for $\B{\mu}^{\text{e}}$ the parameter of the \ac{MRC} corresponding with the confidence vector $|\B{\tau}_\infty-\B{\tau}|$.
 %equals
%\begin{align*}\label{upper-bound2}1-\B{\tau}_n\B{\mu}^*+\varphi(\B{\mu}^*)+|\B{\tau}_\infty-\B{\tau}|^{\text{T}}|\B{\mu}^*|\end{align*}
%that 
%\begin{align}\label{upper-bound2}1-\B{\tau}_n\B{\mu}^*+\varphi(\B{\mu}^*)+|\B{\tau}_\infty-\B{\tau}|^{\text{T}}|\B{\mu}^*|.\end{align}
Near-optimal generalization bounds can be obtained using confidence vectors that approximate $|\B{\tau}_\infty-\B{\tau}|$. Specifically, 
for any uncertainty set $\set{U}$ as above, we have that
%if $\up{h}^{\sset{U}}$ is an \ac{MRC} given by mean and confidence vectors $\B{\tau}$ and $\B{\lambda}$, and determined by parameter $\B{\mu}^*$, its error probability satisfies 
\begin{align*}\overline{R}(\set{U}_{\text{e}})&\leq\overline{R}(\set{U})+(|\B{\tau}_\infty-\B{\tau}|-\B{\lambda})^{\text{T}}|\B{\mu}^*|\\
&= 1-\B{\tau}^{\text{T}}\B{\mu}^*+\varphi(\B{\mu}^*)+|\B{\tau}_\infty-\B{\tau}|^{\text{T}}|\B{\mu}^*|\\
&\leq1-\B{\tau}^{\text{T}}\B{\mu}^{\text{e}}+\varphi(\B{\mu}^{\text{e}})+\B{\lambda}^{\text{T}}|\B{\mu}^{\text{e}}|+(|\B{\tau}_\infty-\B{\tau}|-\B{\lambda})^{\text{T}}|\B{\mu}^*|\\
&=\overline{R}(\set{U}_{\text{e}})+(|\B{\tau}_\infty-\B{\tau}|-\B{\lambda})^{\text{T}}(|\B{\mu}^*|-|\B{\mu}^{\text{e}}|)\end{align*}
with $\B{\mu}^\text{e}$ solution of \eqref{opt-prob} for uncertainty set $\set{U}_\text{e}$. Therefore, the difference between the upper bound in \eqref{bound3u} for uncertainty set $\set{U}$ and the smallest upper bound $\overline{R}(\set{U}_{\text{e}})$ is bounded by 
$(|\B{\tau}_\infty-\B{\tau}|-\B{\lambda})^{\text{T}}(|\B{\mu}^*|-|\B{\mu}^{\text{e}}|)$, and both terms in such scalar product are small when $\B{\lambda}\approx|\B{\tau}_\infty-\B{\tau}|$.
%and both terms in the scalar product $(\B{\lambda}-|\B{\tau}_\infty-\B{\tau}|)^{\text{T}}(|\B{\mu}^{\text{e}}|-|\B{\mu}^*|)$ are small when $\B{\lambda}\approx|\B{\tau}_\infty-\B{\tau}|$.

% For $\B{\lambda}\neq|\B{\tau}_\infty-\B{\tau}|$, the suboptimality of the generalization bound \eqref{bound3u} is proportional to $\B{\lambda}-|\B{\tau}_\infty-\B{\tau}|$

%coincides with the suboptimality of minimizing \eqref{upper-bound2} using confidence vector $\B{\lambda}$ instead of $|\B{\tau}_\infty-\B{\tau}|$ that is $\set{O}(\B{\lambda}-|\B{\tau}_\infty-\B{\tau}|)$. 

\item The minimax risk optimized at learning can offer an adequate assessment of the \ac{MRC}'s error probability even if the uncertainty set used does not include the underlying distribution. Such minimax risk $\overline{R}(\set{U})$ obtained as a byproduct of the learning process provides valid upper bounds for the error probability of \acp{MRC} in cases where \mbox{$|\B{\tau}_\infty-\B{\tau}|\preceq\B{\lambda}$}, i.e., $\up{p}^*\in\set{U}$. As shown in \eqref{bound3u} and \eqref{bound3l}, $\overline{R}(\set{U})$ and $\underline{R}(\set{U})$ still provide approximate bounds in other cases as long as the confidence vector is not significantly smaller than the error in the mean vector estimates. 
%\newpage
\item High-confidence upper and lower bounds for error probabilities can be obtained using confidence vectors with high coverage probability. Such confidence vectors can be obtained using the expressions in Theorem~\ref{th2} or numerical methods such as those proposed in \cite{WauRam:20}. Those vectors may not be adequate for \ac{MRC} learning since they are often much larger than $|\B{\tau}_\infty-\B{\tau}|$.  Notice that a confidence vector $\B{\lambda}_\delta$ that ensures $\{\B{\lambda}_\delta\succeq|\B{\tau}_\infty-\B{\tau}|\}$ occurs with high probability for any underlying distribution is likely to be much larger than $|\B{\tau}_\infty-\B{\tau}|$ for a particular training set drawn from a specific underlying distribution.  However, confidence vectors with high coverage probability can be used to obtain error probability bounds for general \acp{MRC} that are valid with probability at least $1-\delta$. A simple way to get such bounds from $\overline{R}(\set{U})$ and $\underline{R}(\set{U})$ corresponding with a confidence vector $\B{\lambda}$ follows by noticing that \eqref{bound3u} and \eqref{bound3l} imply that 
\begin{align}\label{conf-bound}\underline{R}(\set{U})-(\B{\lambda}_\delta-\B{\lambda})^{\text{T}}|\underline{\B{\mu}}|\leq R(\up{h}^{\sset{U}})\leq\overline{R}(\set{U})+(\B{\lambda}_\delta-\B{\lambda})^{\text{T}}|\B{\mu}^*|\end{align} with probability at least $1-\delta$ because $\B{\lambda}_\delta\succeq|\B{\tau}_\infty-\B{\tau}|$ with that probability.
Another way to obtain high-confidence bounds follows by solving the convex optimization problems in Theorem~\ref{th-bounds} for the \ac{MRC} corresponding to $\B{\lambda}$ and the uncertainty set corresponding to $\B{\lambda}_\delta$.
\end{itemize}

Theorem~\ref{th_6} and the above discussion exhibit the different roles played by confidence vectors that aim to bound the error in mean vector estimates with high probability for any probability distribution versus those that only aim to approximate such error. The former can be used to obtain provably valid performance guarantees while the later can be used to obtain small error probability and approximate performance bounds. Such roles are further studied numerically in Section~\ref{sec-numerical} using multiple real datasets.%. and the robustness of \acp{MRC} to the choice of confidence

The improved performance of methods that use aggressively chosen parameters has been observed in multiple fields of machine learning. For instance, in on-line learning methods, the theoretical prescriptions for learning rates that lead to valid prediction guarantees are routinely outperformed by choices that minimize prediction error on the data seen so far but have worse guarantees (see e.g., \cite{DevGaiGouSto:13}). This phenomenon is related to the fact that in-expectation generalization bounds can be significantly better than in-probability generalization bounds, and it is further explored in a PAC-Bayes setting by \cite{GruSteZak:21}. The generalization bounds in Theorem~\ref{th_6} shed light on such phenomenon for the proposed \acp{MRC}.

% and is further studied numerically in Section~\ref{sec-numerical}. In particular, the third set of numerical results studies how the \acp{MRC} performance and bounds change by varying the confidence vector used showing that \acp{MRC} do not need a 

\subsection{Universal consistency}\label{sec-4.3}
We show next that \acp{MRC} can be strongly universally consistent using rich feature mappings, that is, as the training size grows, the \ac{MRC}'s error probability tends to the Bayes risk with probability one for any underlying distribution. 

The theorem below shows universal consistency for \acp{MRC} given by feature mappings determined by random features corresponding with rich \acp{RKHS} (see e.g., \cite{RahRec:08,Bac:17}).  Specifically, let $v_1,v_2,\ldots$ be a sequence of i.i.d. samples from distribution $\up{p}(v)$ generating random features $\psi_v:\set{X}\to\mathbb{R}$ for kernel $k:\set{X}\times\set{X}\to\mathbb{R}$, that is 
$$\mathbb{E}_{\up{p}(v)}\{\psi_v(x)\psi_v(x')\}=k(x,x'),\  \forall x,x'\in\set{X}.$$ In addition, for $n=1,2,\ldots$ the feature map $\Phi_n$ is defined by the $D_n$ random features $\psi_{v_1},\psi_{v_2},\ldots,\psi_{v_{D_n}}$ as
\begin{align}\label{feature_random}
\Phi_n(x,y)=\V{e}_y\otimes [1,\psi_{v_1}(x),\psi_{v_2}(x),\ldots,\psi_{v_{D_n}}(x)]^{\text{T}}.
\end{align}
Using such feature mappings, we have the following result.
\begin{theorem}\label{th_7}
 For $n=1,2,\ldots$, let $\up{h}_n$ be an \ac{MRC} for uncertainty set $\set{U}_n$ given by \eqref{uncertainty} with $\Phi=\Phi_n$ as in \eqref{feature_random}, $\B{\tau}=\B{\tau}_n$ as in \eqref{tau} and $\B{\lambda}=\B{\lambda}_n\succeq\V{0}$.
 
 If we have that 
 \begin{itemize}
 \item[(1)] $k:\set{X}\times\set{X}\to\mathbb{R}$ is a characteristic kernel,
 \item[(2)] scalar features are bounded, i.e., there exists $C>0$ such that $|\psi_v(x)|<C$ for all $v$ and $x$,
 \item[(3)] the number of random features $D_n$ defining the feature mapping $\Phi_n$ is non-decreasing and tends to infinity, and
 \item[(4)] any component of $\{\B{\lambda}_n\}$ is non-increasing and tends to $0$.
 \end{itemize}
Then, the sequence of smallest minimax risks $\{R_{\Phi_n}\}$ tends to the Bayes risk $R_{\text{Bayes}}$ with probability one for any underlying distribution $\up{p}^*$.

If, in addition to (1)-(4), we have that  
\begin{itemize}
\item[(5)] any component of $\{\B{\lambda}_n\sqrt{n/\log n}\}$ tends to $\infty$, and
\item[(6)] $D_n=\set{O}(n^k)$ for some $k>0$.
\end{itemize}
Then, the sequence of \acp{MRC}' probabilities of error $\{R(\up{h}_n)\}$ tends to the Bayes risk $R_{\text{Bayes}}$ with probability one for any underlying distribution $\up{p}^*$. 
\end{theorem}

\begin{proof}
See Appendix~\ref{apd:proof_th_7}.
\end{proof}

The conditions under which \acp{MRC} are strongly universally consistent are analogous to those corresponding to conventional \ac{ERM} methods based on \ac{SRM} and \acp{RKHS} \citep{Ste:05,Zha:04}, e.g., regularization parameters that tend to zero not very quickly and broad \acp{RKHS}. The universal consistency in \ac{ERM} techniques is achieved using \acp{RKHS} given by universal kernels while the theorem above uses characteristic kernels, which is in general a slightly weaker condition \citep{MuaFukSriSch:17}. In addition, the result above provides \acp{MRC}' universal consistency for binary and multiclass cases and does not rely on surrogate losses. 

Notice that for \ac{ERM} methods based on \ac{SRM}, the decrease of the regularization parameter for increasing number of samples corresponds to consider broader families of rules (balls in the \ac{RKHS} with increased radius). On the other hand, for the proposed approach, such decrease corresponds to consider reduced uncertainty sets. This fact illustrates how the bias-complexity trade-off addressed in the \ac{SRM} approach by controlling the generality of the family of classification rules is analogous to controlling the generality of the uncertainty set of probability distributions in the proposed approach (see also discussion in Section~\ref{sec-2.1} above). In paticular, 
the conditions (5) and (6) in Theorem~8 result in uncertainty sets that shrink as we get more samples but contain the underlying distribution with high probability.

\section{Efficient learning of \acp{MRC}}\label{sec-5}
The learning stage for \acp{MRC} consists on solving optimization problem \eqref{opt-prob} that obtains the parameters for $\up{h}^{\sset{U}}$ and $\up{h}_{\text{d}}^{\sset{U}}$  as well as the minimax risk $\overline{R}(\set{U})$. This stage can be complemented by solving optimization problems \eqref{opt-lower} and \eqref{opt-upper} that can provide tight performance guarantees.  In this section we first show that such optimization problems can be simplified by using a reduced instances' set given by instances' samples. We then propose efficient subgradient methods that take advantage of the specific structure of the above mentioned optimization problems.
\subsection{Reduced instances' set}

Solving the optimization problems that learn \acp{MRC} and obtain their performance bounds can be inefficient in cases where the feature mapping range $\{\Phi(x,y): x\in\set{X}, y\in\set{Y}\}$ has a large cardinality, since the evaluation of objective functions in \eqref{opt-prob}, \eqref{opt-lower}  and \eqref{opt-upper} may require to search over such a range. The next result shows that this difficulty can be avoided by using instances' samples instead of the whole set $\set{X}$, e.g., using the instances obtained at training.

%, \color{blue}$\set{U}_s$ the uncertainty set given by \eqref{uncertainty} using $\set{X}_s$ instead of $\set{X}$, \color{black}

\begin{theorem}\label{th_8}
Let $\set{X}_s=\{x_1,x_2,\ldots,x_s\}$ be $s$ i.i.d. instances from the underlying distribution $\up{p}^*(x)$, $\up{h}_s$, $\overline{R}_s(\set{U})$, and $\underline{R}_s(\set{U})$ be the \ac{MRC} and bounds obtained by solving \eqref{opt-prob} and \eqref{opt-lower} using $\set{X}_s$ instead of $\set{X}$. If $\overline{R}_s(\set{U})$ is finite and $\up{p}^*\in\set{U}$ with probability at least $1-\delta$, we have that
$$\underline{R}_s(\set{U})-\varepsilon_s(2C\|\B{\mu}_l\|_1+1)\leq R(\up{h}_s)\leq \overline{R}_s(\set{U})+\varepsilon_s(2C\|\B{\mu}_u\|_1+1)\leq \overline{R}(\set{U})+\varepsilon_s(2C\|\B{\mu}_u\|_1+1)$$
with probability at least $1-2\delta$, where $\B{\mu}_l$ and $\B{\mu}_u$ are the solutions of \eqref{opt-lower} and \eqref{opt-upper}, respectively, using $\set{X}_s$ instead of $\set{X}$, $\|\Phi(x,y)\|_\infty\leq C$, $\forall x\in\set{X},y\in\set{Y}$, and $\varepsilon_s$ is given by
 $$\varepsilon_s=6|\set{Y}|\sqrt{\frac{4+|\set{Y}|(m+1)\log s +\log|\set{Y}|/\delta}{s}}.$$
\end{theorem}

\begin{proof}
See Appendix~\ref{apd:proof_th_8}.
\end{proof}

The above result shows that a sufficiently large set of instances' samples $\set{X}_s$ can be safely used instead of the whole set $\set{X}$ in optimization problems \eqref{opt-prob} and \eqref{opt-lower} since the subsequent approximation error is $\set{O}(\sqrt{\log(s)/s})$. An analogous result can be obtained for deterministic $\acp{MRC}$ and \eqref{opt-upper} using the same arguments as in Appendix~\ref{apd:proof_th_8} for the above result. 

The condition $|\overline{R}_s(\set{U})|<\infty$ can be easily satisfied in practice. Such condition is equivalent to 
$$\set{U}_s=\Big\{\up{p}\in\Delta(\set{X}_s\times\set{Y}):\  |\mathbb{E}_{\up{p}}\{\Phi\} -\B{\tau}|\preceq \B{\lambda}\Big\}\neq\emptyset$$
that is directly achieved if $\B{\tau}$ is obtained as the sample mean of samples with instances $\set{X}_s$. In other cases, it can be ensured that such uncertainty set is not empty by increasing the confidence vector and shifting the mean vector. In particular, if $\B{\lambda}_1^*$ and $\B{\lambda}_2^*$ are solutions of the linear optimization problem
\begin{align*}
\min_{\up{p},\B{\lambda}_1,\B{\lambda}_2}\  &\V{1}^{\text{T}}(\B{\lambda}_1+\B{\lambda}_2)\\
\mbox{s.t. }
&\B{\tau}-\B{\lambda}_1\preceq\sum_{x\in\set{X}_s,y\in\set{Y}}\up{p}(x,y)\Phi(x,y)\preceq\B{\tau}+\B{\lambda}_2\\
&\B{\lambda}_1\succeq\B{\lambda}, \B{\lambda}_2\succeq\B{\lambda},\up{p}\in\Delta(\set{X}_s\times\set{Y})\\
\end{align*}
with $2m+s|\set{Y}|$ variables and $4m+s|\set{Y}|+1$ constraints. Then, 
taking $\widetilde{\B{\tau}}=\B{\tau}+\frac{\B{\lambda}^*_2-\B{\lambda}^*_1}{2}$ and $\widetilde{\B{\lambda}}=\frac{\B{\lambda}^*_1+\B{\lambda}^*_2}{2}\succeq \B{\lambda}$ we have that 
$$\set{U}_s\subseteq\widetilde{\set{U}}_s=\Big\{\up{p}\in\Delta(\set{X}_s\times\set{Y}):\  |\mathbb{E}_{\up{p}}\{\Phi\} -\widetilde{\B{\tau}}|\preceq \widetilde{\B{\lambda}}\Big\}\neq\emptyset.$$

The numerical results of next Section~\ref{sec-numerical} show that the optimization problems for \ac{MRC} learning can be accurately solved in practice using quite reduced sets of instances. In particular, such results show that the set of instances $\set{X}$ in optimization problems \eqref{opt-prob}, \eqref{opt-lower}, and \eqref{opt-upper} can be taken as the set $\set{X}_n=\{x_1,x_2,\ldots,x_n\}$ of instances obtained at training.

\subsection{Efficient optimization}
\label{ssc:ep}
This section presents efficient optimization techniques for \ac{MRC} learning that comprises to solve problems \eqref{opt-prob}, \eqref{opt-lower} and \eqref{opt-upper}.
These three problems can be written as: 
\begin{align}
\label{opt-prob-1}
\begin{array}{ll} \underset{\B{\mu}}{\min}&f(\B{\mu}):=\V{a}^{\text{T}}\B{\mu}+\B{\lambda}^{\text{T}}|\B{\mu}|+\max\{\V{F} \B{\mu}+\V{b}\}
\end{array}
\end{align}
with $\V{a}\in \mathbb{R}^{m}$, $\V{b}\in\mathbb{R}^{p}$ and $\V{F}\in \mathbb{R}^{p\times m}$.
The size of vector $\V{a}$ and the number of columns of matrix $\V{F}$, $m$, equals the number of components of the feature mapping $\Phi:\set{X}\times\set{Y}\to\mathbb{R}^m$, while the size of vector $\V{b}$ and the number of rows of matrix $\V{F}$, $p$, is given by the number of instances $\set{X}_s$ used to evaluate $\varphi$,  e.g., the number of instances at training. Specifically, $p$ equals
 $s(2^{|\set{Y}|}-1)$ in problem \eqref{opt-prob}, and $s{|\set{Y}}|$ in problems \eqref{opt-lower} and \eqref{opt-upper}.

 The optimization problem  \eqref{opt-prob-1} can be reformulated as the \ac{LP} 
  \begin{align}
\label{opt-prob-3}
  \begin{array}{ll} \underset{\B{\mu}_1,\B{\mu}_2}{\min}&\B{a}^{\text{T}}(\B{\mu}_1-\B{\mu}_2)
 +\B{\lambda}^{\text{T}}(\B{\mu}_1+\B{\mu}_2)+\nu\\
    \mbox{s.t.}&\V{F} \B{\mu}_1-\V{F} \B{\mu}_2+\V{b}\preceq \nu\V{1}\\
&\B{\mu}_1,\B{\mu}_2\succeq\V{0}
\end{array}
\end{align}
by introducing new variables $\nu\in\mathbb{R}$ and
$\B{\mu}_1,\B{\mu}_2\in\mathbb{R}^m$ with $\B{\mu}=\B{\mu}_1-\B{\mu}_2$. % and $ , $|\B{\mu}| = \B{\mu}_1+\B{\mu}_2$:
Such an \ac{LP} has $p+2m$ constraints and $2m+1$ variables, and can be solved with high accuracy by off-the-shelf solvers at the expenses of a high computational time in cases where $p$ or $m$ is large.

Problem \eqref{opt-prob-1} belongs to the class of nondifferentiable convex optimization problems with subgradients that are uniformly bounded.
The \acp{SM} are often an attractive option to solve this class of problems since they can efficiently provide solutions with adequate accuracy \citep{Ber:15}.
A subgradient of the objective function $f$ in \eqref{opt-prob-1} at point $\B{\mu}$ can be directly obtained from the maximum of vector $\V{v}=\V{F} \B{\mu}+\V{b}$. Specifically, if $i(\B{\mu})\in\arg\max\V{v}$ 
%\begin{align}\label{index}i(\B{\mu})\in \underset{i}{\mbox{argmax}}\left\{(\V{v})^{(i)}\right\}\end{align}
we have that a subgradient of $f$ at $\B{\mu}$ is given by
\begin{equation}
  \label{eq:subgrad}
  g(\B{\mu})=\V{a}+ \B{\lambda}\odot\text{sign}(\B{\mu})+\text{col}_{i(\B{\mu})}(\V{F}^{\text{T}})\end{equation}
where $\odot$ denotes the Hadamard product, $\text{sign}(\B{\mu})$ denotes the vector given by the signs of the components of $\B{\mu}$, and $\text{col}_{i}(\cdot)$ denotes the $i$-th column of the argument. 

Often, the main limitation of \acp{SM} is the large number of iterations required. The \acp{ASM} \citep{NesShi:15,TaoZhiGaoTao:20} have been developed to reduce the number of iterations by using the Nesterov's extrapolation strategy \citep{Nes:83}. In particular, Algorithm~\ref{alg.Tao} describes the application of the \ac{ASM} proposed in \cite{TaoZhiGaoTao:20} to problem \eqref{opt-prob-1} for \ac{MRC} learning.\begin{algorithm}
\setstretch{1}
\caption{\label{alg.Tao}-- \ac{ASM} for \ac{MRC} learning}	
 \small
  \begin{tabular}{ll}
 %\hspace{-0.1cm}\textbf{Input:}&\hspace{-.7cm}
 %initial points $\B{\mu}_0$ \\
\hspace{-0.1cm}\textbf{Output:}&\hspace{-0.3cm}approximate solution $\B{\mu}^\ast$ 
 \end{tabular}
\begin{algorithmic}[1] 
\setstretch{1.2}
% \vspace{-0.1cm}
% \STATE $i_{0} \gets \underset {i}{\mbox{argmax}}\left\{(\B{M}\B{\mu}_{0}+\B{b})_i \right\}$
\STATE Initialize $\B{\mu}_1$ and take $c_1=\theta_1=1$,  $\eta_1=0$
\STATE $\V{y}_1\gets\B{\mu}_1$, $\V{v}_1\gets\V{F}\B{\mu}_1+\V{b}$, 
$i_{1} \gets \arg\max\V{v}_{1} $, $\B{\mu}^\ast\gets \B{\mu}_1$, $f^\ast\gets \V{a}^{\text{T}}\B{\mu}_1+\B{\lambda}^{\text{T}}|\B{\mu}_1|+\V{v}_1^{(i_1)}$
\FOR{$k=1,2,\ldots$}
\STATE $\V{g}_k\gets\V{a}+\B{\lambda}\odot\text{sign}(\B{\mu}_k)+\text{col}_{i_k}(\V{F}^{\text{T}})$
%\STATE $\V{g}_k\gets\V{a}+\B{\lambda}\odot\text{sign}(\B{\mu}_k)+\text{col}_{i_k}(\V{F}^{\text{T}})$
\STATE $\V{y}_{k+1} \gets\B{\mu}_k -c_k\V{g}_k$,\   \   $\B{\mu}_{k+1} \gets (1+ \eta_{k})\V{y}_{k+1}-\eta_{k} \V{y}_{k}$
%\STATE $\B{\mu}_{k+1} \gets (1+ \eta_{k})\V{y}_{k+1}-\eta_{k} \V{y}_{k}$
\STATE $\V{v}_{k+1}\gets\V{F}\B{\mu}_{k+1}+\V{b}$, \   \   
$i_{k+1} \gets \arg\max\V{v}_{k+1} $
\STATE$c_{k+1}\gets (k+1)^{-3/2} $,  $\ \theta_{k+1}\gets \frac{2}{k+1},\ $ $\ \eta_{k+1} \gets \theta_{k+1}\left(\frac{1}{\theta_{k}}-1\right)$
\STATE $f_{k+1}\gets   \V{a}^{\text{T}}\B{\mu}_{k+1}+\B{\lambda}^{\text{T}}|\B{\mu}_{k+1}|+\V{v}_{k+1}^{(i_{k+1})}$
\IF{$f_{k+1}<f^{\ast}$}
\STATE $f^\ast\gets f_{k+1}$, $\B{\mu}^\ast\gets \B{\mu}_{k+1}$
\ENDIF
                \ENDFOR
\end{algorithmic}
\end{algorithm}

The time complexity per iteration of the \acp{SM} described above can be high in cases where $p$ or $m$ is large. This computational cost is due to the fact that the subgradient $g(\B{\mu}_k)$ is computed by evaluating $\V{v}_k=\V{F}\B{\mu}_k+\V{b}$ that requires $pm$ multiplications. Such computation can be carried out in a significantly more efficient manner by exploiting the specific structure of the subgradient. The vector $\V{v}_k$ can be efficiently computed from $\V{v}_{k-1}$ because we have that
$$\V{F}g(\B{\mu}_{k-1})=\V{F}\V{a}+\V{F}\text{diag}(\B{\lambda})\text{sign}(\B{\mu}_{k-1})+\text{col}_{i(\B{\mu}_{k-1})}(\V{F}\V{F}^{\text{T}})$$
and $\V{F}\V{a}$, $\V{F}\V{F}^{\text{T}}$, and $\V{F}\text{diag}(\B{\lambda})$ can be computed only once. In addition, the sequence of vectors \mbox{$\{\V{d}_k=\V{F}\text{diag}(\B{\lambda})\text{sign}(\B{\mu}_k)\}$} can be obtained as
\begin{align}\label{d-vector}\V{d}_{k}=\V{d}_{k-1}+\V{F}\text{diag}(\B{\lambda})\B{\Delta}_k\end{align}
where the vector $\B{\Delta}_k=\text{sign}(\B{\mu}_{k})-\text{sign}(\B{\mu}_{k-1})$ 
takes values $\pm 2$ in the components where the parameters $\B{\mu}_{k-1}$ and $\B{\mu}_{k}$ change signs, takes value $0$ in the components where they have the same sign, and would take values $\pm 1$ only if some of the components of such parameters are exactly zero. Therefore, the computation vector $\V{d}_{k}$ in \eqref{d-vector} can be carried out very efficiently in practice by adding or subtracting the columns of $2\V{F}\text{diag}(\B{\lambda})$ corresponding to the components where the iterates change sign.% (non-zero components in $\B{\Delta}_k$).

% or $0$ depending if the components of the iterates $\B{\mu}_{k-1}$ and $\B{\mu}_{k}$ change signs or not, and would take values $\pm 1$ only if some of the components of such iterates are exactly zero. 

%components of vector $\B{\Delta}_k=\text{sign}(\B{\mu}_{k})-\text{sign}(\B{\mu}_{k-1})$ take values $\pm 2$ 

%can only take values $\{-2,-1,0,1,2\}$ and values $\pm 1$ would require that some components of $\B{\mu}_{k-1}$ or $\B{\mu}_{k}$ are exactly zero. 

% In addition,  

%the such vector  would be $\pm 1$ only if some components of $\B{\mu}_{k-1}$ or $\B{\mu}_{k}$ are exactly zero, and 

%Such vector is non-zero only in the components where $\B{\mu}_{k-1}$ and $\B{\mu}_{k}$ have different signs, and would be $\pm 1$ only if some components of $\B{\mu}_{k-1}$ or $\B{\mu}_{k}$ are exactly zero. 

% Therefore computation requires $\gamma(\B{\Delta}_k)pm$ multiplications, where $\gamma(\B{\Delta}_k)$ denotes the fraction of non-zero components of $\B{\Delta}_k$.

Algorithm~\ref{alg.Tao-1} shows the efficient implementation of \ac{ASM} that exploits the specific structure of the subgradient as described above. The result below describes how such implementation can result in significant computational savings.

\begin{algorithm}
\setstretch{1}
\caption{\label{alg.Tao-1}--  Efficient \ac{ASM} for \ac{MRC} learning}	
 \small
 \begin{tabular}{ll}
   %\hspace{-0.1cm}\textbf{Input:}&\hspace{-0.7cm}
%problem data $\B{a}$, $\B{\lambda}$, $\B{b}$, $\B{M}$, starting parameters $\theta_0=1$, $a_0=1$, $\eta_0=0$, initial points $\B{\mu}_0$,\\

\hspace{-0.1cm}\textbf{Output:}&\hspace{-0.3cm}approximate solution $\B{\mu}^\ast$ 
\end{tabular}
\begin{algorithmic}[1] 
\setstretch{1.2}
% \vspace{-0.1cm}
\STATE Initialize $\B{\mu}_1$ and take $c_1=\theta_1=1$,  $\eta_1=0$, $\B{\alpha}=\V{F}\V{a}$, 
$\V{G}=\V{F} \V{F}^{\text{T}}$, $\V{H}=2\V{F}\text{diag}(\B{\lambda})$
%\STATE $\B{\alpha}\gets\V{F}\V{a}$, 
%$\V{G} \gets\V{F} \V{F}^{\text{T}}$, $\V{H}\gets \V{F}\text{diag}(\B{\lambda})$
\STATE $\V{y}_1\gets\B{\mu}_1$, $\V{v}_1 \gets\V{F}\B{\mu}_1+\B{b}$, $\V{w}_1 \gets\V{v}_1$, 
$\V{s}_1 \gets\text{sign}(\B{\mu}_1)$,  $\V{d}_1 \gets(1/2)\V{H}\V{s}_1$
\STATE $i_{1} \gets \arg\max\V{v}_{1} $, $\B{\mu}^\ast\gets \B{\mu}_1$, $f^\ast\gets \V{a}^{\text{T}}\B{\mu}_1+\B{\lambda}^{\text{T}}|\B{\mu}_1|+\V{v}_1^{(i_1)}$
% \vspace{-15pt}
\FOR{$k=1,2,\ldots$}
\STATE $\V{g}_k\gets\V{a}+\B{\lambda}\odot\V{s}_k+\text{col}_{i_k}(\V{F}^{\text{T}})$
%\STATE $\V{g}_k\gets\V{a}+\B{\lambda}\odot\V{s}_k+\text{col}_{i_k}(\V{F}^{\text{T}})$
\STATE $\V{y}_{k+1} \gets\B{\mu}_k-c_k \V{g}_k$,\   $\B{\mu}_{k+1} \gets (1+ \eta_{k})\V{y}_{k+1}-\eta_{k} \V{y}_{k}$
%\STATE $\B{\mu}_{k+1} \gets (1+ \eta_{k})\V{y}_{k+1}-\eta_{k} \V{y}_{k}$
\STATE $\V{u}_{k}\gets  \B{\alpha}+\V{d}_{k}+\text{col}_{i_{k}}(\V{G})$%  $\V{w}_{k+1} \gets\V{v}_k-c_k\V{u}_k$
\STATE $\V{w}_{k+1} \gets\V{v}_k-c_k\V{u}_k$,\   $\V{v}_{k+1} \gets (1+ \eta_{k})\V{w}_{k+1}-\eta_{k} \V{w}_{k}$,\  $i_{k+1} \gets \arg\max\V{v}_{k+1}$
%\STATE $\V{v}_{k+1} \gets (1+ \eta_{k})\B{w}_{k+1}-\eta_{k} \V{w}_{k}$
\STATE $ \V{s}_{k+1} \gets  \text{sign}(\B{\mu}_{k+1})$,\  $\B{\Delta}_{k} \gets \V{s}_{k+1}-\V{s}_{k}$,\  $\V{d}_{k+1}\gets \V{d}_{k}$
%\STATE $\B{\Delta}_{k+1} \gets \V{s}_{k+1}-\V{s}_{k},\ $
%\STATE $ \set{I}_{k+1} \gets  \{i\,|\, \B{\Delta}_{k}^{(i)}\neq 0\}$
%\STATE $\V{d}_{k+1}\gets \V{d}_{k}$
\FOR{$i=1,2,\ldots m$}
\IF{$\Delta_{k}^{(i)}= 2$}
\STATE $\V{d}_{k+1}\gets \V{d}_{k+1}+\text{col}_i(\V{H}) $
%\ENDIF
\ELSIF{$\Delta_{k}^{(i)}=-2$}
\STATE $\V{d}_{k+1}\gets \V{d}_{k+1}-\text{col}_i(\V{H}) $
%\ENDIF
\ELSIF{$\Delta_{k}^{(i)}\in\{-1,1\}$}
\STATE $\V{d}_{k+1}\gets \V{d}_{k+1}+(1/2)\text{sign}(\Delta_{k}^{(i)})\text{col}_i(\V{H}) $
\ENDIF
\ENDFOR
\STATE $c_{k+1}\gets (k+1)^{-3/2} $,  $\ \theta_{k+1}\gets \frac{2}{k+1},\ $ $\ \eta_{k+1} \gets \theta_{k+1}\left(\frac{1}{\theta_{k}}-1\right)$
\STATE $f_{k+1}\gets   \V{a}^{\text{T}}\B{\mu}_{k+1}+\B{\lambda}^{\text{T}}|\B{\mu}_{k+1}|+\V{v}_{k+1}^{(i_{k+1})}$
\IF{$f_{k+1}<f^{\ast}$}
\STATE $f^\ast\gets f_{k+1}$, $\B{\mu}^\ast\gets \B{\mu}_{k+1}$
\ENDIF
\ENDFOR
\end{algorithmic}
\end{algorithm}

\begin{theorem}\label{thm:comp.Tao-1}
The sequences $\{\B{\mu}_k\}$ generated by Algorithms \ref{alg.Tao} and \ref{alg.Tao-1} are identical, while the computational complexity of Algorithm~\ref{alg.Tao-1} is significantly smaller at iterations $k$ where $\text{sign}(\B{\mu}_{k})$ differs from $\text{sign}(\B{\mu}_{k-1})$ in few components. Specifically, if $\gamma_{K}$ is the sparsity coefficient given by the average fraction of non-zero components of $\B{\Delta}_k$ for $k=1,2,\ldots,K$; then, the computational complexity of Algorithm~\ref{alg.Tao-1} after $K$ iterations is $\set{O}( K pm\gamma_{K})$, while that of Algorithm~\ref{alg.Tao} is $\set{O}( K pm)$. 
 \end{theorem}
\begin{proof}
  See Appendix~\ref{apd:proof_thm:comp.Tao-1}.
\end{proof}

The computation in step 6 of Algorithm~\ref{alg.Tao} that has cost $\set{O}(pm)$ is effectively replaced by steps 7-18 in 
Algorithm~\ref{alg.Tao-1} that have cost $\set{O}\big(pm\gamma(\B{\Delta}_k)\big)$ where $\gamma(\B{\Delta}_k)$ denotes the fraction of  non-zero components of $\B{\Delta}_k$. As the algorithm progresses and the subgradient steps decrease to zero, it is  expected to have few changes in the signs of $\B{\mu}_k$ and therefore to have a highly sparse vector $\B{\Delta}_k$. The usage of an \ac{ASM} also contributes to a reduced number of sign changes since such method provides more stable iterations than basic subgradient methods. The method in \cite{Nes:14} also exploits the sparsity in subgradient methods for piecewise linear functions but it considers problems given only by a term $\max\{ \V{F}\B{\mu}+\V{b}\}$ and with a sparse matrix $\V{F}$. %This reduced oscillations are also promoted by the usage of \acp{ASM} that provide more stable iterations than basic subgradient methods.

In order to further decrease the number of iterations, we also use restarts similarly to other methods for non-smooth optimization \citep{YanQia:18}. Our numerical results confirm the efficiency of the implementation described in Algorithm \ref{alg.Tao-1} which resulted in a significant reduction (more than $90\%$) of the running time per iteration in comparison with Algorithm \ref{alg.Tao}.
%Specifically, we utilize a simple scheme that restarts the algorithm whenever the number of iterations or the objective's relative decrease over those iterations reach pre-defined thresholds. 

 \section{Numerical results}\label{sec-numerical}

This section shows four sets of numerical results that describe how the proposed techniques can enable to learn \acp{MRC} efficiently and to obtain reliable and tight performance bounds at learning. We utilize 12 common datasets from the UCI repository \citep{DuaGra:2019} with characteristics given in Table~\ref{table:datasets}. \acp{MRC}' implementation is available in the open-source Python library \texttt{MRCpy} \citep{KarMazPer:21} \url{https://MachineLearningBCAM.github.io/MRCpy/}.

In this section, \acp{MRC} are implemented using a feature mapping given by random features corresponding with a Gaussian kernel \citep{RahRec:08,Bac:17}, that is
\begin{align}\label{feature_exp}\Phi(x,y)=\V{e}_y\otimes\Psi(x)=\V{e}_y\otimes[\cos \V{u}_1^{\text{T}}x,\sin \V{u}_1^{\text{T}}x,\cos \V{u}_2^{\text{T}}x,\sin \V{u}_2^{\text{T}}x,\ldots,\cos \V{u}_D^{\text{T}}x,\sin \V{u}_D^{\text{T}}x]\end{align}
where $x$ is a normalized instance and $\V{u}_1,\V{u}_2,\ldots,\V{u}_D$ are i.i.d. samples from a zero-mean Gaussian distribution with covariance $(1/\sigma^2)\V{I}$ for a scaling parameter $\sigma$. In addition, for $n$ training samples $(x_1,y_1),(x_2,y_2),\ldots(x_n,y_n)$ we obtain confidence vectors as 
\begin{align}\label{confidence_exp}\B{\lambda}=\lambda_0\sqrt{\B{\upsilon}/n}\end{align} where $\B{\upsilon}$ is the vector formed by the sample variances of the feature mapping components. In all the numerical experiments, if not stated otherwise, we take $\lambda_0=0.3$ and use $D=500$ random Fourier features corresponding with the scaling parameter $\mbox{$\sigma=\sqrt{d/2}$}$ for $d$ the number of instances' components. %, that is
%$$\widehat{\B{\sigma}}=\sum_{i=1}^n\Phi(x_i,y_i)
We compare \acp{MRC} performance and bounds with those obtained by Wasserstein-based \ac{RRM} and by PAC-Bayes methods. Specifically, we utilize the \ac{DRLR} method as described in \cite{AbaMohKuh:15} with uncertainty sets defined by Wasserstein distances, and a \ac{SVM} with upper bounds given by PAC-Bayes as described in \cite{Lan:05,AmbParSha:07}.

\begin{table}
\caption{\small Characteristics of the datasets used for experimentation.}
%\vspace{0.1cm}
\label{table:datasets}
\centering
%\def\arraystretch{1.3}
%\resizebox{\textwidth}{!}{%
\begin{tabular}{lccc}%\\[-0.45cm]
\toprule Data set&Number of samples&Instances' dimensionality&\% Majority class\\[0.cm]
\hline
Adult&48,842&14&76\\
Pulsar&17,898&8&91\\
Credit&690&15&56\\
QSAR&1055&41&66\\
Mammographic&961&5&54\\
Haberman&306&3&74\\
Ion&351&34&64\\
Heart&270&13&56\\
Liver&583&10&71\\
Blood&748&4&76\\
Diabetes&768&8&65\\
Audit&776&26&61\\
%WDBC&569&30&63\\
\bottomrule
\end{tabular}
\end{table}

The first set of numerical results shows that \acp{MRC}' optimization problems can be efficiently addressed in practice by using a reduced set of instances.  Specifically, for $s$ i.i.d. instances $\set{X}_s=\{x_1,x_2,\ldots,x_s\}$, we solve optimization problems \eqref{opt-prob} and \eqref{opt-lower} taking $\set{X}=\set{X}_s$ and we study how the error of such approximation decreases with increasing $s$. For these numerical results we use ``Adult'' and ``Pulsar'' datasets from UCI repository \citep{DuaGra:2019} that contain a total of 48,842 and 17,898 instances, respectively.  We obtain $\B{\tau}$ and $\B{\lambda}$ using $500$ training samples. Then, we solve \eqref{opt-prob} and \eqref{opt-lower} taking $\set{X}$ composed by all the instances in the dataset and taking $\set{X}$ composed by a subset of randomly selected instances $\set{X}_s$ of size $s$. 

\begin{figure}
\psfrag{x}[l][b][0.6]{\hspace{-1.8cm}Number of instances $s$}
\psfrag{y}[l][t][0.6]{\hspace{-1.57cm}Classification error}
\psfrag{A1234567890123456789012345}[l][][0.5]{\hspace{-1.65cm}Lower bound $\underline{R}_s(\set{U})$}
\psfrag{A2}[l][][0.5]{\hspace{-0.1cm}Upper bound $\overline{R}_s(\set{U})$}
\psfrag{A3}[l][][0.5]{\hspace{-0.1cm}Error prob. $R(\up{h}_s)$}
\psfrag{A4}[l][][0.5]{\hspace{-0.1cm}Lower bound $\underline{R}(\set{U})$}
\psfrag{A5}[l][][0.5]{\hspace{-0.1cm}Upper bound $\overline{R}(\set{U})$}
\psfrag{A6}[l][][0.5]{\hspace{-0.1cm}Error prob. $R(\up{h})$}
\psfrag{0.1}[][][0.4]{0.1}
\psfrag{0.2}[][][0.4]{0.2}
\psfrag{0.3}[][][0.4]{0.3}
\psfrag{0.04}[][][0.4]{0.04}
\psfrag{0.08}[][][0.4]{0.08}
\psfrag{0.12}[][][0.4]{0.12}
\psfrag{0.16}[][][0.4]{0.16}
\psfrag{100}[][][0.4]{100}
\psfrag{1000}[][][0.4]{1000}
\psfrag{2000}[][][0.4]{2000}
\psfrag{3000}[][][0.4]{3000}
\psfrag{4000}[][][0.4]{4000}
\psfrag{5000}[][][0.4]{5000}

\subfigure[Adult dataset]{\includegraphics[width=0.48\textwidth]{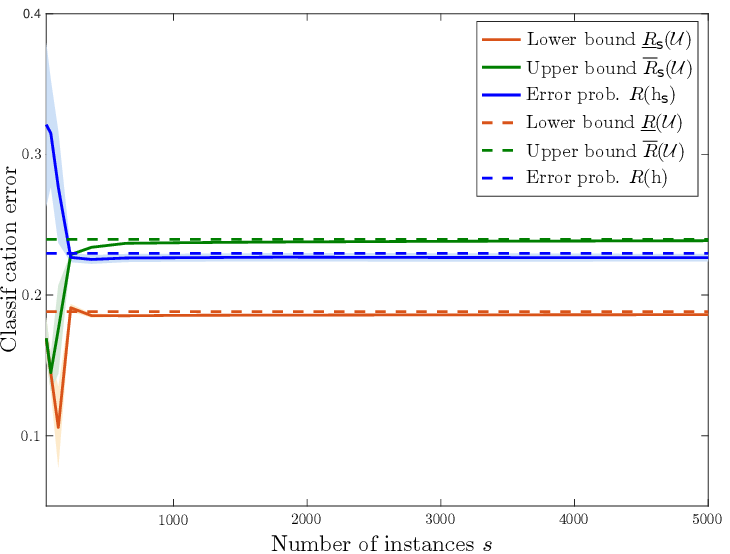}}
\hfill
\subfigure[Pulsar dataset]{\includegraphics[width=0.48\textwidth]{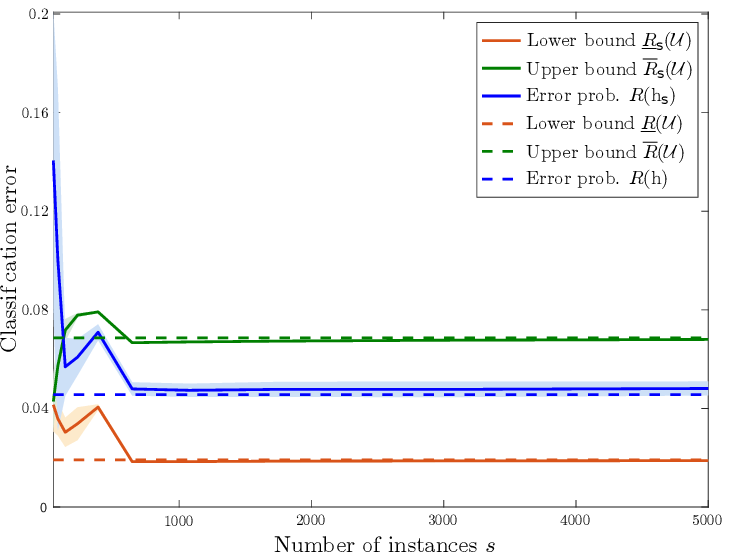}}
	\caption{Decrease of optimization error with the number of instances used.\label{fig_reduced}}
\end{figure}

In Figure~\ref{fig_reduced},  $R(\up{h}_s)$, $\underline{R}_s(\set{U})$ and $\overline{R}_s(\set{U})$ denote, as in Theorem~\ref{th_8}, the error probability and bounds of the \ac{MRC} obtained solving \eqref{opt-prob} and \eqref{opt-lower} taking $\set{X}=\set{X}_s=\{x_1,x_2,\ldots,x_s\}$.  Solid curves describe the average results in 50 random repetitions, the shaded areas describe the intervals formed by the $\pm$ standard deviation around the averages, and dashed lines describe the results obtained taking $\set{X}$ composed by all the instances. The figure shows that the results obtained using a reduced set of instances quickly converge to those obtained using all the instances. These results agree with the theoretical guarantees for this approximation shown in Theorem~\ref{th_8} and show that the constants of the bounds in such result can be quite small in practice. In the remaining numerical results we solve optimization problems \eqref{opt-prob}, \eqref{opt-lower},  and \eqref{opt-upper} taking $\set{X}$ as the $n$ instances at training.

\begin{figure}

\psfrag{1}[][][0.4]{1}
\psfrag{0}[][][0.4]{0}
\psfrag{2}[][][0.4]{2}
\psfrag{3}[][][0.4]{3}
\psfrag{4}[][][0.4]{4}
\psfrag{-1}[][][0.4]{\hspace{-0.4cm}$10^{-1}$}
\psfrag{-2}[][][0.4]{\hspace{-0.4cm}$10^{-2}$}
\psfrag{-3}[][][0.4]{\hspace{-0.4cm}$10^{-3}$}
\psfrag{2}[][][0.4]{2}
\psfrag{3}[][][0.4]{3}
\psfrag{4}[][][0.4]{4}
\psfrag{5}[][][0.4]{5}
\psfrag{y}[][t][0.6]{\hspace{-0.45cm}Optimization error}
\psfrag{x}[][][0.6]{\hspace{-0.45cm}Running time [s]}
\psfrag{U1}[][][0.6]{\hspace{-0cm}PAC}
%\psfrag{U1}[l][][0.6]{\hspace{-0.9cm}$\begin{array}{c}\mbox{UB}\\\mbox{PAC}\end{array}$}
\psfrag{U2}[][][0.6]{\hspace{-0cm}Upper B.}
\psfrag{L2}[][][0.6]{\hspace{-0cm}Lower B.}
\psfrag{U3}[][][0.6]{\hspace{-0cm}Upper B.}
\psfrag{L3}[][][0.6]{\hspace{-0cm}Lower B.}
\psfrag{U4}[][][0.6]{\hspace{-0cm}Upper B.}
\psfrag{L4}[][][0.6]{\hspace{-0cm}Lower B.}
\psfrag{A1234567890123}[][][0.5]{\hspace{-.93cm}BSM}
\psfrag{A3}[][][0.5]{\hspace{.6cm}ASM}
\psfrag{A2}[][][0.5]{\hspace{1cm}E-BSM}
\psfrag{A4}[][][0.5]{\hspace{1.cm}E-ASM}
\psfrag{A5}[][][0.5]{\hspace{1.4cm}E-ASM-R}
\centering
\subfigure[Credit dataset]{\includegraphics[width=0.48\textwidth]{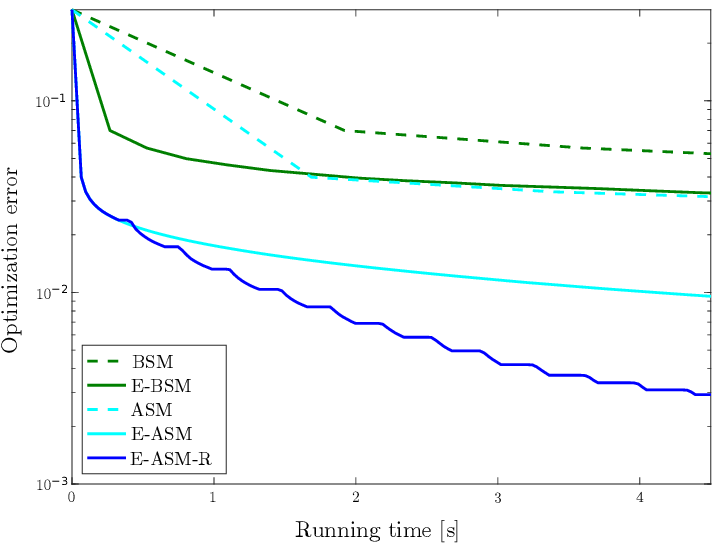}}\hfill
\subfigure[Haberman dataset]{\includegraphics[width=0.48\textwidth]{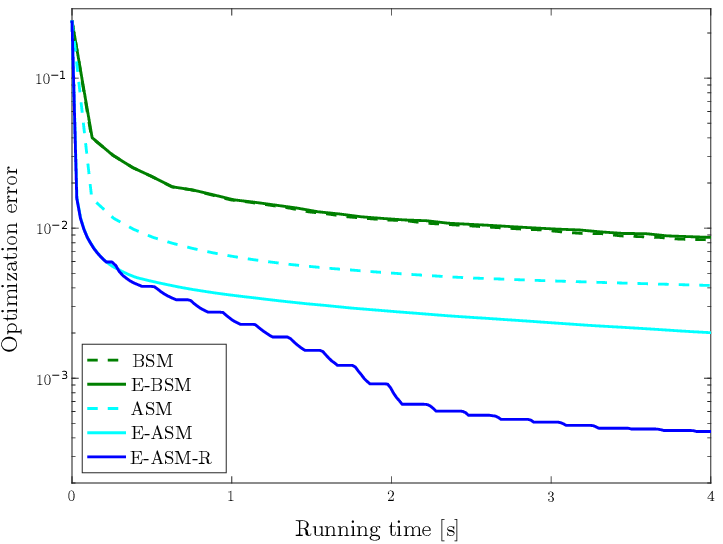}}
\subfigure[Mammographic dataset]{\includegraphics[width=0.48\textwidth]{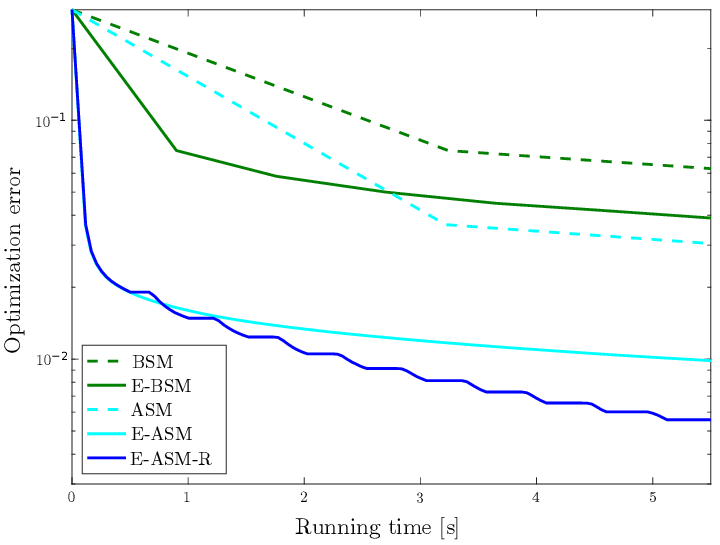}}\hfill
\subfigure[Diabetes dataset]{\includegraphics[width=0.48\textwidth]{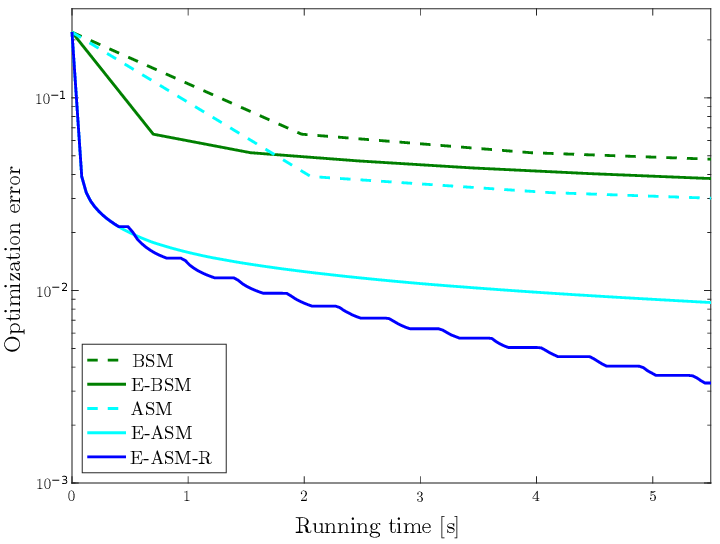}}
\caption{The proposed efficient \ac{ASM}'s implementation that exploits the subgradients' structure can significantly reduce the running time required to accurately learn \acp{MRC}.\label{fig:opt}}
\end{figure}

The second set of numerical results assesses the performance of the efficient SMs presented in Section~\ref{ssc:ep}.
We solve problem \eqref{opt-prob} for \ac{MRC} learning using four methods: the basic SM (BSM) with step size $c_k=1/(\sqrt{k+1}\|\V{g}_k\|)$; the efficient BSM that exploits the subgradient structure as shown in Section~\ref{ssc:ep} (E-BSM); the \ac{ASM} described in Algorithm~\ref{alg.Tao}; the efficient \ac{ASM} that exploits the subgradient structure as detailed in Algorithm~\ref{alg.Tao-1} (E-ASM); and the efficient \ac{ASM} described in Algorithm~\ref{alg.Tao-1} combined with restarts every 10,000 iterations (E-ASM-R). %Specifically, such restarts are carried out each 10,000 or the objective's value is decreased at least $90\%$ of the its value at the previous restart.

Figure \ref{fig:opt} shows the optimization error per running time averaged over 10 random partitions of four datasets. The results manifest the significant computational savings obtained by the efficient implementation presented in Section~\ref{ssc:ep}. Such efficient implementation is especially advantageous for \ac{ASM} since it results in a sparsity coefficient around $10^{-3}$ while that achieved in BSM is around $10^{-1}$. The restart strategy achieves worse computing time per iteration since it results in sparsity coefficients around $10^{-2}$. However, such strategy provides an overall improvement in running time since it requires significantly less iterations.

The third set of numerical results shows how the \acp{MRC} performance and bounds change by varying the confidence vectors used. 
We study the effect of varying the confidence vectors at learning and the performance guarantees obtained using confidence vectors with high coverage probability. In addition, we compare the results obtained in the practical case where confidence vectors are obtained from training samples with those obtained in the ideal case where the error in the mean vector estimates is known. Specifically, for the practical case we learn \acp{MRC} using $\B{\lambda}$ as in \eqref{confidence_exp} and obtain high-confidence performance guarantees using $\B{\lambda}_\delta$ for $\delta=0.05$ given by the confidence intervals proposed in \cite{WauRam:20}. For the ideal case, we learn \acp{MRC} using $\B{\lambda}=\lambda_0|\B{\tau}_\infty-\B{\tau}|$, and obtain high-confidence performance guarantees using $\B{\lambda}_\delta=|\B{\tau}_\infty-\B{\tau}|$.  In particular, a simple high-confidence upper bound is obtained using \eqref{conf-bound}, and a tighter high-confidence upper bound is obtained using Theorem~\ref{th-bounds} with $\B{\lambda}=\B{\lambda}_\delta$.  In order to reproduce the ideal case, in these numerical results we use ``Adult'' and ``Pulsar'' datasets, the mean $\B{\tau}_\infty$ is calculated using all the samples while $1,000$ samples are randomly sampled for training in $50$ repetitions.

% \color{blue} We learn \acp{MRC} using 

%$
%\B{\lambda}$ as in \eqref{confidence_exp} (practical case) and using $\B{\lambda}=\lambda_0|\B{\tau}_\infty-\B{\tau}|$ (ideal case) with $\B{\tau}_\infty$ approximated averaging with all the samples in the datasets. In addition,  high-confidence performance guarantees are obtained using $\B{\lambda}_\delta$ for $\delta=0.05$ given by the confidence intervals proposed in \cite{WauRam:20}. In particular, a simple high-confidence upper bound is obtained using \eqref{conf-bound}, and a tigther high-confidence upper bound is obtained using Theorem~\ref{th-bounds} with $\B{\lambda}=\B{\lambda}_\delta$. 

% \eqref{conf-bound} and Theorem~\ref{th-bounds} where $\B{\lambda}_\delta$ for $\delta=0.05$ is obtained using the confidence intervals proposed in \cite{WauRam:20}. \color{black}

%using $\B{\lambda}_\delta$ for $\delta=0.05$ obtained from the confidence intervals proposed in \cite{WauRam:20}. In addition, 
%we compare such results with those obtained in the ideal case in which the confidence vector at learning is taken as $\B{\lambda}
%=\lambda_0|\B{\tau}_\infty-\B{\tau}|$ and that used to obtain the high-confidence performance guarantees is taken as $|
%\B{\tau}_\infty-\B{\tau}|$. In order to reproduce such ideal case, in these numerical results we use ``Adult'' and ``Pulsar'' datasets, the mean $\B{\tau}_\infty$ is calculated using all the samples while $1,000$ samples are randomly sampled for training in $50$ repetitions. 

%\color{blue} in blue \color{black} 

\begin{figure}
%\psfrag{x}[l][b][0.6]{\hspace{-0.35cm} $\lambda_0$}
%\psfrag{y}[l][t][0.6]{\hspace{-1.57cm}Classification error}
%\psfrag{A1234567890123456789012345678901234567890}[l][][0.5]{\hspace{-2.65cm}Upper bound}
%\psfrag{A2}[l][][0.5]{\hspace{-0.1cm}Tight high conf. upper bound}
%\psfrag{A3}[l][][0.5]{\hspace{-0.1cm}Simple high conf. upper bound}
%\psfrag{A4}[l][][0.5]{\hspace{-0.1cm}Error probability}
%\psfrag{0.4}[][][0.4]{0.4}
%\psfrag{0.1}[][][0.4]{0.1}
%\psfrag{0}[][][0.4]{0}
%\psfrag{0.5}[][][0.4]{0.5}
%\psfrag{1}[][][0.4]{1}
%\psfrag{1.5}[][][0.4]{1.5}
%\psfrag{2}[][][0.4]{2}
%\psfrag{0.2}[][][0.4]{0.2}
%\psfrag{0.3}[][][0.4]{0.3}
%\psfrag{0.04}[][][0.4]{0.04}
%\psfrag{0.08}[][][0.4]{0.08}
%\psfrag{0.12}[][][0.4]{0.12}
%\psfrag{0.16}[][][0.4]{0.16}
%\psfrag{100}[][][0.4]{100}
%\psfrag{1000}[][][0.4]{1000}
%\psfrag{2000}[][][0.4]{2000}
%\psfrag{3000}[][][0.4]{3000}
%\psfrag{4000}[][][0.4]{4000}
%\psfrag{5000}[][][0.4]{5000}

\subfigure[Adult dataset, ideal case]{\includegraphics[width=0.48\textwidth]{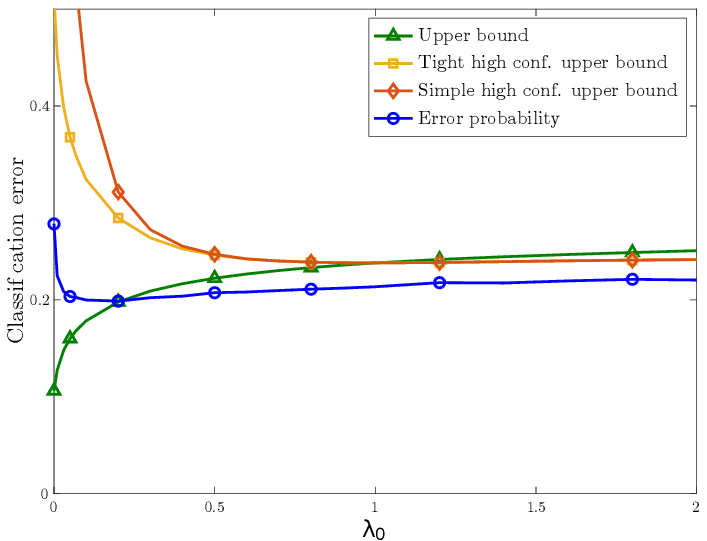}}
\hfill
\subfigure[Adult dataset, practical case]{\includegraphics[width=0.48\textwidth]{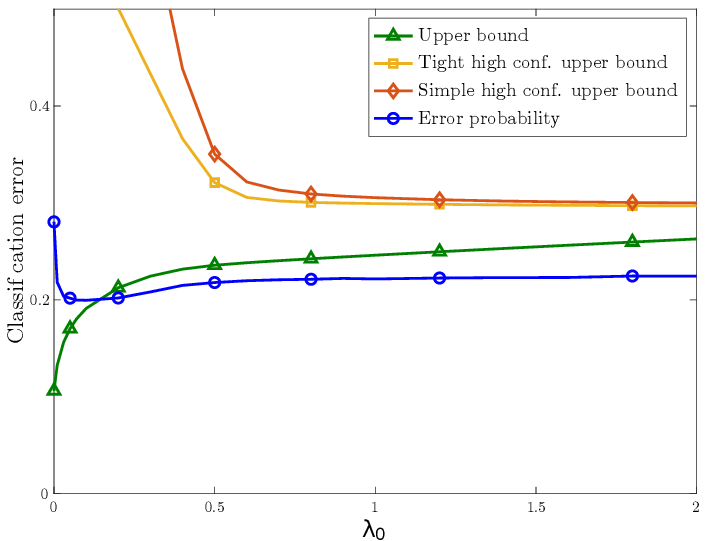}}
\subfigure[Pulsar dataset, ideal case]{\includegraphics[width=0.48\textwidth]{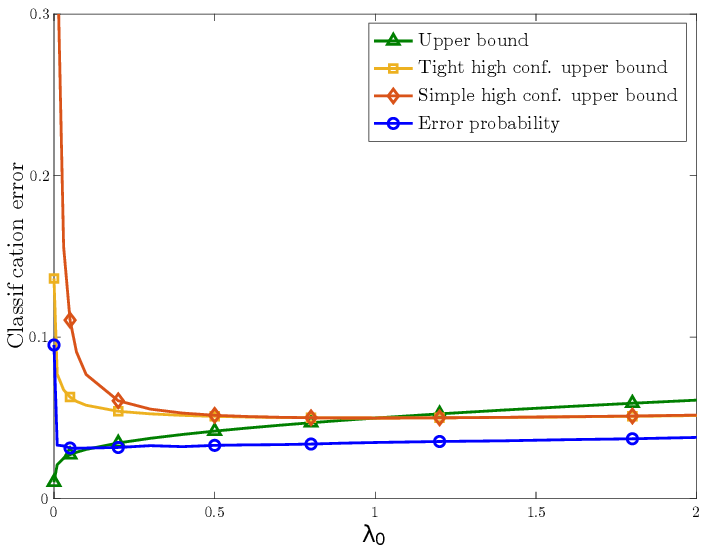}}
\hfill
\subfigure[Pulsar dataset, practical case]{\includegraphics[width=0.48\textwidth]{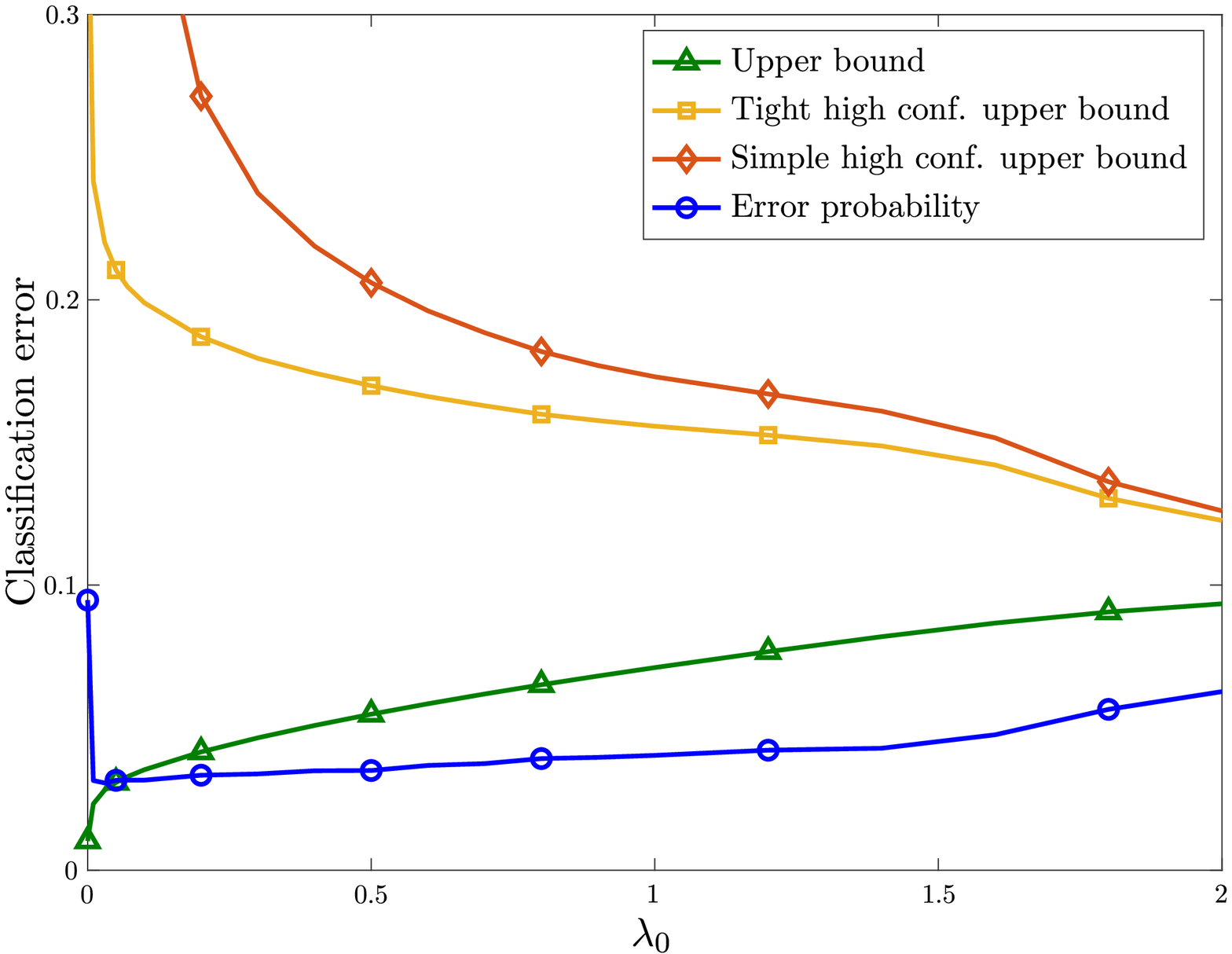}}
	\caption{Error probabilities and performance bounds of \acp{MRC} for different choices of confidence vectors in comparison with the ideal case that use the actual error in mean vector estimates.\label{fig_lambda}}	
	
\end{figure}

Figure~\ref{fig_lambda} shows the error probability $R(\up{h}^{\sset{U}})$ and upper bound $\overline{R}(\set{U})$ of \acp{MRC} corresponding with different values of $\lambda_0$, together with the simple high-confidence upper bound obtained using \eqref{conf-bound}, and the tighter high-confidence upper bound obtained using Theorem~\ref{th-bounds}. Such figure shows that \acp{MRC} do not heavily rely on the choice of the confidence vector. In particular, the error probability obtained using $\B{\lambda}$ as in \eqref{confidence_exp} is similar to the error probability that would be obtained using the actual difference $|
\B{\tau}_\infty-\B{\tau}|$, and the minimax risk $\overline{R}(\set{U})$ optimized at learning is similar to the error probability for most values of $\lambda_0$. In addition, the usage of confidence vectors with high coverage probability can result in performance guarantees that are both valid with high probability and informative. Such numerical results are in agreement with the conclusions drawn from Theorem~\ref{th_6} in Section~\ref{sec-4.2}. In particular, Figure~\ref{fig_lambda} shows that the smallest valid upper bound is obtained using $\B{\lambda}=|\B{\tau}_\infty-\B{\tau}|$ but more aggressive confidence vectors can result in improved error probability and yet provide useful performance bounds.

\begin{figure}

\psfrag{x1}[l][b][0.6]{\hspace{-0.2cm}$\lambda_0$}
\psfrag{x2}[l][b][0.6]{\hspace{-0.2cm}$\rho$}
\psfrag{y}[l][t][0.6]{\hspace{-1.57cm}Classification error}
\psfrag{A1234567890123456789012345678}[l][][0.5]{\hspace{-1.92cm}Upper bound $\overline{R}(\set{U})$}
\psfrag{A2}[l][][0.5]{\hspace{-0.15cm}Lower bound $\underline{R}(\set{U})$}
\psfrag{A3}[l][][0.5]{\hspace{-0.15cm}Error prob. $R(\up{h}^{\sset{U}})$}
\psfrag{A4}[l][][0.5]{\hspace{-0.17cm}Upper bound $\overline{R}(\set{U},\up{h}^{\sset{U}}_{\text{d}})$}
\psfrag{A5}[l][][0.5]{\hspace{-0.17cm}Lower bound $\underline{R}(\set{U},\up{h}^{\sset{U}}_{\text{d}})$}
\psfrag{A6}[l][][0.5]{\hspace{-0.17cm}Error prob. $R(\up{h}^{\sset{U}}_{\text{d}})$}
\psfrag{A123456789012345678}[l][][0.5]{\hspace{-1.27cm}Upper bound}
\psfrag{B2}[l][][0.5]{\hspace{-0.15cm}Lower bound}
\psfrag{B3}[l][][0.5]{\hspace{-0.15cm}Error prob.}
\psfrag{0.1}[][][0.4]{0.1}
\psfrag{0.2}[][][0.4]{0.2}
\psfrag{0.3}[][][0.4]{0.3}
\psfrag{0.4}[][][0.4]{0.4}
\psfrag{0.6}[][][0.4]{0.6}
\psfrag{0.8}[][][0.4]{0.8}
\psfrag{-5}[][][0.4]{$10^{-5}$}
\psfrag{-4}[][][0.4]{$10^{-4}$}
\psfrag{-3}[][][0.4]{$10^{-3}$}
\psfrag{-2}[][][0.4]{$10^{-2}$}
\psfrag{-1}[][][0.4]{$10^{-1}$}
\psfrag{1}[][][0.4]{1}
\psfrag{0.04}[][][0.4]{0.04}
\psfrag{0.08}[][][0.4]{0.08}
\psfrag{0.12}[][][0.4]{0.12}
\psfrag{0.16}[][][0.4]{0.16}
\psfrag{100}[][][0.4]{100}
\psfrag{1000}[][][0.4]{1000}
\psfrag{2000}[][][0.4]{2000}
\psfrag{3000}[][][0.4]{3000}
\psfrag{4000}[][][0.4]{4000}
\psfrag{5000}[][][0.4]{5000}

\subfigure[\ac{MRC} performance bounds and classification error in ``credit'' dataset.]{\includegraphics[width=0.48\textwidth]{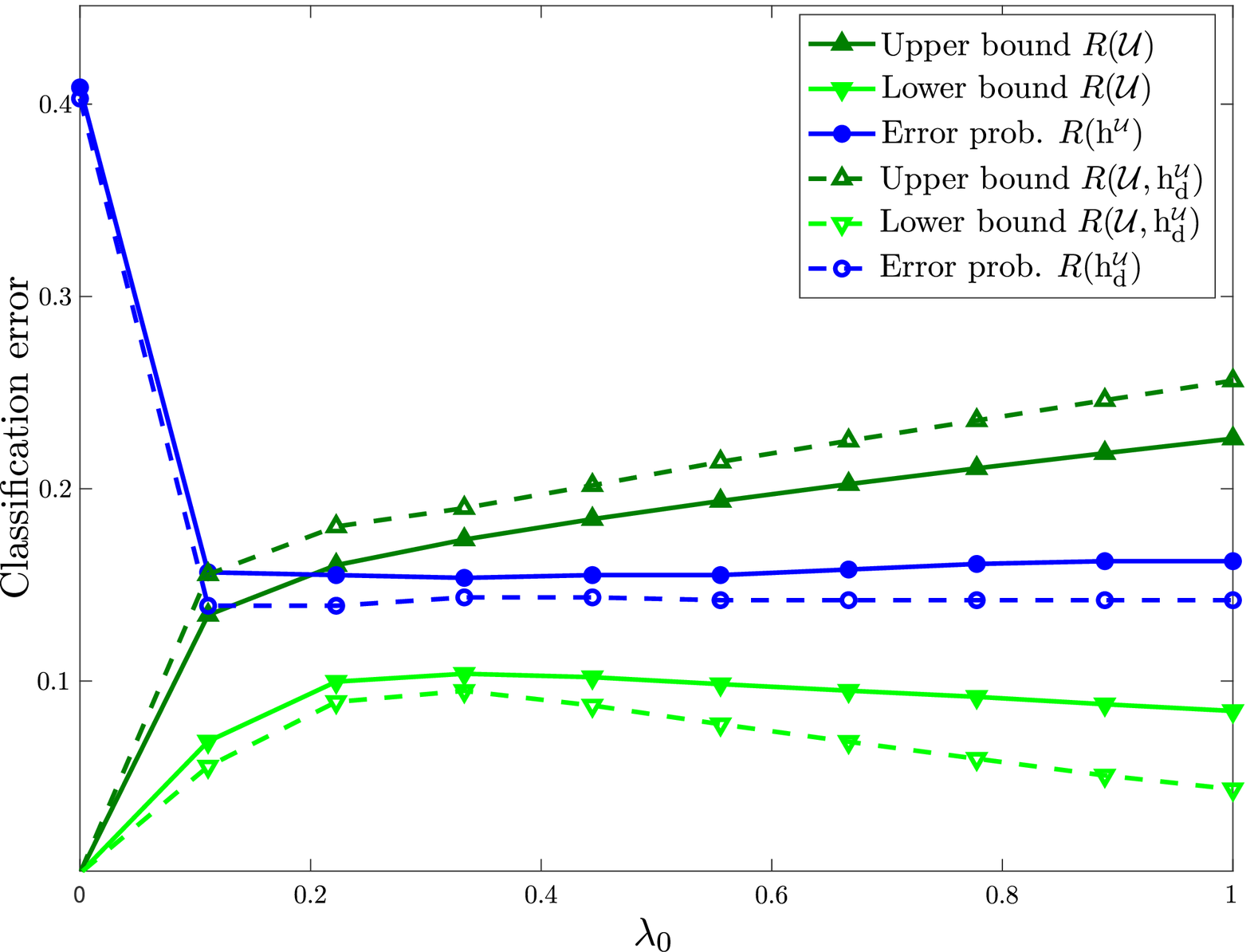}\label{fig_bounds-l1}}
\hfill
\subfigure[\ac{DRLR} performance bounds and classification error in ``credit'' dataset.]{\includegraphics[width=0.48\textwidth]{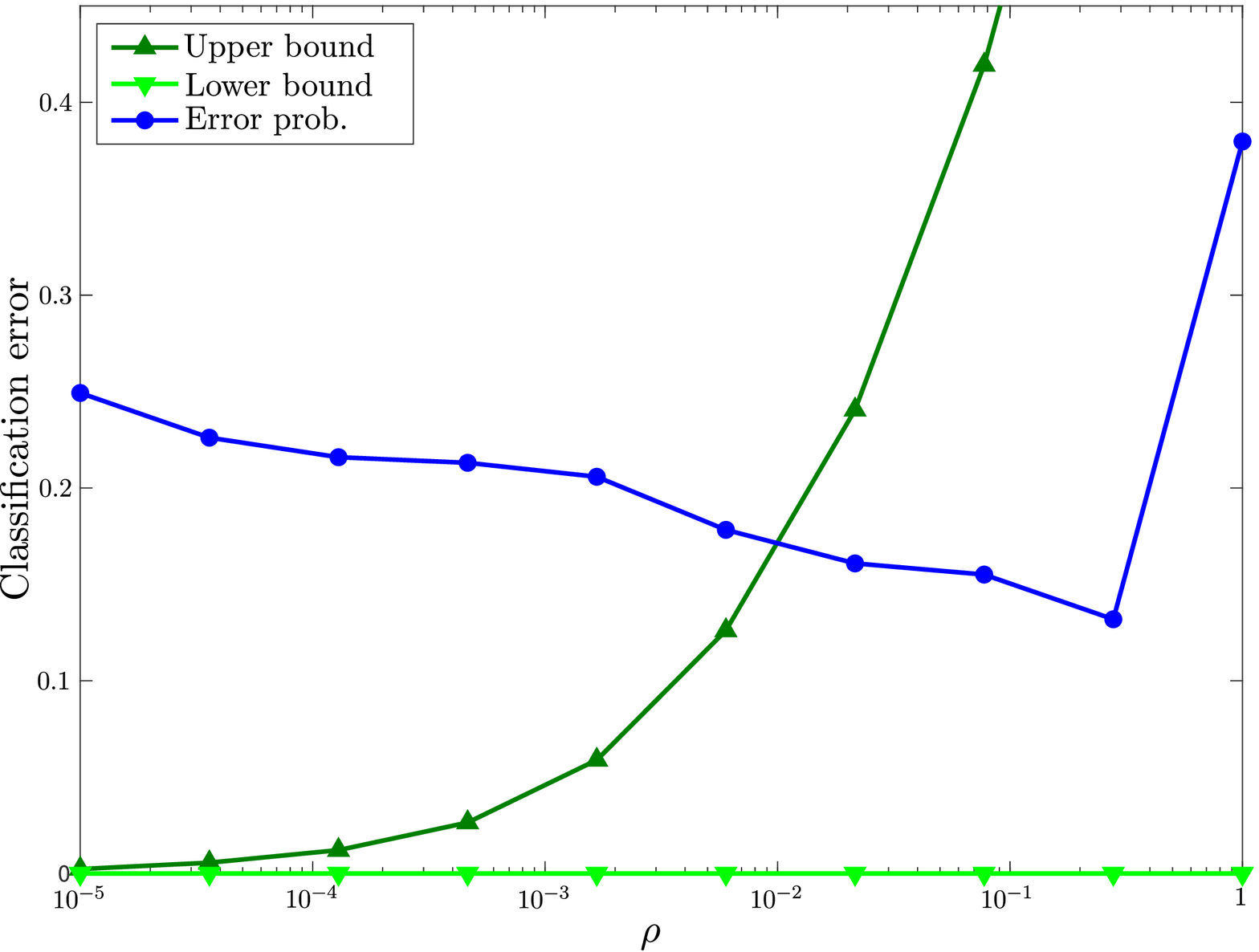}\label{fig_bounds-r1}}
\subfigure[\ac{MRC} performance bounds and classification error in ``QSAR'' dataset.]{\includegraphics[width=0.48\textwidth]{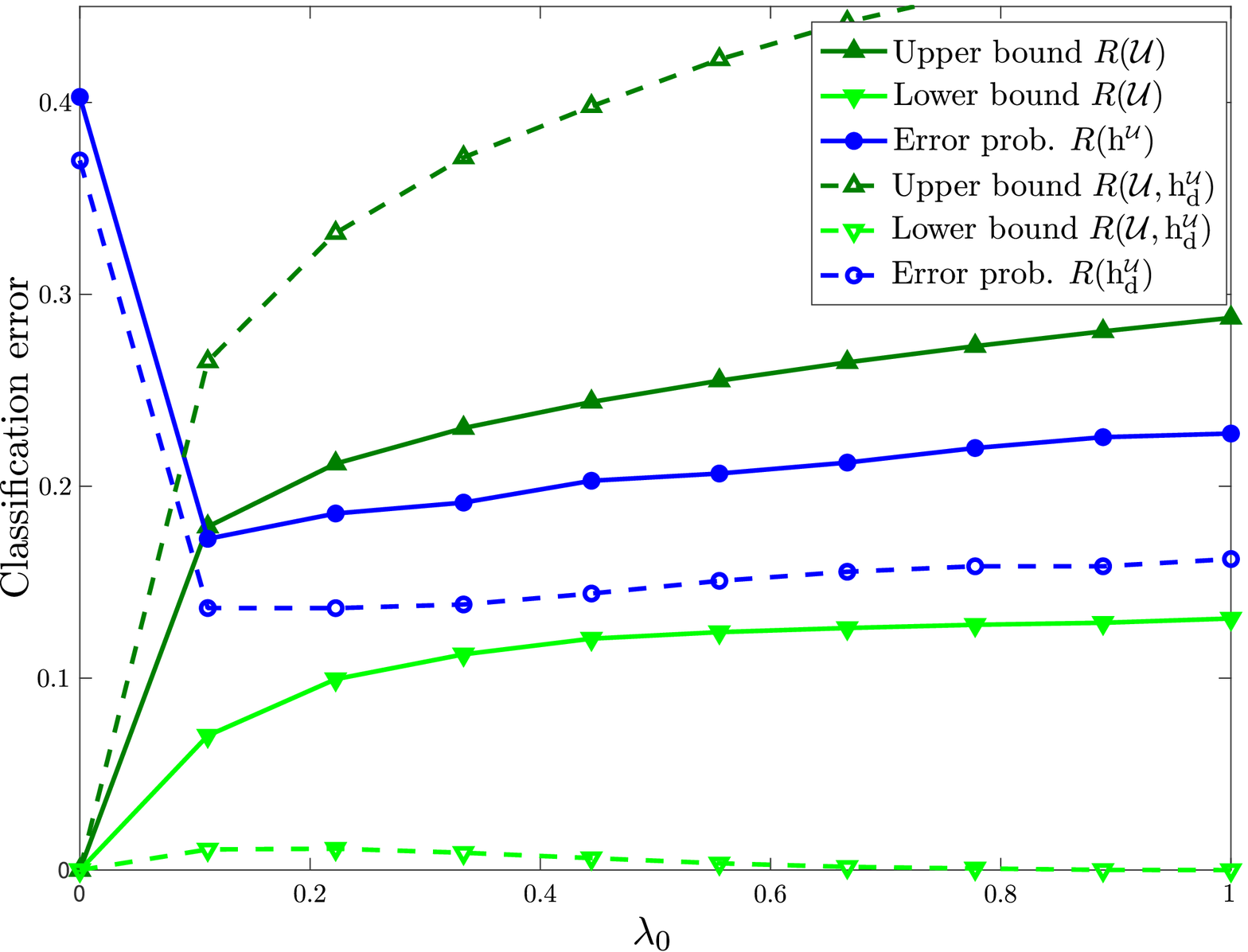}\label{fig_bounds-l2}}
\hfill
\subfigure[\ac{DRLR} performance bounds and classification error in ``QSAR'' dataset.]{\includegraphics[width=0.48\textwidth]{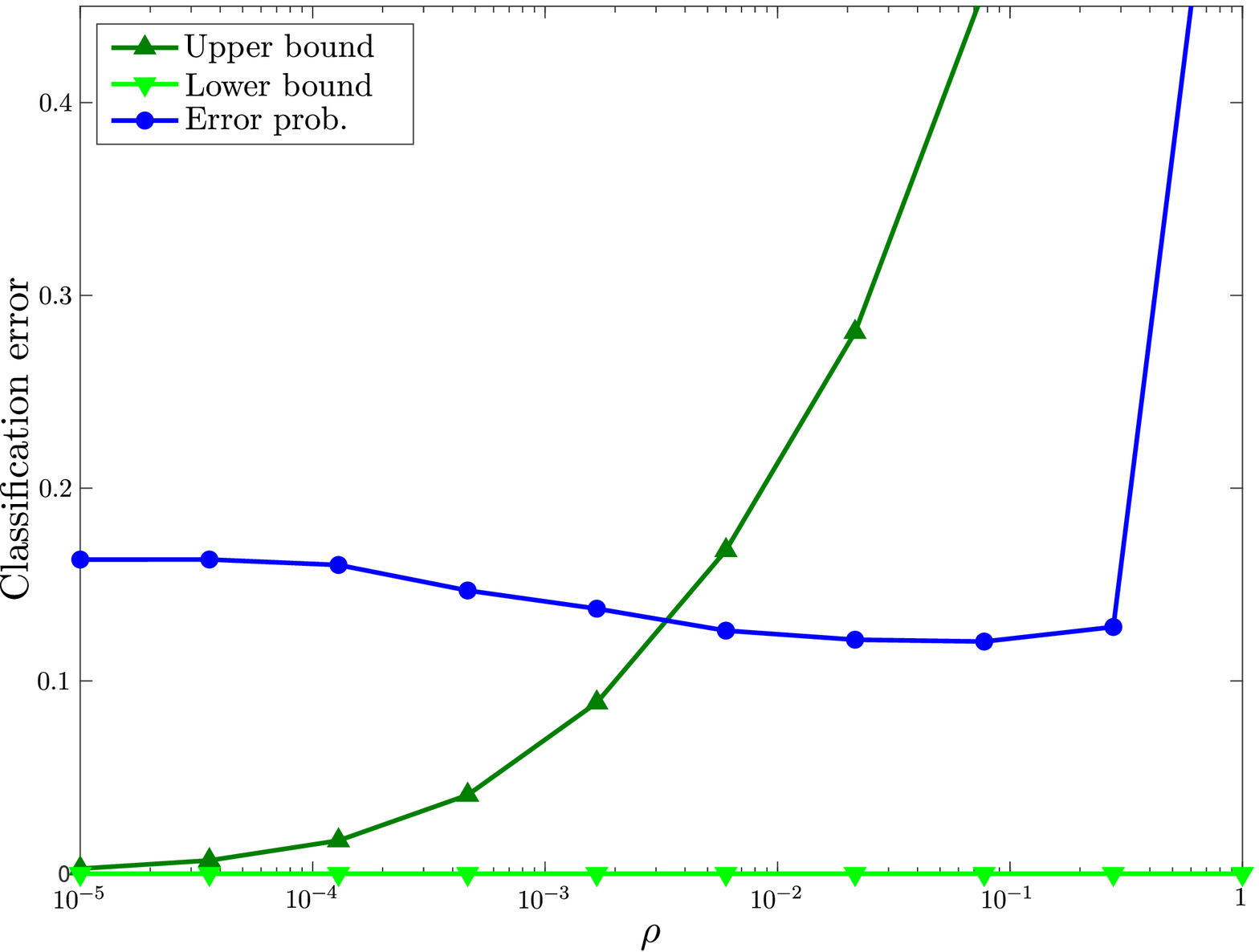}\label{fig_bounds-r2}}
\subfigure[\ac{MRC} performance bounds and classification error in ``mammographic'' dataset.]{\includegraphics[width=0.48\textwidth]{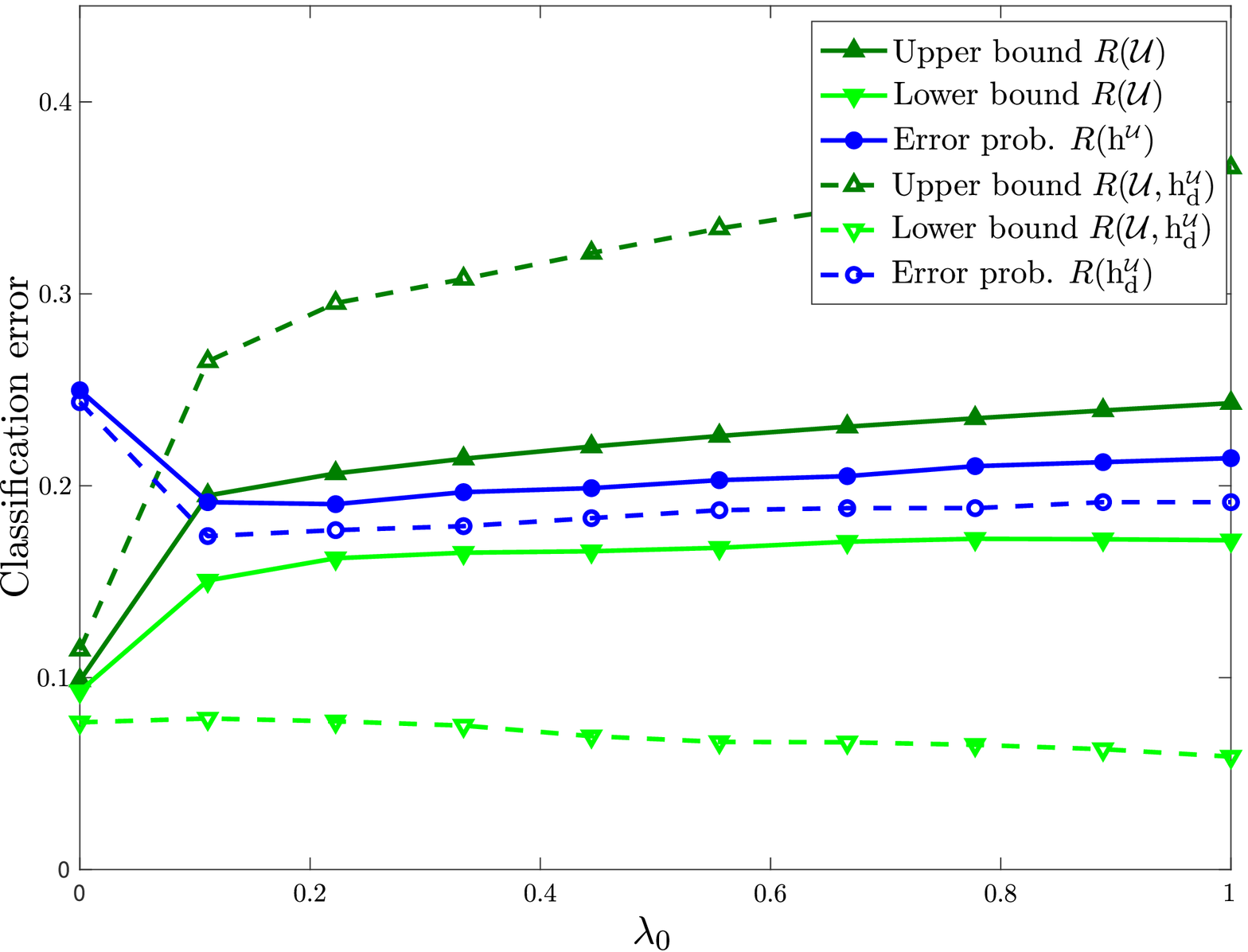}\label{fig_bounds-l3}}
\hfill
\subfigure[\ac{DRLR} performance bounds and classification error in ``mammographic'' dataset.]{\includegraphics[width=0.48\textwidth]{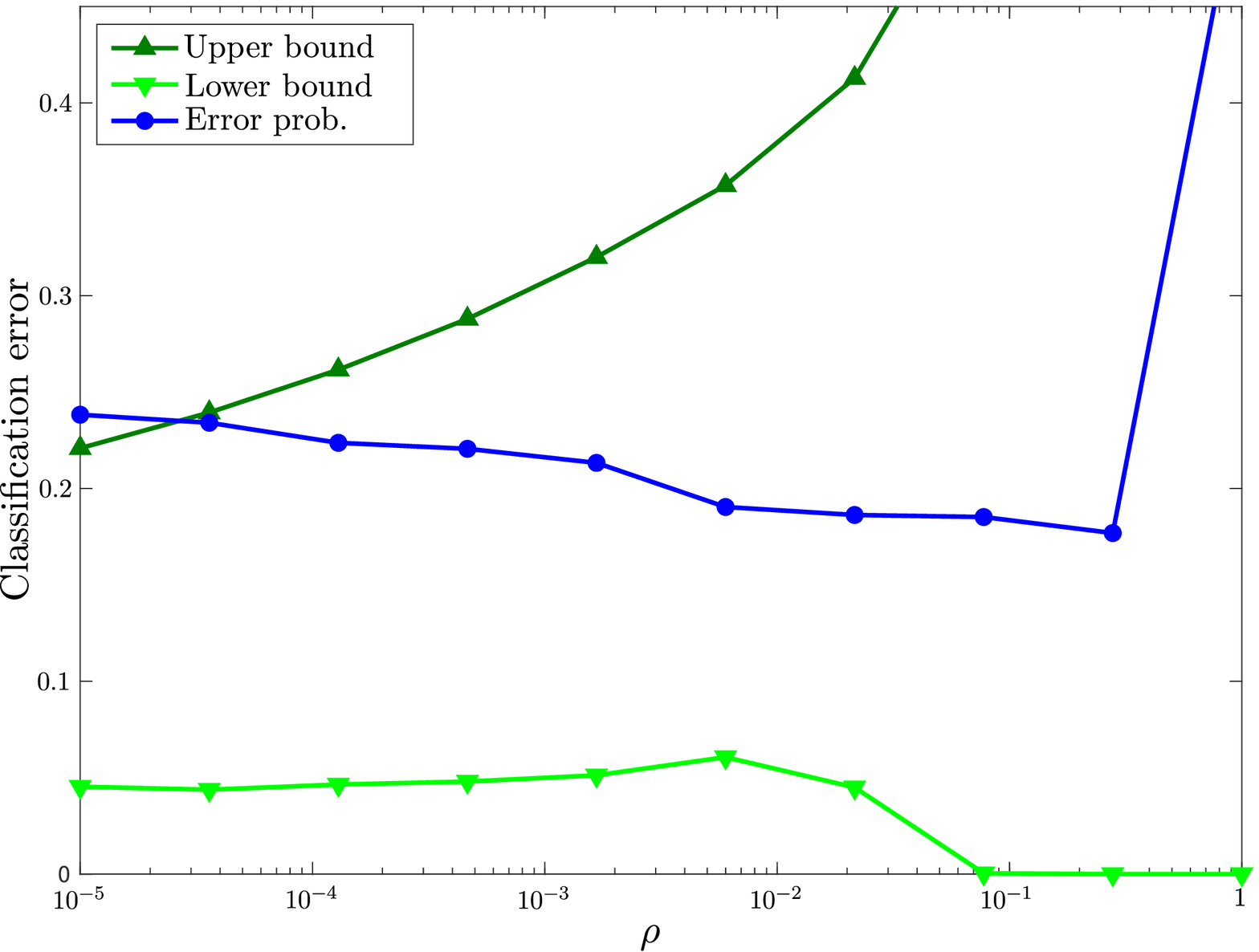}\label{fig_bounds-r3}}
	\caption{\acp{MRC} performance bounds for varying confidence vectors in comparison with those obtained by \ac{RRM} techniques based on Wasserstein distances.\label{fig_bounds}}
\end{figure}

We also compare the performance and bounds of \acp{MRC} for varying confidence vectors with those obtained by \acp{DRLR} for varying Wasserstein radius. Figure~\ref{fig_bounds} shows the error probability and performance bounds of \acp{MRC} and \acp{DRLR} obtained by varying $0\leq\lambda_0\leq1$ and the Wasserstein radius $10^{-5}\leq\rho\leq 1$, respectively. In particular, for each value of $\lambda_0$ and $\rho$ we averaged over 10-fold stratified partitions the error and bounds of \acp{MRC} and \acp{DRLR}. Figure~\ref{fig_bounds} shows that \acp{MRC} can provide tighter performance bounds than existing techniques based on \ac{RRM} with Wasserstein distances. The figure also shows that, even for datasets of similar size, \ac{DRLR} methods require to fine-tune the Wasserstein radius $\rho$ in order to obtain reliable and tight performance bounds while \acp{MRC} can utilize a confidence parameter $\lambda_0$ that does not strongly depend on the dataset. Figure~\ref{fig_bounds} also shows that deterministic 0\,-1 \acp{MRC}, $\up{h}_\text{d}^{\sset{U}}$, can obtain improved classification error in practice at the expenses of less tight performance guarantees.

The fourth set of numerical results shows how \acp{MRC}' performance bounds can be used for model selection. We compare the performance obtained by selecting the scaling parameter of a Gaussian kernel using four methods, one based on conventional cross-validation and three based on error upper bounds. Specifically, SVM-CV selects the scaling parameter for which a \ac{SVM} classifier achieves the smallest cross-validation error over 10-fold partitions of the training data. SVM-PAC, \ac{DRLR}-UB, and \ac{MRC}-UB select the scaling parameter with smallest error upper bound in training. Such upper bounds are obtained for SVM-PAC by using PAC-Bayes methods as in \cite{Lan:05,AmbParSha:07}, while for \ac{DRLR}-UB and \ac{MRC}-UB the upper bounds are obtained as the value the optimization problem solved at learning as shown in \cite{AbaMohKuh:15} and in equation \eqref{opt-prob}, respectively. For all datasets, \acp{MRC} utilize confidence vectors as in \eqref{confidence_exp} with $\lambda_0=0.3$ and \acp{DRLR} utilize Wasserstein radious of $\rho=0.003$ as in \cite{AbaMohKuh:15}.

\begin{table}
\caption{\small Classification error of \ac{MRC} with model selection based on the performance bounds in comparison with state-of-the-art techniques.}
%\vspace{0.1cm}
\label{table:results}
\centering
%\def\arraystretch{1.3}
%\resizebox{\textwidth}{!}{%
\begin{tabular}{lccccc}%\\[-0.45cm]
\toprule Data set&SVM-CV&SVM-PAC&\ac{DRLR}-UB&\ac{MRC}-UB&Det. \ac{MRC}-UB\\[0.cm]
\hline
Haberman&.26&.27&.32&.25&.25\\
Heart&.17&.17&.23&.21&.18\\
Liver&.28&.28&.33&.29&.28\\
Blood&.22&.23&.25&.24&.22\\
Credit&.14&.13&.20&.15&.14\\
Diabetes&.23&.23&.30&.27&.24\\
Ion&.08&.08&.07&.12&.07\\
QSAR&.11&.12&.14&.18&.12\\
Mammographic&.17&.17&.21&.19&.18\\
Audit&.05&.06&.04&.07&.06\\
\bottomrule
\end{tabular}
\end{table}

Table~\ref{table:results} shows the classification error obtained by the methods compared in 10 datasets. Specifically, we generate 20 random stratified splits with 20\% test samples. In each split, we utilize the training samples to learn a classifier selecting a scaling parameter over 20 uniformly spaced candidates between the 10th and 90th percentiles of the Euclidean distances among normalized instances. Then, the error of the corresponding classifier is estimated using the test samples and the final values in the table are obtained by averaging the results over the 20 random splits of the data.  Table~\ref{table:results} shows that the proposed \acp{MRC} as well as methods based on PAC-Bayes can obtain similar performance to that obtained by conventional methods based on cross-validation. However, \acp{MRC} and PAC-Bayes methods require a significantly smaller complexity since they only need to carry out one optimization per parameter value while cross-validation methods require to carry out such optimization 10 times per parameter value. In addition, the performance bounds obtained by \acp{MRC} can be used not only for model selection but also to obtain tight estimates for the classification error. Figure~\ref{fig:bounds} shows the differences between performance bounds and error probabilities for all the splits, datasets, and hyper-parameters in this fourth set of numerical results.  The figure shows that the performance bounds provided by \acp{MRC} are much tighter than those based on PAC-Bayes and much more reliable that those based on \ac{RRM} that use Wasserstein balls without a fine-tuned radious. In particular, the difference between the classification error and the performance bounds are around 0.05 for \acp{MRC} and around 0.1 for deterministic \acp{MRC}.

\begin{figure}
\psfrag{M1}[l][][0.7]{\hspace{-0.4cm}SVM}
\psfrag{M2}[l][][0.7]{\hspace{-0.5cm}DRLR}
\psfrag{M3}[l][][0.7]{\hspace{-0.5cm}MRC}
\psfrag{M4}[l][][0.7]{\hspace{-1cm}Det. MRC}
\psfrag{-0.1}[][][0.4]{-0.1}
\psfrag{0}[][][0.4]{0}
\psfrag{0.1}[][][0.4]{0.1}
\psfrag{0.2}[][][0.4]{0.2}
\psfrag{0.3}[][][0.4]{0.3}
\psfrag{y}[][][0.6]{\hspace{-0.45cm}Difference bound and error}
\psfrag{U1}[][][0.6]{\hspace{-0cm}PAC}
%\psfrag{U1}[l][][0.6]{\hspace{-0.9cm}$\begin{array}{c}\mbox{UB}\\\mbox{PAC}\end{array}$}
\psfrag{U2}[][][0.6]{\hspace{-0cm}Upper B.}
\psfrag{L2}[][][0.6]{\hspace{-0cm}Lower B.}
\psfrag{U3}[][][0.6]{\hspace{-0cm}Upper B.}
\psfrag{L3}[][][0.6]{\hspace{-0cm}Lower B.}
\psfrag{U4}[][][0.6]{\hspace{-0cm}Upper B.}
\psfrag{L4}[][][0.6]{\hspace{-0cm}Lower B.}
\psfrag{A1234567890123456789012345678}[l][][0.5]{\hspace{-1.92cm}Upper bound $\overline{R}(\set{U})$}
\centering
\includegraphics[width=0.6\textwidth]{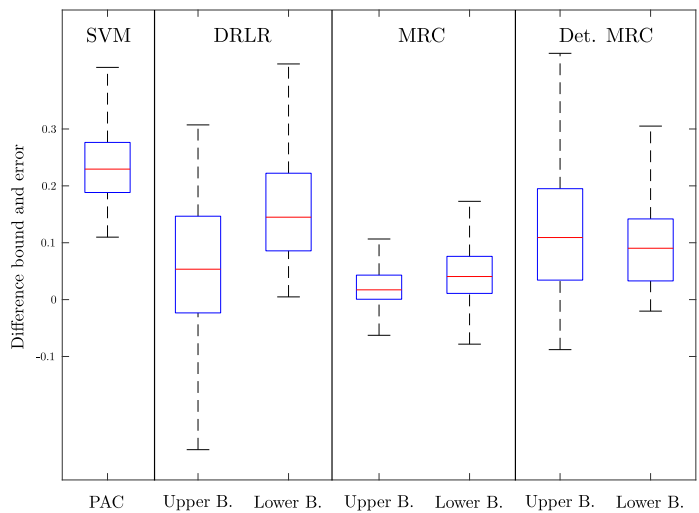}
\caption{Boxplots with differences between error probability and bounds for PAC, \ac{DRLR}, and \ac{MRC} methods over multiple datasets and hyperparameter configurations.\label{fig:bounds}}
\end{figure}

\section{Conclusion}\label{sec-conclusion}
%The paper presents minimax risk classifiers (MRCs) that minimize the worst-case 0\,-1 loss over general classification rules and provide tight performance guarantees. We show how the out-of-sample performance of \acp{MRC} can be reliably estimated at learning, and that the \acp{MRC}' error due to finite training sizes is determined by the accuracy of expectation estimates. In addition, we show that \acp{MRC} are strongly universally consistent in situations analogous as those corresponding with kernel-based methods. The proposed methodology can offer classification techniques that provide competitive performance together with tight performance guarantees. In addition, \acp{MRC} do not rely on specific training samples but on expectation estimates. Therefore, the methodology proposed can provide techniques that are robust to practical situations that defy common assumptions, e.g., training samples that follow a different distribution or display heavy tails.

The paper presents minimax risk classifiers (MRCs) that minimize the worst-case 0\,-1 loss over general classification rules and provide tight performance guarantees. We show how the out-of-sample performance of \acp{MRC} can be reliably estimated at learning, and that the \acp{MRC}' error due to finite training sizes is determined by the accuracy of expectation estimates. In addition, we show that \acp{MRC} are strongly universally consistent in situations analogous to those corresponding with kernel-based methods. 
The proposed methodology can offer classification techniques that do not rely on specific training samples and choices for surrogate losses/hypothesis classes. Instead, \ac{MRC} learning is based on expectation estimates, and its inductive bias comes only from a feature mapping that determines which expectations are estimated. Therefore, the methods presented can provide techniques that are robust to practical situations that defy common assumptions, e.g., training samples that follow a different distribution or display heavy tails.

%serves to constrain the distributions considered. 
\newpage
\acks{Funding in direct support of this work has been provided
by projects PID2022-137063NB-I00, CNS2022-135203, and CEX2021-001142-S funded by MCIN/AEI/10.13039/501100011033 and the European Union ``NextGenerationEU''/PRTR, and programes ELKARTEK and BERC-2022-2025 funded by the Basque Government.}

\appendix
\section{}\begin{lemma}\label{lemma}
Let $\set{U}$ be an uncertainty set given by \eqref{uncertainty}. For any $\up{h}\in T(\set{X},\set{Y})$, we have that
\begin{align}\label{wduality}
\sup_{\up{p}\in\set{U}} \ell(\up{h},\up{p})\leq\inf_{\B{\mu}\in\mathbb{R}^m} 1-\B{\tau}^{\text{T}}\B{\mu}+\sup_{x\in\set{X},y\in\set{Y}}\big\{\Phi(x,y)^{\text{T}}\B{\mu}-\up{h}(y|x)\big\}+\B{\lambda}^{\text{T}}|\B{\mu}|
\end{align}
In addition, if $\set{U}$ that satisfies R1 or R2, we have that
\begin{align}\label{duality}
\sup_{\up{p}\in\set{U}} \ell(\up{h},\up{p})=\min_{\B{\mu}\in\mathbb{R}^m} 1-\B{\tau}^{\text{T}}\B{\mu}+\sup_{x\in\set{X},y\in\set{Y}}\big\{\Phi(x,y)^{\text{T}}\B{\mu}-\up{h}(y|x)\big\}+\B{\lambda}^{\text{T}}|\B{\mu}|
\end{align}
\end{lemma}
\begin{proof}

In the first step of the proof we show that the right hand side of \eqref{wduality} is equivalent to the Fenchel dual of the left hand side, then the second step of the proof shows that strong duality holds if  $\set{U}$ satisfies R1 or R2.

Let $0\leq r\leq m$ be the number of non-zero components of $\B{\lambda}$ (without loss of generality we assume such components are the $r$ first components of $\B{\lambda}$), and $M(\set{X}\times\set{Y})$ be the set of signed Borel measures over $\set{X}\times\set{Y}$ with bounded total variation. The set $M(\set{X}\times\set{Y})$ is a Banach space with the total variation norm (see e.g., Chapter 10 in \cite{ChaKim:94}).

%and its dual is the set $F(\set{X}\times\set{Y})$ of bounded Borel measurable functions over $\set{X}\times\set{Y}$ (see e.g., . 

If $A$ is the linear mapping
\begin{align}
\begin{array}{cccl}A:\hspace{-0.1cm}&M(\set{X}\times\set{Y})&\hspace{-0.1cm}\to&\hspace{-0.1cm}\mathbb{R}^{r+m+1}\\
\hspace{-0.1cm}&\up{p}&\hspace{-0.1cm}\mapsto&\hspace{-0.1cm}[\int\Phi_1(x,y)d\up{p}(x,y),-\int\Phi_1(x,y)d\up{p}(x,y),\int\Phi_2(x,y)d\up{p}(x,y),\int d\up{p}(x,y)]\end{array}
\end{align} 
where $\Phi_1(x,y)\in\mathbb{R}^r$ denotes the first $r$ components of $\Phi(x,y)$ and $\Phi_2(x,y)\in\mathbb{R}^{m-r}$ denotes the last $m-r$ components of $\Phi(x,y)$.
$A$ is bounded and its adjoint operator is
\begin{align}
\begin{array}{cccl}A^*:&\mathbb{R}^{r+m+1}&\to&F(\set{X}\times\set{Y})\subseteq(M(\set{X}\times\set{Y}))^*\\
&\B{\mu}_1,\B{\mu}_2,\B{\mu}_3,\nu&\mapsto&\Phi_1(\pun,\pun)^{\text{T}}(\B{\mu}_1-\B{\mu}_2)+\Phi_2(\pun,\pun)^{\text{T}}\B{\mu}_3+\nu.
\end{array}
\end{align} 
where $F(\set{X}\times\set{Y})$ is the set of bounded Borel measurable functions over $\set{X}\times\set{Y}$.

Then, we have that
\begin{align}\label{primal}\sup_{\up{p}\in\set{U}} \ell(\up{h},\up{p})=\sup_{\up{p}\in\set{U}} 1-\int \up{h}(y|x)d\up{p}(x,y)=1-\inf_{\up{p}\in M(\set{X}\times\set{Y})} f(\up{p})+g(A(\up{p}))\end{align}
where $f$ and $g$ are the lower semi-continuous convex functions
\begin{align}
\begin{array}{cccl} g:&\mathbb{R}^{r+m+1}&\to&\mathbb{R}\cup\{\infty\}\\
&(\V{a}_1,\V{a}_2,\V{a}_3,b)&\mapsto&\left\{\begin{array}{ccc}0&\mbox{if}&\V{a}_1\preceq\B{\tau}_1+\B{\lambda}_1,\  \V{a}_2\preceq-\B{\tau}_1+\B{\lambda}_1,\V{a}_3=\B{\tau}_2,b=1\\\infty&\mbox{otherwise}&\end{array}\right.
\end{array}
\end{align}
for $\B{\tau}_1,\B{\lambda}_1\in\mathbb{R}^r$ (resp. $\B{\tau}_2,\B{\lambda}_2\in\mathbb{R}^{m-r}$)  given by the first $r$ (resp. last $m-r$) components of $\B{\tau}$ and $\B{\lambda}$, and 
\begin{align}
\begin{array}{cccl} f:&M(\set{X}\times\set{Y})&\to&\mathbb{R}\cup\{\infty\}\\
&\up{p}&\mapsto&\left\{\begin{array}{ccc}\int \up{h}(y|x)d\up{p}(x,y)&\mbox{if}&\up{p}\mbox{ is nonnegative}\\\infty&\mbox{otherwise.}&\end{array}\right.
\end{array}
\end{align}

Then, the Fenchel dual (see e.g., \cite{BorZhu:04}) of \eqref{primal} is 
\begin{align}\label{dual1}1-\sup_{\B{\mu}_1,\B{\mu}_2,\B{\mu}_3,\nu}-f^*(A^*(\B{\mu}_1,\B{\mu}_2,\B{\mu}_3,\nu))-g^*(-\B{\mu}_1,-\B{\mu}_2,-\B{\mu}_3,-\nu)\end{align}
where $f^*$ and $g^*$ are the conjugate functions of $f$ and $g$. If $w\in F(\set{X}\times\set{Y})$, we have that 
\begin{align*}f^*(w)&=\sup_{\up{p}\succeq 0}\int (w(x,y)-\up{h}(y|x))d\up{p}(x,y)\\
&=\left\{\begin{array}{ccc}0&\mbox{if}&w(x,y)\leq\up{h}(y|x),\  \forall x,y\in\set{X}\times\set{Y}\\\infty&\mbox{otherwise}&\end{array}\right.
\end{align*}
and $g^*(-\B{\mu}_1,-\B{\mu}_2,-\B{\mu}_3,-\nu)$ is given by
\begin{align*}
%&\begin{array}{ccc}g^*(-\B{\mu}_1,-\B{\mu}_2,-\B{\mu}_3,-\nu)=&\sup&-\V{a}_1^{\text{T}}\B{\mu}_1-\V{a}_2^{\text{T}}\B{\mu}_2-\B{\tau}_2^{\text{T}}\B{\mu}_3-\nu\\
&\begin{array}{cc}\sup&-\V{a}_1^{\text{T}}\B{\mu}_1-\V{a}_2^{\text{T}}\B{\mu}_2-\B{\tau}_2^{\text{T}}\B{\mu}_3-\nu\\
\mbox{s.t.}&\V{a}_1\preceq\B{\tau}_1+\B{\lambda}_1,\  \V{a}_2\preceq-\B{\tau}_1+\B{\lambda}_1\end{array}\\
&\hspace{1cm}=\left\{\begin{array}{ccc}-(\B{\tau}_1+\B{\lambda}_1)^{\text{T}}\B{\mu}_1+(\B{\tau}_1-\B{\lambda}_1)^{\text{T}}\B{\mu}_2-\B{\tau}_2^{\text{T}}\B{\mu}_3-\nu&\mbox{if}&-\B{\mu}_1\succeq\V{0},-\B{\mu}_2\succeq\V{0}\\\infty&\mbox{otherwise.}&\end{array}\right.
\end{align*}
Hence, the dual problem \eqref{dual1} becomes
\begin{align}&\begin{array}{cl}\underset{\B{\mu}_1,\B{\mu}_2,\B{\mu}_3,\nu}{\inf}& 1-(\B{\tau}_1+\B{\lambda}_1)^{\text{T}}\B{\mu}_1-(-\B{\tau}_1+\B{\lambda}_1)^{\text{T}}\B{\mu}_2-\B{\tau}_2^\text{T}\B{\mu}_3-\nu\\
\mbox{s.t.}&\Phi_1(x,y)^{\text{T}}(\B{\mu}_1-\B{\mu}_2)+\Phi_2(x,y)^{\text{T}}\B{\mu}_3+\nu\leq \up{h}(y|x),\  \forall (x,y)\in\set{X}\times\set{Y}\\
&-\B{\mu}_1\succeq\V{0},-\B{\mu}_2\succeq\V{0}\end{array}\nonumber\\\label{step0}
&\begin{array}{cl}=\underset{\B{\mu}_1,\B{\mu}_2,\B{\mu}_3,\nu}{\inf}& 1+(\B{\tau}_1+\B{\lambda}_1)^{\text{T}}\B{\mu}_1+(-\B{\tau}_1+\B{\lambda}_1)^{\text{T}}\B{\mu}_2-\B{\tau}_2\B{\mu}_3-\nu\\
\mbox{\hspace{0.5cm}s.t.}&\Phi_1(x,y)^{\text{T}}(\B{\mu}_2-\B{\mu}_1)+\Phi_2(x,y)^{\text{T}}\B{\mu}_3+\nu\leq \up{h}(y|x),\  \forall (x,y)\in\set{X}\times\set{Y}\\
&\B{\mu}_1\succeq\V{0},\B{\mu}_2\succeq\V{0}\end{array}\\\label{step1}
&\begin{array}{cl}=\underset{\B{\mu},\nu}{\inf}& 1-\B{\tau}^{\text{T}}\B{\mu}+\B{\lambda}^{\text{T}}|\B{\mu}|-\nu\\
\mbox{\hspace{0.4cm}s.t.}&\Phi(x,y)^{\text{T}}\B{\mu}+\nu\leq \up{h}(y|x),\  \forall (x,y)\in\set{X}\times\set{Y}\end{array}\\\label{step2}
&\begin{array}{cl}=\underset{\B{\mu}}{\inf}& 1-\B{\tau}^{\text{T}}\B{\mu}+\underset{(x,y)\in\set{X}\times\set{Y}}{\sup}\big\{\Phi(x,y)^{\text{T}}\B{\mu}-\up{h}(y|x)\big\}+\B{\lambda}^{\text{T}}|\B{\mu}|.\end{array}
\end{align}
The expression in \eqref{step1} is obtained taking  $\B{\mu}=[(\B{\mu}_2-\B{\mu}_1)^{\text{T}},\B{\mu}_3^{\text{T}}]^{\text{T}}$. Firstly,
in \eqref{step0} we can consider only pairs $\B{\mu}_1,\B{\mu}_2$ such that $\B{\mu}_1+\B{\mu}_2=|\B{\mu}_1-\B{\mu}_2|$ because for any pair $\B{\mu}_1,\B{\mu}_2$ feasible in \eqref{step0}, we have that $\tilde{\B{\mu}}_1=(\B{\mu}_2-\B{\mu}_1)_+$, $\tilde{\B{\mu}}_2=(\B{\mu}_1-\B{\mu}_2)_+$ is a feasible pair because $\tilde{\B{\mu}}_1-\tilde{\B{\mu}}_2=\B{\mu}_1-\B{\mu}_2$, and we also have that $\B{\lambda}_1^{\text{T}}|\tilde{\B{\mu}}_1-\tilde{\B{\mu}}_2|=\B{\lambda}_1^{\text{T}}(\tilde{\B{\mu}}_1+\tilde{\B{\mu}}_2)\leq \B{\lambda}_1^{\text{T}}(\B{\mu}_1+\B{\mu}_2)$. Then, we obtain \eqref{step1} from \eqref{step0} because 
\begin{align*}
\B{\tau}_1^{\text{T}}(\B{\mu}_2-\B{\mu}_1)+\B{\tau}_2^{\text{T}}\B{\mu}_3&=\B{\tau}^{\text{T}}\B{\mu}\\
\B{\lambda}_1^{\text{T}}|\B{\mu}_2-\B{\mu}_1|&=\B{\lambda}^{\text{T}}|\B{\mu}|,\mbox{ since }\B{\lambda}_2=\V{0}\\
\Phi_1(x,y)^{\text{T}}(\B{\mu}_2-\B{\mu}_1)+\Phi_2(x,y)^{\text{T}}\B{\mu}_3&=\Phi(x,y)^{\text{T}}\B{\mu}.\end{align*}

%\begin{align*}(\B{\tau}^{\text{T}}+\B{\lambda})^{\text{T}}\B{\mu}_1+(-\B{\tau}^{\text{T}}+\B{\lambda})^{\text{T}}\B{\mu}_2=-\B{\tau}^{\text{T}}(\B{\mu}_2-\B{\mu}_1)+\B{\lambda}^{\text{T}}(\B{\mu}_1+\B{\mu}_2)-\nu\end{align*}
%and 
%\begin{align*}\B{\mu}_2-\B{\mu}_1&=(\B{\mu}_2-\B{\mu}_1)_+-(\B{\mu}_1-\B{\mu}_2)_+=\B{\mu}\\
%\B{\mu}_1+\B{\mu}_2&\succeq(\B{\mu}_2-\B{\mu}_1)_++(\B{\mu}_1-\B{\mu}_2)_+=|\B{\mu}|.\end{align*}
The expression in \eqref{step2} is obtained since for any feasible ($\B{\mu},\nu$) in \eqref{step1} we have that
\mbox{($\B{\mu}$, $\tilde{\nu}$)} is feasible if
$$\tilde{\nu}=\underset{(x,y)\in\set{X}\times\set{Y}}{\inf}\big\{\up{h}(y|x)-\Phi(x,y)^{\text{T}}\B{\mu}\big\}$$
and $\tilde{\nu}\geq\nu$. Then, the inequality in \eqref{wduality} follows by weak duality.

For the second step of the proof, if $\set{U}$ satisfies R1, we have that strong duality holds because the dual becomes the Lagrange dual and the constraints in \eqref{primal} are linear affine (see e.g., Chapter 5 in \cite{BoyVan:04}). If $\set{U}$ satisfies R2, we show in the following that strong duality holds because $\V{0}\in\text{int}(\text{dom}\, g-A\text{dom}\, f)$ (see e.g., Chapter 4 in \cite{BorZhu:04}),  where $\text{dom}$ denotes the set where an extended-valued function takes finite values, and $\text{int}$ denotes the interior of a set. 

If $\set{U}$ satisfies R2.1, there exists $R>0$ such that $R<\lambda^{(i)}-|\mathbb{E}_{\up{p}}\Phi^{(i)}(x,y)-\tau^{(i)}|$ for $i=1,2,\ldots,m$. Then, the second step of the proof is obtained by showing that if $0<\varepsilon<R/(R+1+\|\mathbb{E}_{\up{p}}\{\Phi\}\|_2)<1$, we have that the ball with radius $\varepsilon$ centered in $\V{0}\in\mathbb{R}^{r+m+1}$, $B(\V{0},\varepsilon)$ satisfies $B(\V{0},\varepsilon)\subset(\text{dom}\, g-A\text{dom}\, f)\subset\mathbb{R}^{r+m+1}$. 

We have that for any $\V{z}\in B(\V{0},\varepsilon)\subset\mathbb{R}^{2m+1}$, there exist $\B{\xi}_1,\B{\xi}_2\in B(\V{0},R)\subset\mathbb{R}^m$ such that 
\begin{align*}\big(z^{(1)},z^{(2)},\ldots,z^{(m)}\big)&=\mathbb{E}_{\up{p}}\{\Phi\}+\B{\xi}_1-(1-z^{(2m+1)})\mathbb{E}_{\up{p}}\{\Phi\}\\
\big(z^{(m+1)},z^{(m+2)},\ldots,z^{(2m)}\big)&=-\mathbb{E}_{\up{p}}\{\Phi\}+\B{\xi}_2+(1-z^{(2m+1)})\mathbb{E}_{\up{p}}\{\Phi\}\\
%z^{(2r+1)},z^{(2)},\ldots,z^{(r+m)}&=\mathbb{E}_{\up{p}}\{\Phi\}-(1-z^{(r+m+1)})(\mathbb{E}_{\up{p}}\{\Phi\}+\B{\xi}_3)
z^{(2m+1)}&=1-(1-z^{(2m+1)})\end{align*}
because we have that
$$\|\big(z^{(1)},z^{(2)},\ldots,z^{(m)}\big)-z^{(2m+1)}\mathbb{E}_{\up{p}}\{\Phi\}\|_2\leq\varepsilon(1+\|\mathbb{E}_{\up{p}}\{\Phi\}\|_2)<R$$
$$\|\big(z^{(m+1)},z^{(m+2)},\ldots,z^{(2m)}\big)+z^{(2m+1)}\mathbb{E}_{\up{p}}\{\Phi\}\|_2\leq\varepsilon(1+\|\mathbb{E}_{\up{p}}\{\Phi\}\|_2)<R.$$

Then, the result is obtained observing that $$\big(\mathbb{E}_{\up{p}}\{\Phi\}+\B{\xi}_1,-\mathbb{E}_{\up{p}}\{\Phi\}+\B{\xi}_2,1\big)\in\text{dom}\, g$$ because $R<\lambda^{(i)}-|\mathbb{E}_{\up{p}}\Phi^{(i)}(x,y)-\tau^{(i)}|$ for $i=1,2,\ldots,m$, and $$\big((1-z^{(2m+1)})\mathbb{E}_{\up{p}}\{\Phi\},-(1-z^{(2m+1)})\mathbb{E}_{\up{p}}\{\Phi\},(1-z^{(2m+1)})\big)\in A\text{dom}\, f$$ because $|z^{(2m+1)}|\leq \varepsilon< 1$ and hence $(1-z^{(2m+1)})\up{p}$ is a nonnegative measure.

If $\set{U}$ satisfies R2.2,  we have that $\mathbb{E}_{\up{p}}\{\Phi\}\in\text{int}(\text{Conv}\, S)$, where $S$ denotes the support of $\Phi(x,y)$ if $(x,y)\sim\up{p}$, and $\text{Conv}$ denotes the convex hull of a set. Firstly, $\text{int}(\text{Conv}\, S)\neq\emptyset$ because $S$ and hence $(\text{Conv}\, S)\supseteq S$ are not contained in a proper affine subspace. In the case that $\mathbb{E}_{\up{p}}\{\Phi\}\in \overline{\text{Conv}\, S}\setminus\text{Conv}\,S$, using the Hanh-Banach separation theorem (see e.g., Theorem 2, Sec 5.12 in \cite{Lue:97}), there would exists a hyperplane \mbox{$\Gamma=\{\V{z}; \B{\alpha}^{\text{T}}\V{z}=c\}\subset\mathbb{R}^m$} such that $\mathbb{E}_{\up{p}}\{\Phi\}\in\Gamma$ and $\B{\alpha}^{\text{T}}\V{z}\geq c$ for all $\V{z}\in\text{Conv}\,S$. Then, the real-valued random variable $\B{\alpha}^{\text{T}}\Phi(x,y)-c$ with $(x,y)\sim \up{p}$ is nonnegative with probability one and has expectation zero. Therefore, $\B{\alpha}^{\text{T}}\Phi(x,y)=c$ with probability one, and we would have that the $S$ is contained in the hyperplane $\Gamma$, which leads to a contradiction.

Let $R>0$ be such that 
\begin{enumerate}
\item $R<\lambda^{(i)}-|\mathbb{E}_{\up{p}}\Phi^{(i)}(x,y)-\tau^{(i)}|$ for $i=1,2,\ldots,r$, and 
\item for any $\B{\xi}$ in the Euclidean ball of $\mathbb{R}^m$ centered at $\V{0}$ with radius $R$, $B(\V{0},R)$, we have that
$\mathbb{E}_{\up{p}}\{\Phi\}+\B{\xi}\in\text{Conv}\, S$, that is, there exists $\up{p}_\xi\in\Delta(\set{X}\times\set{Y})$ such that $\mathbb{E}_{\up{p}_\xi}\{\Phi\}=\mathbb{E}_{\up{p}}\{\Phi\}+\B{\xi}$.
\end{enumerate}

As in the previous case, the second step of the proof is obtained by showing that if $0<\varepsilon<R/(R+1+\|\mathbb{E}_{\up{p}}\{\Phi\}\|_2)<1$, we have that $B(\V{0},\varepsilon)\subset(\text{dom}\, g-A\text{dom}\, f)\subset\mathbb{R}^{r+m+1}$. 

If $r=0$, we have that for any $\V{z}\in B(\V{0},\varepsilon)\subset\mathbb{R}^{m+1}$, there exist $\B{\xi}_3\in B(\V{0},R)$ such that 
\begin{align*}\big(z^{(1)},z^{(2)},\ldots,z^{(m)}\big)&=\mathbb{E}_{\up{p}}\{\Phi\}-(1-z^{(m+1)})(\mathbb{E}_{\up{p}}\{\Phi\}+\B{\xi}_3)\\
z^{(m+1)}&=1-(1-z^{(m+1)})\end{align*}
because we have that
$$\frac{\|-\big(z^{(1)},z^{(2)},\ldots,z^{(m)}\big)+z^{(m+1)}\mathbb{E}_{\up{p}}\{\Phi\}\|_2}{1-z^{m+1}}\leq\frac{\varepsilon}{1-\varepsilon}(1+\|\mathbb{E}_{\up{p}}\{\Phi\}\|_2)<R.$$
Then, the result is obtained observing that $$\big(\mathbb{E}_{\up{p}}\{\Phi\},1\big)\in\text{dom}\, g$$ and $$\big((1-z^{(m+1)})(\mathbb{E}_{\up{p}}\{\Phi\}+\B{\xi}_3),(1-z^{(m+1)})\big)\in A\text{dom}\, f$$ because $|z^{(m+1)}|\leq \varepsilon< 1$ and hence $(1-z^{(m+1)})\up{p}_{\B{\xi}_3}$ is a nonnegative measure, for $\up{p}_{\B{\xi}_3}$ satisfying $\mathbb{E}_{\up{p}_{\B{\xi}_3}}\{\Phi\}= \mathbb{E}_{\up{p}}\{\Phi\}+\B{\xi}_3$.

If $0<r<m$, we have that for any $\V{z}\in B(\V{0},\varepsilon)\subset\mathbb{R}^{r+m+1}$, there exist \mbox{$\B{\xi}_1,\B{\xi}_2\in B(\V{0},R)\subset\mathbb{R}^{r}$}, and $\B{\xi}_3\in B(\V{0},R)\subset{R}^{m-r}$ such that 
\begin{align*}\big(z^{(1)},z^{(2)},\ldots,z^{(r)}\big)&=\mathbb{E}_{\up{p}}\{\Phi_1\}+\B{\xi}_1-(1-z^{(r+m+1)})\mathbb{E}_{\up{p}}\{\Phi_1\}\\
\big(z^{(r+1)},z^{(2)},\ldots,z^{(2r)}\big)&=-\mathbb{E}_{\up{p}}\{\Phi_1\}+\B{\xi}_2+(1-z^{(r+m+1)})\mathbb{E}_{\up{p}}\{\Phi_1\}\\
\big(z^{(2r+1)},z^{(2r+2)},\ldots,z^{(r+m)}\big)&=\mathbb{E}_{\up{p}}\{\Phi_2\}-(1-z^{(r+m+1)})(\mathbb{E}_{\up{p}}\{\Phi_2\}+\B{\xi}_{3})\\
z^{(r+m+1)}&=1-(1-z^{(r+m+1)})\end{align*}
%with $\B{\xi}_{3,1}$ the first $r$ components of $\B{\xi}_{3}$ and $\B{\xi}_{3,2}$ the last $m-r$ components of $\B{\xi}_{3}$.
because we have that
$$\|\big(z^{(1)},z^{(2)},\ldots,z^{(m)}\big)-z^{(r+m+1)}\mathbb{E}_{\up{p}}\{\Phi_1\}\|_2\leq\varepsilon(1+\|\mathbb{E}_{\up{p}}\{\Phi\}\|_2)<R$$
$$\|\big(z^{(r+1)},z^{(r+2)},\ldots,z^{(2r)}\big)+z^{(r+m+1)}\mathbb{E}_{\up{p}}\{\Phi_1\}\|_2\leq\varepsilon(1+\|\mathbb{E}_{\up{p}}\{\Phi\}\|_2)<R$$
%$$\|\big(z^{(2r+1)},z^{(2r+2)},\ldots,z^{(r+m)}\big)+z^{(r+m+1)}\mathbb{E}_{\up{p}}\{\Phi_1\}\|_2\leq\varepsilon(1+\|\mathbb{E}_{\up{p}}\{\Phi\}\|_2)<R$$
$$\frac{\|-\big(z^{(2r+1)},z^{(2r+2)},\ldots,z^{(r+m)}\big)+z^{(r+m+1)}\mathbb{E}_{\up{p}}\{\Phi_2\}\|_2}{1-z^{r+m+1}}\leq\frac{\varepsilon}{1-\varepsilon}(1+\|\mathbb{E}_{\up{p}}\{\Phi\}\|_2)<R.$$

Then, the result is obtained observing that $$\big(\mathbb{E}_{\up{p}}\{\Phi_1\}+\B{\xi}_1,-\mathbb{E}_{\up{p}}\{\Phi_1\}+\B{\xi}_2,\mathbb{E}_{\up{p}}\{\Phi_2\},1\big)\in\text{dom}\, g$$ because $R<\lambda^{(i)}-|\mathbb{E}_{\up{p}}\Phi^{(i)}(x,y)-\tau^{(i)}|$ for $i=1,2,\ldots,r$, and 
$$(1-z^{(r+m+1)})\big(\mathbb{E}_{\up{p}}\{\Phi_1\},-\mathbb{E}_{\up{p}}\{\Phi_1\},\mathbb{E}_{\up{p}}\{\Phi_2\}+\B{\xi}_{3},1\big)\in A\text{dom}\, f$$ because $|z^{(r+m+1)}|\leq \varepsilon< 1$ and hence $(1-z^{(r+m+1)})\up{p}_{\B{\xi}_3}$ is a nonnegative measure, for $\up{p}_{\B{\xi}_3}$ satisfying $\mathbb{E}_{\up{p}_{\B{\xi}_3}}\{\Phi_1\}= \mathbb{E}_{\up{p}}\{\Phi_1\}$ and $\mathbb{E}_{\up{p}_{\B{\xi}_3}}\{\Phi_2\}= \mathbb{E}_{\up{p}}\{\Phi_2\}+\B{\xi}_3$ that exists because $\|(\V{0}^{\text{T}},\B{\xi}_3^{\text{T}})^{\text{T}}\|_2\leq R$. 

Finally, since strong duality holds and $\set{U}$ is not empty we have that the optimal value in \eqref{duality} is finite and hence the optimal in the dual is attained \citep{BorZhu:04} and the `$\inf$' in \eqref{wduality} becomes `$\min$'.

\end{proof}

%\appendix
\section{Proof of Theorem~\ref{th1}}\label{apd:proof_th_1}
\begin{proof}
Using Lemma~\ref{lemma} we have that
\begin{align*}\inf_{\up{h}\in \text{T}(\set{X},\set{Y})}\sup_{\up{p}\in\set{U}}\ell(\up{h},\up{p})&=\inf_{\up{h},\B{\mu}}\  1-\B{\tau}^{\text{T}}\B{\mu}+\B{\lambda}^{\text{T}}|\B{\mu}|+\sup_{x\in\set{X},y\in\set{Y}}\{\Phi(x,y)^{\text{T}}\B{\mu}-\up{h}(y|x)\}\\
&=\min_{\B{\mu}}\  1-\B{\tau}^{\text{T}}\B{\mu}+\B{\lambda}^{\text{T}}|\B{\mu}|+\inf_{\up{h}}\sup_{x\in\set{X},y\in\set{Y}}\{\Phi(x,y)^{\text{T}}\B{\mu}-\up{h}(y|x)\}
\end{align*}
and 
\begin{align*}
\begin{array}{ccl}\underset{\up{h}}{\inf}\underset{x\in\set{X},y\in\set{Y}}{\sup}\{\Phi(x,y)^{\text{T}}\B{\mu}-\up{h}(y|x)\}=&\underset{\up{h},\nu}{\inf}&\nu\\&\mbox{s.t.}& \Phi(x,y)^{\text{T}}\B{\mu}-\up{h}(y|x)\leq\nu,\  \forall x\in\set{X},y\in\set{Y}.\end{array}
\end{align*}
Then, the result is obtained because \begin{align*}\Phi(x,y)^{\text{T}}\B{\mu}-\up{h}(y|x)\leq\nu, \forall x\in\set{X}, y\in\set{Y}&\Rightarrow \up{h}(y|x)\geq \Phi(x,y)^{\text{T}}\B{\mu}-\nu, \forall x\in\set{X},y\in\set{Y}\\&\Rightarrow \sum_{y\in\set{C}}(\Phi(x,y)^{\text{T}}\B{\mu}-\nu)\leq 1,\  \forall x\in\set{X},\set{C}\subseteq\set{Y}\\
&\Rightarrow \nu\geq \frac{\sum_{y\in\set{C}}\Phi(x,y)^{\text{T}}\B{\mu}-1}{|\set{C}|},\  \forall x\in\set{X},\set{C}\subseteq\set{Y}\\
&\Rightarrow \nu\geq\varphi(\B{\mu})
\end{align*}
with $\varphi(\B{\mu})$ given by \eqref{varphi}. For each $\B{\mu}$, there exist 
classification rules $\up{h}$ satisfying $$\up{h}(y|x)\geq \Phi(x,y)^{\text{T}}\B{\mu}-\varphi(\mu),\  \forall x\in\set{X}, y\in\set{Y}$$
by definition of $\varphi(\mu)$. Hence, such classification rules are solution of 
$$\inf_{\up{h}}\sup_{x\in\set{X},y\in\set{Y}}\{\Phi(x,y)^{\text{T}}\B{\mu}-\up{h}(y|x)\}$$ with optimal value $\varphi(\B{\mu})$.
\end{proof}

\section{Proof of Theorem~\ref{th_empirical}}\label{apd:proof_th_1_1}
\begin{proof}
The result can be proven analogously of that in Theorem~\ref{th1} shown in Appendix~\ref{apd:proof_th_1}. Firstly, for each $\up{h}\in\text{T}(\set{X},\set{Y})$, we have that 

\begin{align}\label{opt11}\begin{array}{ccl}\sup_{\up{p}\in\set{V}} \ell(\up{h},\up{p})=1-&\underset{\V{p}}{\min} &\V{h}^{\text{T}}\V{p}+I_+(\V{p})\\
&\mbox{s.t.}&\sum_{y\in\set{Y}}\up{p}(x_i,y)=\frac{1}{n},\  i=1,2,\ldots,n\\
&&\B{\tau}-\B{\lambda}\preceq\B{\Phi}^{\text{T}}\V{p}\preceq\B{\tau}+\B{\lambda}
\end{array}\end{align}
 where $\V{p}$, $\V{h}$, and $\B{\Phi}$ denote the vectors and matrix with rows $\up{p}(x_i,y)$, $\up{h}(y|x_i)$ and $\Phi(x_i,y)^{\text{T}}$, respectively, for $y\in\set{Y}$, $i=1,2,\ldots,n$, and
$$I_+(\V{p})=\left\{\begin{array}{cc}0&\mbox{if }\V{p}\succeq\V{0}\\\infty&\mbox{otherwise.}\end{array}\right.$$
Optimization problem \eqref{opt11} has Fenchel (Lagrange) dual
$$\begin{array}{ccl}1-&\underset{\B{\mu}_1,\B{\mu}_2,\B{\nu}}{\max} &\big(\B{\tau}-\B{\lambda}\big)^{\text{T}}\B{\mu}_1
-\big(\B{\tau}+\B{\lambda}\big)^{\text{T}}\B{\mu}_2-\frac{1}{n}\sum_{i=1}^n\nu^{(i)} -f^*(\B{\Phi}(\B{\mu}_1-\B{\mu}_2)-\widetilde{\B{\nu}})\\&
\mbox{s.t.}&\B{\mu}_1,\B{\mu}_2\succeq\V{0}\end{array}$$
where $\widetilde{\B{\nu}}$ is the vector in $\mathbb{R}^{n|\set{Y}|}$ with component corresponding with $(x_i,y)$ for $i=1,2,\ldots,n$, $y\in\set{Y}$ given by $\nu^{(i)}$, and $f^*$ is the conjugate function of $f(\V{p})=\V{h}^{\text{T}}\V{p}+I_+(\V{p})$ given by
$$f^*(\V{w})=\sup_{\V{p}\succeq\V{0}}\V{w}^{\text{T}}\V{p}-\V{h}^{\text{T}}\V{p}=\left\{\begin{array}{cc}0&\mbox{if }\V{w}\preceq\V{h}\\\infty&\mbox{otherwise.}\end{array}\right.$$
%using the Lemma~\ref{lemma} becomes
%$$f^*(\V{w})=\left\{\begin{array}{cc}0&\mbox{if }\V{w}\preceq\V{q}\\\infty&\mbox{otherwise}\end{array}\right..$$
Therefore, the Lagrange dual above becomes
$$\begin{array}{ccl}1-&\underset{\B{\mu}_1,\B{\mu}_2,\B{\nu}}{\max} &\big(\B{\tau}-\B{\lambda}\big)^{\text{T}}\B{\mu}_1
-\big(\B{\tau}+\B{\lambda}\big)^{\text{T}}\B{\mu}_2-\frac{1}{n}\sum_{i=1}^n\nu^{(i)}\\&
\mbox{s.t.}&\B{\mu}_1,\B{\mu}_2\succeq\V{0}\\&
&\Phi(x_i,y)^{\text{T}}(\B{\mu}_1-\B{\mu}_2)-\nu^{(i)}\leq\up{h}(y|x_i),\  \forall y\in\set{Y},i=1,2,\ldots,n.\end{array}$$
It is easy to see that the solution of such optimization problem $\bar{\B{\mu}}_1,\bar{\B{\mu}}_2$ satisfies that $\mbox{$\bar{\mu}_{1}^{(i)}\bar{\mu}_{2}^{(i)}=0$}$ for any $i$ such that $\lambda_i>0$. Then $\B{\lambda}^{\text{T}}(\bar{\B{\mu}}_1+\bar{\B{\mu}}_2)=\B{\lambda}^{\text{T}}|\bar{\B{\mu}}_1-\bar{\B{\mu}}_2|$ and taking $\B{\mu}=\B{\mu}_1-\B{\mu}_2$ the Lagrange dual above is equivalent to
$$\begin{array}{ccl}1-&\underset{\B{\mu},\B{\nu}}{\max} &\B{\tau}^{\text{T}}\B{\mu}
-\B{\lambda}^{\text{T}}|\B{\mu}|-\frac{1}{n}\sum_{i=1}^n\nu^{(i)}\\&
&\Phi(x_i,y)^{\text{T}}\B{\mu}-\nu^{(i)}\leq\up{h}(y|x_i),\  \forall y\in\set{Y},i=1,2,\ldots,n\end{array}$$
that has the same value as $\sup_{\up{p}\in\set{V}} \ell(\up{h},\up{p})$ since the constraints in \eqref{opt11} are affine and $\set{V}$ is non-empty. 

Therefore, 
\begin{align*}\inf_{\up{h}\in\text{T}(\set{X},\set{Y})}\sup_{\up{p}\in\set{V}} \ell(\up{h},\up{p})=\inf_{\up{h},\B{\mu},\B{\nu}}&\  1-\B{\tau}^{\text{T}}\B{\mu}
+\B{\lambda}^{\text{T}}|\B{\mu}|+\frac{1}{n}\sum_{i=1}^n\nu^{(i)}\\
&\Phi(x_i,y)^{\text{T}}\B{\mu}-\nu^{(i)}\leq\up{h}(y|x_i),\  \forall y\in\set{Y},i=1,2,\ldots,n\end{align*}
and, similarly to the proof for Theorem~\ref{th1}, we have that
\begin{align*}\Phi(x_i,y)^{\text{T}}\B{\mu}&-\nu^{(i)}\leq\up{h}(y|x_i),\  \forall y\in\set{Y},i=1,2,\ldots,n\\
&\Rightarrow\sum_{y\in\set{C}}\Phi(x_i,y)^{\text{T}}\B{\mu}-\nu^{(i)}\leq 1,\  \forall \set{C}\subseteq\set{Y}, i=1,2,\ldots,n\\
&\Rightarrow \nu^{(i)}\geq \frac{\sum_{y\in\set{C}}\Phi(x_i,y)^{\text{T}}\B{\mu}-1}{|\set{C}|}, \  \forall \set{C}\subseteq\set{Y}, i=1,2,\ldots,n\\
&\Rightarrow \nu^{(i)}\geq\upvarphi(\B{\mu},x_i), \  \forall i=1,2,\ldots,n.
\end{align*}
Therefore, for each $\B{\mu}$, we have that any classification rule satisfying 
$$\up{h}(y|x)\geq \Phi(x,y)^{\text{T}}\B{\mu}-\upvarphi(\B{\mu},x),\  \forall x\in\set{X},y\in\set{Y}$$
is solution of 
\begin{align*}\inf_{\up{h},\B{\nu}}&\  \frac{1}{n}\sum_{i=1}^n\nu^{(i)}&\\
&\  \Phi(x_i,y)^{\text{T}}\B{\mu}-\nu^{(i)}\leq\up{h}(y|x_i),\  \forall y\in\set{Y},i=1,2,\ldots,n&\end{align*}
that has optimal value $\frac{1}{n}\sum_{i=1}^n\upvarphi(\B{\mu},x_i)$. Then, the result is obtained because for any $x\in\set{X}$, we have that
$$\sum_{y\in\set{Y}}\big(\Phi(x,y)^{\text{T}}\B{\mu}-\upvarphi(\B{\mu},x)\big)_+=1$$
because otherwise there would exist $\nu_x<\upvarphi(\B{\mu},x)$ such that
$$1=\sum_{y\in\set{Y}}\big(\Phi(x,y)^{\text{T}}\B{\mu}-\nu_x\big)_+=\max_{\set{C}\subseteq\set{Y}}\sum_{y\in\set{C}}\Phi(x,y)^{\text{T}}\B{\mu}-\nu_x$$
which contradicts the definition of $\upvarphi(\B{\mu},x)$.
\end{proof}

%which is equivalent to 
%$$\begin{array}{cl}\underset{\B{\mu}_1,\B{\mu}_2}{\max} &\big(\B{\tau}-\B{\lambda}\big)^{\text{T}}\B{\mu}_1
%-\big(\B{\tau}+\B{\lambda}\big)^{\text{T}}\B{\mu}_2+\underset{x\in\set{X},y\in\set{Y}}{\min}\,\up{h}(y|x)-\Phi(x,y)^{\text{T}}(\B{\mu}_1-\B{\mu}_2)\\
%\mbox{s.t.}&\B{\mu}_1,\B{\mu}_2\succeq\V{0}.\end{array}$$
%It is easy to see that the solution of such optimization problem $\bar{\B{\mu}}_1,\bar{\B{\mu}}_2$ satisfies that $\mbox{$\bar{\mu}_{1}^{(i)}\bar{\mu}_{2}^{(i)}=0$}$ for any $i$ such that $\lambda_i>0$. Then $\B{\lambda}^{\text{T}}(\bar{\B{\mu}}_1+\bar{\B{\mu}}_2)=\B{\lambda}^{\text{T}}|\bar{\B{\mu}}_1-\bar{\B{\mu}}_2|$ and taking $\B{\mu}=\B{\mu}_1-\B{\mu}_2$ the Lagrange dual above is equivalent to
%$$\underset{\B{\mu}}{\max}\  \B{\tau}^{\text{T}}\B{\mu}+\underset{x\in\set{X},y\in\set{Y}}{\min}\,q(x,y)-\Phi(x,y)^{\text{T}}\B{\mu}-\B{\lambda}^{\text{T}}|\B{\mu}|$$
%that has the same value as its primal $\min_{\tilde{\up{p}}\in\set{U}}\mathbb{E}_{\tilde{\up{p}}}q$ since the constraints defining $\set{U}$ are affine and $\set{U}\neq\emptyset$.

\section{Proof of Theorem~\ref{th2}}\label{apd:proof_th_2}
\begin{proof} The first result is a direct consequence of Hoeffding's inequality and the union bound since each component of $\Phi$ is bounded by its corresponding scalar feature. Similarly, the second result is a consequence of the empirical Bernstein inequality in \cite{MauPon:09}.

For the last result, we have that for each $j\in\set{Y}$
\begin{align*}&\left |\frac{1}{n}\sum_{i=1}^n\psi(x_{i})\mathbb{I}\{y_i=j\}-\mathbb{E}\big\{\psi(x)\mathbb{I}\{y=j\}\big\}\right |\\
&=\left |\frac{n_j}{n}\frac{1}{n_j}\sum_{i=1}^n\psi(x_{i})\mathbb{I}\{y_i=j\}-\up{p}^*(y=j)\mathbb{E}\{\psi(x|y=j)\}\right |\\
&=\left |\frac{n_j}{n}\left(\frac{1}{n_j}\sum_{i=1}^n\psi(x_{i})\mathbb{I}\{y_i=j\}-\mathbb{E}\{\psi(x|y=j)\}\right)+\mathbb{E}\{\psi(x|y=j)\}\left(\frac{n_j}{n}-\up{p}^*(y=j)\right)\right |\\
&\leq\frac{n_j}{n}\left |\frac{1}{n_j}\sum_{i=1}^n\psi(x_{i})\mathbb{I}\{y_i=j\}-\mathbb{E}\{\psi(x|y=j)\}\right|+C\left|\frac{n_j}{n}-\up{p}^*(y=j)\right|\end{align*}
For the first term, we have that with probability at least $1-\delta$
$$\left |\frac{1}{n_j}\sum_{i=1}^n\psi(x_{i})\mathbb{I}\{y_i=j\}-\mathbb{E}\{\psi(x|y=j)\}\right|\leq 2\set{R}_{n_j}(\set{F})+\frac{2C\sqrt{\log 2/\delta}}{\sqrt{2n_j}}$$
using the uniform concentration bound in terms of Rademacher complexities (see e.g., \cite{MehRos:18}).
For the second term, we have that with probability at least $1-\delta$
$$\left|\frac{n_j}{n}-\up{p}^*_y(y=j)\right|\leq \frac{\sqrt{\log 2/\delta}}{\sqrt{2n}}$$
using Hoeffding's inequality. Therefore, the result is obtained using the union bound.
\end{proof}

\section{Proof of Theorem~\ref{th-bounds}}\label{apd:proof_th_3}
\begin{proof}
The result for $\overline{R}(\set{U},\up{h})$ is a direct consequence of Lemma~\ref{lemma}, and the second result for $\underline{R}(\set{U},\up{h})$ is obtained analogously since
\begin{align*}\inf_{\up{p}\in\set{U}}\ell(\up{h},\up{p})&=-\sup_{\up{p}\in\set{U}}-1-\int\big(-\up{h}(y|x)\big)\text{d}\up{p}(x,y)\\
&\geq-\inf_{\B{\mu}}-1-\B{\tau}^{\text{T}}\B{\mu}+\B{\lambda}^{\text{T}}|\B{\mu}|+\sup_{x\in\set{X},y\in\set{Y}}\{\Phi(x,y)^{\text{T}}\B{\mu}+\up{h}(y|x)\}\end{align*}
that leads to the expression in \eqref{opt-lower} changing the notation for the variable in the optimization from $\B{\mu}$ to $-\B{\mu}$. 
\end{proof}

\section{Proof of Theorem~\ref{th_6}}\label{apd:proof_th_6}
\begin{proof}
Let $\set{U}_\infty$ be the uncertainty set in \eqref{unc-inf}. It is clear that $\up{p}^*\in\set{U}_\infty$,  then using Theorem~\ref{th-bounds} %with probability at least $1-\delta$
\begin{align}R(\up{h}^{\sset{U}})&\leq \overline{R}(\set{U}_\infty,\up{h}^{\sset{U}})=\underset{\mu}{\min}\  1-\B{\tau}_\infty^{\text{T}}\B{\mu}+\sup_{x\in\set{X}, y\in\set{Y}}\{\Phi(x,y)^{\text{T}}\B{\mu}-\up{h}^{\sset{U}}(y|x)\}\nonumber\\
&\leq 1-\B{\tau}_\infty^{\text{T}}\B{\mu}^*+\sup_{x\in\set{X}, y\in\set{Y}}\{\Phi(x,y)^{\text{T}}\B{\mu}^*-\up{h}^{\sset{U}}(y|x)\}\nonumber\\\label{step}
&\leq 1-\B{\tau}_\infty^{\text{T}}\B{\mu}^*+\underset{x\in\set{X},\set{C}\subseteq\set{Y}}{\sup}\frac{\sum_{y\in\set{C}}\Phi(x,y)^{\text{T}}\B{\mu}^*-1}{|\set{C}|}\\
&=\overline{R}(\set{U})+(\B{\tau}-\B{\tau}_\infty)^{\text{T}}\B{\mu}^*-\B{\lambda}^{\text{T}}|\B{\mu}^*|\nonumber
\end{align}
where \eqref{step} is due to the definition of $\up{h}^{\sset{U}}$.

Analogously,
\begin{align}R(\up{h}^{\sset{U}})&\geq \underline{R}(\set{U}_\infty,\up{h}^{\sset{U}})=\underset{\mu}{\max}\  1-\B{\tau}_\infty^{\text{T}}\B{\mu}+\inf_{x\in\set{X}, y\in\set{Y}}\{\Phi(x,y)^{\text{T}}\B{\mu}-\up{h}^{\sset{U}}(y|x)\}\nonumber\\
&\geq 1-\B{\tau}_\infty^{\text{T}}\underline{\B{\mu}}+\inf_{x\in\set{X}, y\in\set{Y}}\{\Phi(x,y)^{\text{T}}\underline{\B{\mu}}-\up{h}^{\sset{U}}(y|x)\}\nonumber\\
&=\underline{R}(\set{U})+(\B{\tau}-\B{\tau}_\infty)^{\text{T}}\underline{\B{\mu}}+\B{\lambda}^{\text{T}}|\underline{\B{\mu}}|\nonumber.
\end{align}

For inequality \eqref{bound4}, note that from \eqref{step} and using the definition of $\B{\mu}^*$ we have that %if $\up{p}^*\in\set{U}_\delta$
\begin{align}
R(\up{h}^{\sset{U}})&\leq  1-\B{\tau}^{\text{T}}\B{\mu}_\infty+\underset{x\in\set{X},\set{C}\subseteq\set{Y}}{\sup}\frac{\sum_{y\in\set{C}}\Phi(x,y)^{\text{T}}\B{\mu}_\infty-1}{|\set{C}|}+\B{\lambda}^{\text{T}}|\B{\mu}_\infty|-\B{\lambda}^{\text{T}}|\B{\mu}^*|\nonumber+(\B{\tau}-\B{\tau}_\infty)\B{\mu}^*\\
&=R_\Phi+(\B{\tau}_\infty-\B{\tau})^{\text{T}}(\B{\mu}_\infty-\B{\mu}^*)+\B{\lambda}^{\text{T}}(|\B{\mu}_\infty|-|\B{\mu}^*|)\nonumber%\\
%&\leq R_\Phi+(|\B{\tau}_\infty-\B{\tau}_n|+\B{\lambda})^{\text{T}}|\B{\mu}_\infty-\B{\mu}^*|
\end{align}

Finally, the results in \eqref{bound1} and \eqref{bound2} are obtained analogusly as the previous result using $\set{U}$ instead of $\set{U}_\infty$. %as a direct consequence of the previous taking $\delta=1-\mathbb{P}\{|\B{\tau}_\infty-\B{\tau}_n|\preceq\B{\lambda}\}$ and $\B{\lambda}_\delta=\B{\lambda}$.
\end{proof}

\section{Proof of Theorem~\ref{th_7}}\label{apd:proof_th_7}

\begin{proof}
In the first step of the proof we show that if $\set{V}^n$ is the uncertainty set
$$\set{V}^n=\{\up{p}\in\Delta(\set{X}\times\set{Y}):\  \B{\tau}^n_{\infty}-\B{\lambda}_n\preceq\mathbb{E}_{\up{p}}\{\Phi_n\}\preceq\B{\tau}^n_{\infty}+\B{\lambda}_n\}$$
with $\B{\tau}_\infty^n=\mathbb{E}_{\up{p}^*}\{\Phi_n\}$ and $\B{\lambda}_n$ satisfying condition (4) in the theorem's statement, then,
we have that $\{R(\set{V}^n)\}$ tends to $R_{\text{Bayes}}$ with probability one for any underlying distribution $\up{p}^*$. In the second step of the proof we obtain the result using Theorem~\ref{th_6} and the Borel-Cantelli Lemma.

For the first step, we have that $R(\set{V}^n)$ for $n=1,2,\ldots$ is a non-increasing sequence because for $n=1,2,\ldots$ the sequence $D_n$ is non-decreasing and any component of $\B{\lambda}_n$ is non-increasing. In addition, $R(\set{V}^n)$ for $n=1,2,\ldots$ is lower bounded by $R_{\text{Bayes}}$ since $\up{p}^*\in\set{V}^n$ for any $n$. Then, the first step in the proof is obtained  showing that $R_{\text{Bayes}}$ is the largest lower bound and using the monotone convergence theorem. Specifically,  in case there exists $L\in\mathbb{R}$ such that for all $n$
$$R(\set{V}^n)\geq L>R_{\text{Bayes}}=R(\{\up{p}^*\})$$
then, it would exist a distribution $\up{p}'\neq\up{p}^*$ with $\up{p}'\in\set{V}^n$ for all $n$ because $\set{V}^{n_1}\subseteq \set{V}^{n_2}$ if $n_1\geq n_2$. Since $\up{p}',\up{p}^*\in\set{V}^n$ for all $n$, we have that
\begin{align}\label{in}-2\B{\lambda}_n\preceq\mathbb{E}_{\up{p}^*}\{\Phi_n\}-\mathbb{E}_{\up{p}'}\{\Phi_n\}\preceq2\B{\lambda}_n\end{align}
In particular, for all $n$
$$\|\mathbb{E}_{\up{p}^*}\V{e}_y- \mathbb{E}_{\up{p}'}\V{e}_y\|_\infty\leq 2\|\B{\lambda}_n\|_\infty$$
so that $\up{p}^*_y=\up{p}'_y$. Using again \eqref{in} and the definition of $\Phi_n$, denoting $\Psi_n(x)=[\psi_{v_1}(x),\psi_{v_2}(x),\ldots,\psi_{v_{D_n}}(x)]^{\text{T}}$ we get that
$$\frac{1}{D_n}\|\mathbb{E}_{\up{p}^*}\{\V{e}_y\otimes \Psi_n\}-\mathbb{E}_{\up{p}'}\{\V{e}_y\otimes \Psi_n\}\|_2^2\leq 4|\set{Y}|\|\B{\lambda}_n\|_\infty^2\to_{n\to\infty}0$$
because any component of $\B{\lambda}_n$ tends to $0$.

Therefore, for all $y\in\set{Y}$ we have that 
$$\frac{1}{D_n}\Big\|\int_{x\in\set{X}}\Psi_n(x)\text{d}\up{p}^*(x,y)-\int_{x\in\set{X}}\Psi_n(x)\text{d}\up{p}'(x,y)\Big\|_2^2\to_{n\to\infty} 0$$
so that for $j\in\set{Y}$ such that $\up{p}^*_y(j)\neq 0$
\begin{align*}&\frac{1}{D_n}\Big\|\int_{x\in\set{X}}\Psi_n(x)\text{d}\up{p}^*_{x|y=j}-\int_{x\in\set{X}}\Psi_n(x)\text{d}\up{p}'_{x|y=j}\Big\|_2^2\to_{n\to\infty} 0\\
\Rightarrow&\int_{x\in\set{X},x'\in\set{X}}\frac{\Psi_n(x)^{\text{T}}\Psi_n(x')}{D_n}\text{d}(\up{p}^*_{x|y=j}-\up{p}'_{x|y=j})\text{d}(\up{p}^*_{x'|y=j}-\up{p}'_{x'|y=j})\to_{n\to\infty} 0.
\end{align*}
%for all $y\in\set{Y}$, $\up{p}^*(y)=\up{p}'(y)$ and  we have that
%\begin{align*}\sum_{x\in\set{X}}\Psi_n(x)\up{p}^*(x|y)-\Psi_n(x)\up{p}'(x|y)\to_{n\to\infty} \V{0}\Leftrightarrow\\
%\sum_{x,x'\in\set{X}}(\up{p}^*(x|y)-\up{p}'(x|y))\Psi_n(x)^{\text{T}}\Psi_n(x')(\up{p}^*(x'|y)-\up{p}'(x'|y))\to_{n\to\infty} 0
%\end{align*}
Using the law of large numbers we have that with probability one 
%$$\frac{\Psi_n(x)^{\text{T}}\Psi_n(x')}{D_n}\to_{n\to\infty}k(x,x')$$
\begin{align}
\frac{\Psi_n(x)^{\text{T}}\Psi_n(x')}{D_n}&\underset{n\to\infty}{\to}k(x,x')\nonumber\\
\Rightarrow
&\int_{x\in\set{X},x'\in\set{X}}k(x,x')\text{d}(\up{p}^*_{x|y=j}-\up{p}'_{x|y=j})\text{d}(\up{p}^*_{x'|y=j}-\up{p}'_{x'|y=j})= 0\nonumber\\
\Rightarrow&\left\|\int k(x,\cdot)\text{d}\up{p}^*_{x|y=j}-\int k(x,\cdot)\text{d}\up{p}'_{x|y=j}\right\|_\set{H}=0\label{aux}\end{align}
for $\set{H}$ the \ac{RKHS} given by kernel $k$. 
But equality \eqref{aux} contradicts $\up{p}^*\neq\up{p}$ because $k$ is a characteristic kernel and $\up{p}^*_y=\up{p}'_y$  for all $y\in\set{Y}$.

%$$\sum_{x,x'\in\set{X}}(\up{p}^*(x|y)-\up{p}'(x|y))k(x,x')(\up{p}^*(x'|y)-\up{p}'(x'|y))=0$$

As a consequence of the previous result, we also get that $R_{\Phi_n}$ converges with probability one to $R_{\text{Bayes}}$ since the smallest minimax risk satisfies $R_{\Phi_n}=R(\set{U}^n_\infty)$
with  
$$\set{U}_{\infty}^n=\{\up{p}\in\Delta(\set{X}\times\set{Y}):\  \mathbb{E}_{\up{p}}\{\Phi_n(x,y)\}=\B{\tau}_{\infty}^n\}$$ that coincides with $\set{V}^n$ above taking $\B{\lambda}_n=\V{0}$ for all $n$.

For the second step, if $\B{\lambda}_n$ and $D_n$ satisfy the two additional conditions in the theorem's statement, let $N_0$ be an integer such that any component of $\B{\lambda}_n$ is larger than 
$$C\sqrt{\frac{2\log (|\set{Y}|(D_n+1)) +2\log 2n^2}{n}}$$ for any $n\geq N_0$. Such $N_0$ exists since any component of $\B{\lambda}_n\sqrt{n/\log n}$ tends to $\infty$ and $D_n=\set{O}(n^k)$ for some $k>0$. Then, for $n\geq N_0$ we have that $\up{p}^*\in\set{U}_n$ with probability at least $1-1/n^2$ because using Hoeffding's inequality we have that
$$\|\B{\tau}_n^\infty-\B{\tau}_n\|_\infty\leq C\sqrt{\frac{2\log(|\set{Y}|(D_n+1))+2\log(2/(1/n^2))}{n}}$$ 
with probability at least $1-1/n^2$. Therefore, since $\set{V}^n$ satisfies R2.1, using Lemma~\ref{lemma} we have that with probability at least $1-1/n^2$ 
\begin{align*}R(\up{h}_n)&\leq R(\set{U}_n)\leq\inf_{\B{\mu}}1-\B{\tau}_n^{\text{T}}\B{\mu}+\varphi(\B{\mu})+\B{\lambda}_n^{\text{T}}|\B{\mu}|\\
&\leq 1-\B{\tau}_n^{\text{T}}\bar{\B{\mu}}_n+\varphi(\bar{\B{\mu}}_n)+\B{\lambda}_n^{\text{T}}|\bar{\B{\mu}}_n|=R(\set{V}^n)+(\B{\tau}^n_{\infty}-\B{\tau}_n)^{\text{T}}\bar{\B{\mu}}_n
\end{align*}
where $\bar{\B{\mu}}_n$ is the solution of \eqref{opt-prob} for $\B{\tau}=\B{\tau}^n_{\infty}$ and $\B{\lambda}=\B{\lambda}_n$. 

If $\underline{\lambda}_n$ is the smallest component of $\B{\lambda}_n$, we have that $\|\bar{\B{\mu}}_n\|_1\leq 1/\underline{\lambda}_n$ because
$$0\leq 1-(\B{\tau}_\infty^n)^{\text{T}}\bar{\B{\mu}}_n+\varphi(\bar{\B{\mu}}_n)+\B{\lambda}_n^{\text{T}}|\bar{\B{\mu}}_n|=R(\set{V}^n)\leq 1$$
and 
$$1-(\B{\tau}_\infty^n)^{\text{T}}\bar{\B{\mu}}_n+\varphi(\bar{\B{\mu}}_n)\geq R(\set{U}^n_\infty)\geq 0.$$
%where $\set{U}_n^\infty$ is the uncertainty set corresponding to $\B{\tau}=\B{\tau}_n^\infty$ and $\B{\lambda}=\V{0}$. 
Therefore, with probability at least $1-1/n^2$
\begin{align*}R(\up{h}_n)\leq R(\set{V}^n)+\|\B{\tau}^n_{\infty}-\B{\tau}_n\|_{\infty}\|\bar{\B{\mu}}_n\|_1\leq R(\set{V}^n)+C\frac{\sqrt{2\log (|\set{Y}|(D_n+1)) +2\log 2n^2}}{\sqrt{n}\underline{\lambda}_n}.
\end{align*}
Let $\varepsilon>0$ and consider $N\geq N_0$ such that for any $n\geq N$ 
$$R(\set{V}^n)< R_{\text{Bayes}}+\frac{\varepsilon}{2}\  \  \mbox{ and }\  \  C\frac{\sqrt{2\log (|\set{Y}|(D_n+1)) +2\log 2n^2}}{\sqrt{n}\underline{\lambda}_n}< \frac{\varepsilon}{2}.$$ Such $N$ exists because 
$$R(\set{V}^n)\to R_{\text{Bayes}}\  \  \mbox{ and }\  \  \frac{\sqrt{n}\underline{\lambda}_n}{\sqrt{\log n}}\to \infty$$
when $n$ tends to infinity, and $D_n=\set{O}(n^k)$. 

Therefore, 
$$\sum_{n\geq 1}\mathbb{P}\{R(\up{h}_n)-R_{\text{Bayes}}\geq \varepsilon\}\leq N-1+\sum_{n\geq N}\frac{1}{n^2}<\infty$$
so that the result follows using the Borel-Cantelli Lemma.
\end{proof}

\section{Proof of Theorem~\ref{th_8}}\label{apd:proof_th_8}

\begin{proof}
We first show that optimization problems \eqref{opt-lower} and \eqref{opt-upper} using $\set{X}_s$ instead of $\set{X}$ are equivalent to those using subsets of $\set{X}$ that cover most of the probability mass of the underlying distribution. We then prove that the probabilities of error in such subsets are near the probabilities of error in all the set $\set{X}$.

%Let $\up{h}_s$ be the \ac{MRC} obtained by solving \eqref{opt-prob} using $\set{X}_s$ instead of $\set{X}$, 
Let $\B{\mu}_u$ be a solution of the optimization problem
\begin{align}\label{opt-u}\min_{\B{\mu}}1-\B{\tau}^{\text{T}}\B{\mu}+\max_{x\in\set{X}_s,y\in\set{Y}}\left\{\Phi(x,y)^{\text{T}}\B{\mu}-\up{h}_s(y|x)\right\}+\B{\lambda}^{\text{T}}|\B{\mu}|\end{align}
and $\B{\mu}_l$ be a solution of the optimization problem  
\begin{align}\label{opt-l}\max_{\B{\mu}}1-\B{\tau}^{\text{T}}\B{\mu}+\min_{x\in\set{X}_s,y\in\set{Y}}\left\{\Phi(x,y)^{\text{T}}\B{\mu}-\up{h}_s(y|x)\right\}-\B{\lambda}^{\text{T}}|\B{\mu}|\end{align}
Both $\B{\mu}_u$ and $\B{\mu}_l$ exist because $\set{X}_s$ is finite and $\overline{R}_s(\set{U})$ is finite.

If $\set{X}_u$ and $\set{X}_l$ are the sets
$$\set{X}_u=\Big\{x\in\set{X}:\Phi(x,y)^\text{T}\B{\mu}_u-\up{h}_s(y|x)\leq\max_{x\in\set{X}_s}\big(\Phi(x,y)^{\text{T}}\B{\mu}_u-\up{h}_s(y|x)\big), \forall \  y\in\set{Y}\Big\}$$
$$\set{X}_l=\Big\{x\in\set{X}:\Phi(x,y)^\text{T}\B{\mu}_l-\up{h}_s(y|x)\geq\min_{x\in\set{X}_s}\big(\Phi(x,y)^{\text{T}}\B{\mu}_l-\up{h}_s(y|x)\big), \forall \  y\in\set{Y}\Big\}$$
we have that optimization problem \eqref{opt-u} is equivalent to that obtained substituting $\set{X}_s$ by $\set{X}_u$ because the objective function of the former problem is a lower bound for the latter and both coincide in $\B{\mu}_u$ as a direct consequence of the definition of $\set{X}_u$. Similarly, \eqref{opt-l} is equivalent to that obtained substituting $\set{X}_s$ by $\set{X}_l$. We next proof that with probability at least $1-\delta$ over the randomness of $\set{X}_s$, the sets $\set{X}_u$ and $\set{X}_l$ have probability larger than $1-\varepsilon_s$ with respect to the probability distribution $\up{p}^*(x)$. 

For each $j\in\set{Y}$, let $\B{\mu}_{u,j}$ and $\B{\mu}_{l,j}$  be the vectors of size $|\set{Y}|(m+1)$ formed by concatenating the vectors $[\B{\mu}_u^{\text{T}},1]^{\text{T}}\mathbb{I}\{y=j\}$ and $[\B{\mu}_l^{\text{T}},1]^{\text{T}}\mathbb{I}\{y=j\}$ for $y=1,2,\ldots,|\set{Y}|$, respectively. In addition, for each $j\in\set{Y}$, let  $d_{u,j},d_{l,j}\in\mathbb{R}$ be given by
$$d_{u,j}=\max_{x\in\set{X}_s}\big(\Phi(x,j)^{\text{T}}\B{\mu}_u-\up{h}_s(j|x)\big)$$
$$d_{l,j}=\min_{x\in\set{X}_s}\big(\Phi(x,j)^{\text{T}}\B{\mu}_l-\up{h}_s(j|x)\big).$$

Then, denoting $\V{u}(x)=[\Phi(x,1)^{\text{T}},\up{h}_s(1|x),\Phi(x,2)^{\text{T}},\up{h}_s(2|x),\ldots,\Phi(x,|\set{Y}|)^{\text{T}}\up{h}_s(|\set{Y}||x),]^{\text{T}}\in\mathbb{R}^{|\set{Y}|(m+1)}$ we have that 
$$\set{X}_u=\{x\in\set{X}:\  \V{u}(x)^{\text{T}}\B{\mu}_{u,j}\leq d_{u,j},\forall\  j\in\set{Y}\}$$
$$\set{X}_l=\{x\in\set{X}:\  \V{u}(x)^{\text{T}}\B{\mu}_{l,j}\geq d_{l,j},\forall\  j\in\set{Y}\}.$$

Hence, if $\set{A}$ is the set of half spaces in $\mathbb{R}^{|\set{Y}|(m+1)}$, and $S(\set{A},s)$ denotes the $s$-th shatter coefficient of $\set{A}$ (see e.g., Theorem 12.5 in \cite{DevGyoLug:96}), we have that with probability at least $1-\delta$
$$\mathbb{E}\{\mathbb{I}\{\V{u}(x)^{\text{T}}\B{\mu}_{u,j}\leq d_{u,j}\}\}\geq \frac{1}{s}\sum_{i=1}^s\mathbb{I}\{\V{u}(x_i)^{\text{T}}\B{\mu}_{u,j}\leq d_{u,j}\}-\sqrt{\frac{32(\log 8S(\set{A},s)+\log\frac{|\set{Y}|}{\delta})}{s}}$$
$$\mathbb{E}\{\mathbb{I}\{\V{u}(x)^{\text{T}}\B{\mu}_{l,j}\geq d_{l,j}\}\}\geq \frac{1}{s}\sum_{i=1}^s\mathbb{I}\{\V{u}(x_i)^{\text{T}}\B{\mu}_{l,j}\geq d_{l,j}\}-\sqrt{\frac{32(\log 8S(\set{A},s)+\log\frac{|\set{Y}|}{\delta})}{s}}$$
for all $j\in\set{Y}$. Therefore, using the union bound and the fact that $S(\set{A},s)\leq 2(s-1)^{(m+1)|\set{Y}|}+2\leq 4 s^{(m+1)|\set{Y}|}$ (see e.g., Corollary 13.1 in \cite{DevGyoLug:96}) we get that with probability at least $1-\delta$ over the randomness of $\set{X}_s$, the sets $\set{X}_u$ and $\set{X}_l$ have probability larger than $1-\varepsilon_s$.

For the last step of the proof, let $R|_{\set{Z}}(\up{h})$ denote the risk of rule $\up{h}$ restricted to $\set{Z}\subset \set{X}$, that is,
$$R|_{\set{Z}}(\up{h})=\int_{x\in\set{Z},y\in\set{Y}}\ell(\up{h},(x,y))\text{d}\up{p}^*|_{\set{Z}}(x,y)$$
for $\up{p}^*|_{\set{Z}}$ the probability measure corresponding to $\up{p}^*$ restricted to $\set{Z}$, that is
$$\up{p}^*|_{\set{Z}}(x,y)=\left\{\begin{array}{cc}\frac{\up{p}^*(x,y)}{\up{p}^*(\set{Z})}&\mbox{if } x\in\set{Z}\\
0&\mbox{otherwise.} \end{array}\right.$$

Then, with probability at least $1-\delta$, we have that
\begin{align}\label{ineq1}R|_{\set{X}_l}(\up{h}_s)-\varepsilon_s\leq R(\up{h}_s)\leq R|_{\set{X}_u}(\up{h}_s)+\varepsilon_s\end{align}
because for any rule $\up{h}$ and set $\set{Z}\subset \set{X}$ 
\begin{align*}R(\up{h})&=\int_{x\in\set{X},y\in\set{Y}}\ell(\up{h},(x,y))\text{d}\up{p}^*(x,y)\\
&=\up{p}^*(\set{Z})\int_{x\in\set{Z},y\in\set{Y}}\ell(\up{h},(x,y))\text{d}\up{p}^*|_{\set{Z}}(x,y)+\int_{x\in\set{X}\setminus\set{Z},y\in\set{Y}}\ell(\up{h},(x,y))\text{d}\up{p}^*(x,y)\\
&=R|_{\set{Z}}(\up{h})-(1-\up{p}^*(\set{Z}))R|_{\set{Z}}(\up{h})+\int_{x\in\set{X}\setminus\set{Z},y\in\set{Y}}\ell(\up{h},(x,y))\text{d}\up{p}^*(x,y)
\end{align*}
so that
\begin{align*}R(\up{h})&\leq R|_{\set{Z}}(\up{h})+(1-\up{p}^*(\set{Z}))(1-R|_{\set{Z}}(\up{h}))\leq R|_{\set{Z}}(\up{h})+(1-\up{p}^*(\set{Z}))\\
R(\up{h})&\geq R|_{\set{Z}}(\up{h})-(1-\up{p}^*(\set{Z}))\end{align*}
because $0\leq\ell(\up{h},(x,y))\leq 1$ for any $x$ and $y$. 

Taking
$$\set{U}_1=\{\up{p}\in\Delta(\set{X}_u\times\set{Y}):\  |\mathbb{E}_\up{p}\Phi-\B{\tau}|\preceq\B{\lambda}+2C\varepsilon_s\V{1}\}$$
$$\set{U}_2=\{\up{p}\in\Delta(\set{X}_l\times\set{Y}):\  |\mathbb{E}_\up{p}\Phi-\B{\tau}|\preceq\B{\lambda}+2C\varepsilon_s\V{1}\}$$
we have that $\up{p}^*\in\set{U}$ implies that $\up{p}^*|_{\set{X}_u}\in\set{U}_1$ and $\up{p}^*|_{\set{X}_l}\in\set{U}_2$ with probability at least $1-\delta$. Then, using Theorem~\ref{th-bounds} we have that with probability at least $1-2\delta$
\begin{align}\label{ineq2}R|_{\set{X}_u}(\up{h}_s)& \leq\sup_{\up{p}\in\set{U}_1}\ell(\up{h}_s,\up{p})\leq\inf_{\B{\mu}}1-\B{\tau}^{\text{T}}\B{\mu}+\sup_{x\in\set{X}_u,y\in\set{Y}}\{\Phi(x,y)^{\text{T}}\B{\mu}-\up{h}_s(y|x)\}+\B{\lambda}^{\text{T}}|\B{\mu}|+2C\varepsilon_s\|\B{\mu}\|_1\nonumber\\
&\leq 1-\B{\tau}^{\text{T}}\B{\mu}_u+\sup_{x\in\set{X}_u,y\in\set{Y}}\{\Phi(x,y)^{\text{T}}\B{\mu}_u-\up{h}_s(y|x)\}+\B{\lambda}^{\text{T}}|\B{\mu}_u|+2C\varepsilon_s\|\B{\mu}_u\|_1\nonumber\\
&=1-\B{\tau}^{\text{T}}\B{\mu}_u+\max_{x\in\set{X}_s,y\in\set{Y}}\{\Phi(x,y)^{\text{T}}\B{\mu}_u-\up{h}_s(y|x)\}+\B{\lambda}^{\text{T}}|\B{\mu}_u|+2C\varepsilon_s\|\B{\mu}_u\|_1\nonumber\\
&= \overline{R}_s(\set{U})+\varepsilon_s2C\|\B{\mu}_u\|_1
\end{align}
using the definition of $\set{X}_u$ and Corollary~\ref{cor-upper}. In addition,  we have that with probability at least $1-2\delta$
\begin{align}\label{ineq3}R|_{\set{X}_l}(\up{h}_s)& \geq\inf_{\up{p}\in\set{U}_2}\ell(\up{h}_s,\up{p})\geq\sup_{\B{\mu}}1-\B{\tau}^{\text{T}}\B{\mu}+\inf_{x\in\set{X}_l,y\in\set{Y}}\{\Phi(x,y)^{\text{T}}\B{\mu}-\up{h}_s(y|x)\}-\B{\lambda}^{\text{T}}|\B{\mu}|-2C\varepsilon_s\|\B{\mu}\|_1\nonumber\\
&\geq 1-\B{\tau}^{\text{T}}\B{\mu}_l+\inf_{x\in\set{X}_l,y\in\set{Y}}\{\Phi(x,y)^{\text{T}}\B{\mu}_l-\up{h}_s(y|x)\}-\B{\lambda}^{\text{T}}|\B{\mu}_l|-2C\varepsilon_s\|\B{\mu}_l\|_1\nonumber\\
&=1-\B{\tau}^{\text{T}}\B{\mu}_l+\min_{x\in\set{X}_s,y\in\set{Y}}\{\Phi(x,y)^{\text{T}}\B{\mu}_l-\up{h}_s(y|x)\}-\B{\lambda}^{\text{T}}|\B{\mu}_l|-2C\varepsilon_s\|\B{\mu}_l\|_1\nonumber\\
&= \underline{R}_s(\set{U})-\varepsilon_s2C\|\B{\mu}_l\|_1
\end{align}
using the definition of $\set{X}_l$. Therefore,  the result is obtained since inequalities \eqref{ineq1}, \eqref{ineq2}, \eqref{ineq3} are simultaneously satisfied with probability at least $1-2\delta$.
\end{proof}

%because 
%\begin{align}\label{ineq2}\underline{R}_s(\set{U})\leq R|_{\set{X}_l}(\up{h}_s),\   R|_{\set{X}_u}(\up{h}_s) \leq \overline{R}_s(\set{U})\end{align}
%with probability at least $1-\delta$ since when $\up{p}^*\in\set{U}$ we have that
%\color{blue}
%$$R|_{\set{X}_u}(\up{h}_s) \leq\sup_{\up{p}\in\set{U}_1}\ell(\up{h}_s,\up{p})\leq \overline{R}_s(\set{U})$$
%$$R|_{\set{X}_u}(\up{h}_s) \geq\inf_{\up{p}\in\set{U}_2}\ell(\up{h}_s,\up{p})\geq \underline{R}_s(\set{U})$$
%using Theorem~\ref{} and the fact that \eqref{opt-u} and \eqref{opt-l} are equivalent to those obtained substituting $\set{X}_s$ by $\set{X}_u$ and $\set{X}_l$, respectively.
%\color{black}
%Then, the result is obtained since inequalities \eqref{ineq1} and \eqref{ineq2} are simultaneously satisfied with probability at least $1-2\delta$.

%and optimization problems \eqref{opt-u} and \eqref{opt-l} are equivalent to those obtained substituting $\set{X}_s$ by $\set{X}_u$ and $\set{X}_l$, respectively. In particular, rule $\up{h}_s$ restricted to $\set{X}_u$ is an \ac{MRC} for the corresponding uncertainty set of distributions on $\set{X}_u$ that \color{blue} satisfies R1 or R2 because it contains $\set{U}_s$\color{black}. \end{proof}
\section{Proof of Theorem~\ref{thm:comp.Tao-1}}
\label{apd:proof_thm:comp.Tao-1}
\begin{proof}
Let $\B{\alpha}$, $\V{G}$, and $\V{H}$ be given as in Algorithm~\ref{alg.Tao-1}.  If $g(\B{\mu})$ is the subgradient given by  \eqref{eq:subgrad}, we have that 
  \begin{align*}
    \V{F}\left(\B{\mu}-cg(\B{\mu})\right)+\V{b}&=                                          \V{F}\left(\B{\mu}-c\left(\V{a}+
\B{\lambda}\odot\mbox{sign}(\B{\mu}) +
\text{col}_{i{(\B{\mu})}}(\V{F})\right)\right)+\V{b}\\
&  =
\V{F}\B{\mu}+\V{b}-c\left( \B{\alpha} +\V{H}\mbox{sign}(\B{\mu}_-) +\mbox{col}_{i(\B{\mu})}(\V{G})\right)-
c\V{H}\left(\mbox{sign}(\B{\mu})-\mbox{sign}(\B{\mu}_-)\right)
  \end{align*}
  for any  $\B{\mu}_-\in\mathbb{R}^m$ and $c\in\mathbb{R}$. Then, using induction it is straightforward to show that in Algorithm~\ref{alg.Tao-1} we have that
  \begin{align*}
  \V{y}_{k+1} &= \B{\mu}_k-c_kg(\B{\mu}_k),& \\
  \B{\mu}_{k+1} &=\V{y}_{k+1}+\theta_k(\theta_k^{-1}-1)(\V{y}_{k+1}-\V{y}_k)\\
  \V{w}_{k+1}&=\V{F}\V{y}_{k+1}+\V{b}\\
   \V{v}_{k+1}&=\V{F}\B{\mu}_{k+1}+\V{b}\\ i_{k+1} &=  \arg\max\V{F}\B{\mu}_{k+1}+\V{b} 
  \end{align*}
for all $k$. Therefore the sequences $\{\B{\mu}_k\}$ generated by Algorithms \ref{alg.Tao} and \ref{alg.Tao-1} are identical.

To prove the second claim, note that the computational complexity of Algorithm~\ref{alg.Tao} in each iteration is given by the  multiplication $\V{F}\B{\mu}_{k+1}+\V{b}$ in step 6, which has a time complexity of $\set{O}(pm)$. On the other hand,  the computational complexity  of Algorithm~\ref{alg.Tao-1} in each iteration is determined by steps 7-14 which have a time complexity of $\set{O}(pm\gamma(\B{\Delta}_k))$ where $\gamma(\B{\Delta}_k)$ denotes the fraction of non-zero components of vector  $\B{\Delta}_k=\mbox{sign}(\B{\mu}_{k+1})- \mbox{sign}(\B{\mu}_{k})$. 
\end{proof}

\end{document}